\documentclass[11pt]{article}
\usepackage[top=1in,bottom=1in,right=1in,left=1in]{geometry}
\usepackage{amssymb,amsmath,amsthm}
\usepackage{pgfplots}
\usepackage{pgfplotstable}
\usepackage{arydshln}
\usepackage{verbatim}
\usepackage{titlesec}
\usepackage{subfigure}
\usepackage[linesnumbered,ruled]{algorithm2e}
\usepackage{bm}
\usepackage{pdfpages}

\definecolor{forestgreen}{rgb}{0.13, 0.55, 0.13}
\usepackage[colorlinks, linkcolor=blue,citecolor=forestgreen,urlcolor = black, bookmarks=true]{hyperref}
\usepackage[T1]{fontenc}
 
\usepackage[nameinlink, noabbrev, capitalize]{cleveref}
\usepackage{times}

\crefname{equation}{}{}
\crefrangeformat{equation}{(#3#1#4) --~(#5#2#6)}
\crefmultiformat{equation}{(#2#1#3)}{, (#2#1#3)}{, (#2#1#3)}{, (#2#1#3)}
\crefname{lem}{Lemma}{Lemmas}
\crefname{section}{Section}{Sections}
\crefname{subsubsubsection}{Section}{Sections}
\crefname{rem}{Remark}{Remarks}
\crefname{figure}{Figure}{Figures}
\crefname{table}{Table}{Tables}
\Crefname{lem}{Lemma}{Lemmas}
\crefname{thm}{Theorem}{Theorems}
\Crefname{thm}{Theorem}{Theorems}

\usepackage{commath}

\newtheorem{thm}{Theorem}[section]
\newtheorem{assumption}{Assumption}
\newtheorem{claim}{Claim}[section]

\newtheorem{example}[thm]{Example}
\newtheorem{remark}[thm]{Remark}

\newtheorem{question}{Question}
\newtheorem{lem}[thm]{Lemma}
\newtheorem{fact}{Fact}[section]

\newtheorem{proposition}[thm]{Proposition}
\newtheorem{corollary}[thm]{Corollary}

\theoremstyle{definition}
 
\theoremstyle{definition}
\newtheorem{defn}{Definition}

\newcommand\blfootnote[1]{%
  \begingroup
  \renewcommand\thefootnote{}\footnote{#1}%
  \addtocounter{footnote}{-1}%
  \endgroup
}

\title{A Unified Approach to Learning Ising Models: Beyond Independence and Bounded Width\blfootnote{J.G. is supported by Vannevar Bush Faculty Fellowship ONR-N00014-20-1-2826 and Simons Investigator Award 622132. E.M. is supported in part by Vannevar Bush Faculty Fellowship ONR-N00014-20-1-2826, Simons Investigator Award 622132, Simons-NSF DMS-2031883,  and NSF Award CCF 1918421.}}
\author{Jason Gaitonde\\
Massachusetts Institute of Technology\\
\texttt{gaitonde@mit.edu}
\and Elchanan Mossel\\
Massachusetts Institute of Technology\\
\texttt{elmos@mit.edu}}

\pgfplotsset{compat=1.18}
\begin{document}
\clearpage\maketitle
\thispagestyle{empty}

\begin{abstract}

We revisit the well-studied problem of efficiently learning the underlying 
structure and parameters of an Ising model from data. Current algorithmic approaches achieve essentially optimal sample complexity when samples are generated i.i.d. from the stationary measure and the underlying model satisfies ``width'' constraints that bound the total 
$\ell_1$ interaction involving each node.
However, these assumptions are 
not satisfied in
some important settings of interest, like temporally correlated data or more complicated models (like spin glasses) that do not satisfy width bounds.

We analyze a simple existing approach based on node-wise logistic regression, and show it provably succeeds at efficiently recovering the underlying Ising model in several new settings:
\begin{enumerate}
    \item Given dynamically generated data from a wide variety of local Markov chains, including Glauber, block, and round-robin dynamics, logistic regression recovers the parameters with sample complexity that is optimal up to $\log\log n$ factors. 
    This generalizes work of Bresler, Gamarnik, and Shah [IEEE Trans. Inf. Theory '18], which provides a specialized algorithm for only structure recovery in bounded degree graphs from continuous-time Glauber dynamics.
    
    \item For the Sherrington-Kirkpatrick model of spin glasses, given $\mathsf{poly}(n)$ independent samples, logistic regression recovers the parameters in most of the known high-temperature regime, including the entire high-temperature regime with zero field. This improves on recent work of Anari, Jain, Koehler, Pham, and Vuong [ArXiv'23] which gives distribution learning at very high temperature via a reduction of Koehler, Heckett, and Risteski [ICLR'23]. Our analysis provides a much simpler reduction from parameter learning to weaker structural properties of the Gibbs measure.
    
    \item As a straightforward byproduct of our techniques, we also obtain an exponential improvement in learning from samples in the M-regime of data considered by Dutt, Lokhov, Vuffray, and Misra [ICML'21] as well as novel guarantees for learning from the adversarial Glauber dynamics of Chin, Moitra, Mossel, and Sandon [ArXiv'23]. For the former, this sample complexity is exponentially better than what can be achieved using i.i.d. samples from the stationary measure.
\end{enumerate}  

Our approach thus provides a significant generalization of the elegant analysis of logistic regression by Wu, Sanghavi, and Dimakis [Neurips'19] without any algorithmic modification in each setting. Our techniques are quite modular and may have further applications. 
\end{abstract}

\newpage

\clearpage
\setcounter{page}{1}
\section{Introduction}

In this work, we study the problem of efficiently recovering the structure and parameters of an Ising model given data generated from the underlying model. Ising models are special cases of Markov random fields and provide a convenient graphical representation of the joint dependency structure of high-dimensional data. More precisely, given a symmetric matrix of interactions $A\in \mathbb{R}^{n\times n}$ and a vector of external fields $\bm{h}\in \mathbb{R}^n$, the Ising model $\mu_{A,\bm{h}}$ is the distribution on vectors in $\{-1,1\}^n$ where
\begin{equation*}
    \mu_{A,\bm{h}}(\bm{x})=\exp\left(\frac{1}{2}\bm{x}^TA\bm{x}+\bm{h}^T\bm{x}-\log Z_{A,\bm{h}}\right),
\end{equation*}
and the partition function $Z_{A,\bm{h}}$ is defined by 
\begin{equation*}
    Z_{A,\bm{h}}\triangleq \sum_{\bm{x}\in \{-1,1\}^n} \exp\left(\frac{1}{2}\bm{x}^TA\bm{x}+\bm{h}^T\bm{x}\right).
\end{equation*}
In this formalism, the magnitude and sign of each coefficient $A_{i,j}\in \mathbb{R}$ represent the strength and direction of the local dependence of the variables $x_i$ and $x_j$. The quintessential property of such models is that the dependency structure for a given node $i\in [n]$ is entirely given by its graph-theoretic neighborhood; conditioned on the value of $x_k$ for all $k\neq i$, the conditional distribution of $x_i$ only depends on the values of $x_j$ for those $j$ with $A_{i,j}\neq 0$.

Due to their conceptual simplicity and wide-ranging applicability across numerous domains (for instance as abstractions of social networks or atomic structure), Ising models have been an object of intense study in computer science, economics, statistical physics, and beyond. As a result, the statistical and computational problem of recovering the underlying structure or estimating the parameters of an Ising model has been the topic of numerous research works. While learning arbitrary high-dimensional distributions is generically intractable, the significantly lower dimensional parametric structure of Ising models makes them much more amenable to these kinds of statistical tasks. The fundamental question our work sheds light on is the following:

\begin{question}
\label{question:q1}
    What kinds of data generated from $\mu_{A,\bm{h}}$ and what kinds of assumptions on $(A,\bm{h})$ are sufficient for efficiently learning the underlying model?
\end{question}

By far the most commonly studied setting for learning Ising models restricts to the case where one observes independent samples $X_1,\ldots,X_p\sim \mu_{A,\bm{h}}$, and the parameters $(A,\bm{h})$ are assumed to satisfy further constraints that make the learning problem tractable. 
Under these classical assumptions, there are by now multiple algorithmic approaches to this problem. Early work by Bresler, Mossel, and Sly provides an algorithm for structure learning (i.e. recovering the edges of the underlying graph) on graphs with maximum degree $d$ with run-time $\Omega(n^d)$ \cite{DBLP:journals/siamcomp/BreslerMS13}. The first provable guarantees for general bounded-degree models with polynomial run-time independent of the degree (though doubly-exponential in the implicit constants) were obtained by Bresler \cite{DBLP:conf/stoc/Bresler15} using a greedy, combinatorial approach, which was later extended by Hamilton, Koehler, and Moitra to more general Markov random fields \cite{DBLP:conf/nips/HamiltonKM17}. 
More recent approaches take an optimization-oriented approach by implicitly or explicitly minimizing a suitable convex loss function, possibly with regularization \cite{ravikumar2010high,DBLP:conf/nips/VuffrayMLC16,DBLP:conf/focs/KlivansM17,DBLP:conf/nips/WuSD19,DBLP:conf/nips/VuffrayML20}.

To state the current best guarantees, we say the $\ell_1$ width $\lambda=\lambda(A,\bm{h})$ of an Ising model parameterized by $(A,\bm{h})$ is given by $\lambda\triangleq \max_{i\in [n]} \|A_i\|_1+\vert h_i\vert$, where $A_i$ denotes the $i$th row of $A$. For Ising models with $\ell_1$ width known to be at most $\lambda$, the sample complexity 
 \cite{DBLP:conf/focs/KlivansM17, DBLP:conf/nips/WuSD19,DBLP:conf/nips/VuffrayML20} of efficiently learning the parameters up to some small additive accuracy scales roughly like $p=\exp(C\lambda)\log(n)$, where $C\geq 1$ is a reasonable constant. For the general problem of learning over this class, these guarantees are essentially optimal up to the precise constant in the exponential scaling with $\lambda$ due to information-theoretic lower bounds shown by Santhanam and Wainwright \cite{DBLP:journals/tit/SanthanamW12}. We note that for the task of learning \emph{some} Ising model that is statistically close to the generative Ising model from samples without requiring parameter recovery, Devroye, Mehrabian, and Reddad \cite{devroye2020minimax} showed that the minimax sample complexity is $p=\Theta(n^2)$ by VC-type arguments. However, this result does not lead to an efficient algorithm.

While the theoretical problem of efficiently learning the parameters of Ising models from i.i.d. data under \emph{just} these assumptions is thus essentially resolved, these algorithmic guarantees are not entirely satisfying for several reasons. First, obtaining i.i.d. samples from a general Ising model may be provably intractable \cite{DBLP:conf/focs/SlyS12} or at least practically quite difficult, even in cases where the interactions are known (though of course the learning problem is trivial in this case). An arguably more natural assumption is instead that samples are generated from some local Markov chain whose stationary distribution is that of the Ising model, since such samples can be significantly easier to generate. In this case, samples are both temporally and spatially correlated. In certain settings like coordination games \cite{DBLP:conf/focs/MontanariS09}, one may even view these dynamics as the more natural observational primitive than stationary samples. Of course, if the Markov chain rapidly mixes, one may extract nearly independent samples from the stationary measure. But rapid mixing is often analytically unknown, quantitatively slow, or worse, provably false even in simple models where learning is nonetheless possible. A natural question is to determine the extent to which existing methods are robust to these kinds of data generation, or whether specialized approaches are necessary.

To our knowledge, the only prior work on provably and efficiently learning Ising models in the dynamical setting was done by Bresler, Gamarnik, and Shah \cite{DBLP:journals/tit/BreslerGS18}. Their main result is that given the trajectory of single-site, continuous Glauber dynamics on bounded-degree graphs, a suitable thresholding algorithm can succeed in recovering the \emph{structure} of the underlying Ising model (given some minimum value of the nonzero elements in the support of $A$). Their main result shows that the per-node sample complexity is roughly the same as that of the i.i.d. setting up to constants in the exponent: if $A$ has maximum entry at most $\beta$ and each node has degree at most $\Delta$, they show that essentially the same $\exp(O(\beta\Delta))\log n$ updates per node suffices for structure recovery. 

However, the problem of structure recovery is in principle much weaker than parameter recovery in the non-i.i.d. setting, and indeed, their algorithm analysis only coarsely distinguishes between edge and non-edge correlations. 
Moreover, it is not clear 
whether their result extends to the more general bounded $\ell_1$ width setting, nor whether it extends to any other local Markov chain without further alteration. Even with a reduction from algorithms for general graphs to algorithms for bounded degree graphs (which does not appear to be known), one would na\"ively pay a bound of $\exp(\Omega(\lambda^2))\log (n)$ in the sample complexity to obtain constant additive accuracy in all parameters since an $\ell_1$ bounded interaction matrix may have $\Omega(\lambda)$ entries with constant magnitude, as well as an upper bound of $\lambda$ for the maximum entry.

Returning to the i.i.d. setting, the statement of current algorithmic guarantees in terms of $\ell_1$ width also does not appear to reflect any natural \emph{structural properties} of the Gibbs measure under consideration. For instance, consider the Curie-Weiss model $\mu_{\mathsf{CW},\beta}$ that is proportional to $\exp((\beta/n)\sum_{i<j} x_ix_j)$ at inverse temperature $\beta\geq 0$. It is well-known that this model exhibits a phase transition at the critical inverse temperature $\beta = 1$; for $\beta<1$, the model exhibits very weak dependencies between sites, while for $\beta>1$, the model develops global correlations and the Gibbs measure nontrivially polarizes (see e.g. \cite{friedli_velenik_2017}). On the other hand, the above sample complexity guarantees smoothly depend on $\lambda=\beta - o(1)$ so do not reflect any such phase transition. While the width parameter appears related to the well-known Dobrushin condition for fast mixing \cite{dobrushin1987completely}, the correspondence in terms of provable guarantees for learning is seemingly superficial since the Dobrushin condition experiences the same phase transition.

This difference is somewhat mild in the case of the Curie-Weiss model, but this issue manifests much more dramatically when considering more complex Ising models of great importance, like the Sherrington-Kirkpatrick (SK) model of spin glasses \cite{PhysRevLett.35.1792}. In this model, the interaction matrix $A$ is a scaled sample from the Gaussian Orthogonal Ensemble (GOE); that is, $A$ is a symmetric matrix such that each entry above the diagonal is sampled according to a Gaussian with variance $\beta^2/n$. The SK model has profound applications in various problems across computer science, statistical physics, and statistics, and generally serves as a testbed for a variety of algorithmic and mathematical techniques as a tractable mean-field model with intricate geometric structure (as captured by the famous Parisi ansatz \cite{Parisi_1980, talagrand2010mean, panchenko2013sherrington}). With no external field, the high-temperature regime of the SK model where the model is \emph{replica-symmetric}, meaning two independent draws of the measure are nearly orthogonal, provably holds precisely when $\beta<1$ \cite{talagrand2011mean}. The problem of learning the parameters of this model from i.i.d. samples in this regime has been previously considered via non-exact and heuristic methods in the statistical physics community under the name of ``spin glass inversion'' (see e.g. \cite{mezard2009constraint,Bachschmid_Romano_2017}).

Because the $\ell_1$ width concentrates around $\beta\sqrt{n}$ with very high probability, existing algorithmic guarantees generically suffer a sample complexity bound of $\exp(\Omega(\sqrt{n}))$ for any $\beta>0$. Only very recent work by Anari, Jain, Koehler, Pham, and Vuong \cite{anari_universality} has broken this barrier by obtaining an exponential improvement on this bound: in their work, they show that a standard pseudolikelihood maximization approach succeeds in learning the distribution (in total variation distance) up to inverse temperature $\beta=1/4$ given $\mathsf{poly}(n)$ samples. They achieve this result by leveraging an recent insight of Koehler, Heckett, and Risteski \cite{DBLP:conf/iclr/KoehlerHR23} that reduces learning in KL divergence from samples to a quantitative form of ``approximate tensorization of entropy'' \cite{caputo2015approximate}.

Other than the fact that their learning guarantee does not imply parameter recovery, one downside of this approach in general is that approximate tensorization of entropy is a difficult condition to prove, if it holds at all. Indeed, it implies modified log-Sobolev inequalities for Glauber dynamics, and even establishing polynomial mixing of the SK model in the entire high-temperature regime $\beta<1$ seems beyond current techniques.\footnote{See \cite{DBLP:conf/focs/AlaouiMS22,celentano} for an alternative sampling algorithm that succeeds asymptotically, albeit in a much weaker sense than total variation distance.} The proof that SK satisfies this powerful condition for $\beta<1/4$ relies on very recent breakthroughs for proving functional inequalities \cite{bauerschmidt,eldan2022spectral,adhikari2022spectral} that necessarily cannot extend past $\beta=1/4$ in general, whereas the regime for efficient parameter estimation for this model plausibly extends to the entire high-temperature regime $\beta<1$. Moreover, it is also not clear that \emph{dynamical} notions like approximate tensorization of entropy are the right technical approach to this problem since they fundamentally correspond to the relationship between global and local properties of the underlying Ising measure. To efficiently learn, one might expect that only weaker structural properties of the measure should matter.

\subsection{Our Results}

In this work, we validate the intuition that the intrinsic complexity of learning should depend only on the typical structural properties of samples, both in the i.i.d. setting and beyond. In particular, we prove that arguably the simplest of known algorithms for learning Ising models, the constrained logistic regression approach of Wu, Sanghavi, and Dimakis \cite{DBLP:conf/nips/WuSD19}, actually succeeds in recovering the underlying parameters of Ising models under significantly more general conditions than previously known. Our work thus helps clarify the connection between relevant structural features of the Ising measure under consideration and the efficient learnability of the model as per \Cref{question:q1}.

Our first main contribution in \Cref{sec:dynamics} is to show that logistic regression is surprisingly \emph{robust} to the data-generating mechanism. In particular, given samples generated from a wide class of local Markov chains, we show that vanilla logistic regression attains near-optimal sample complexity. An informal version of our result reads as follows:

\begin{thm}[\Cref{thm:learning_block}, informal] 
\label{eq:thm:block_intro} Let $\mu_{A,\bm{h}}$ be an Ising measure with known width $\lambda$. Given a trajectory of a ``nice'' local Markov chain which updates variables according to the conditional law of $\mu_{A,\bm{h}}$  such that each node is updated at least\footnote{Here, we use $\tilde{O}(f(n))$ to denote quantities of the form $f(n)\mathsf{polylog}(f(n))$.}
\begin{equation*}
\tilde{O}\left(\frac{\lambda^2\exp(C\lambda)\log(n)}{\varepsilon^4}\right)
\end{equation*}
times, the output of node-wise logistic regression returns an $\varepsilon$-additive estimate of $A$ with high probability. In particular, this holds for any symmetric block dynamics (i.e. including Glauber and uniform block updates) or round-robin dynamics given any starting configuration for the Markov chain.
\end{thm}

In \Cref{eq:thm:block_intro}, ``nice'' corresponds to mild conditions that ensure some amount of nondegeneracy in the local dynamics; roughly speaking, they amount to requiring that that no subset of nodes gets frequently updated while others do not. On a per-node basis, this result is essentially the same as what can be obtained given i.i.d. samples from the stationary measure. Importantly, no particular specialization is required in the algorithm to account for the fact that the data is correlated across samples. The only difference is we provide a substantially more general \emph{analysis} than that of Wu, Sanghavi, and Dimakis to account for the correlation structure. \Cref{eq:thm:block_intro} thus provides a wide-reaching and conceptually simpler generalization of the specialized and weaker approach of Bresler, Gamarnik, and Shah \cite{DBLP:journals/tit/BreslerGS18} that unifies this setting with the i.i.d. case algorithmically. We demonstrate simple sufficient conditions (\Cref{prop:sufficiecy_lb}) on local Markov chains that enable efficient learning without algorithmic modification. We remark that this partially addresses a question of Dutt, Lokhov, Vuffray, and Misra \cite{DBLP:conf/icml/DuttLVM21}.

In \Cref{sec:sk}, we further show that vanilla logistic regression can significantly improve on the na\"ive sample complexity suggested by current algorithmic guarantees even in the i.i.d. setting. In particular, we show it also succeeds at learning parameters of the Sherrington-Kirkpatrick model (with no external field) with high probability in the entire high-temperature region $\beta<1$:

\begin{thm}[\Cref{thm:sk_no_external}, informal]
   For any $\beta<1$, with $1-o(1)$ probability over the realization of the Sherrington-Kirkpatrick measure $\mu_{\beta,A}$ where $A\sim \mathsf{GOE}(n)$, the following holds. Given $\tilde{O}_{\beta}(n^9)$ independent samples from $\mu_{\beta,A}$, node-wise logistic regression outputs a matrix $\widehat{A}$ such that the total variation distance between $\mu_{\widehat{A}}$ and $\mu_{\beta,A}$ is $o(1)$ with high probability over the samples.
\end{thm}

We provide more precise guarantees in \Cref{sec:sk}; we obtain this result by showing that one can actually recover the underlying parameters of model with high probability. When done to sufficient accuracy, this is easily seen to provide strong guarantees in terms of total variation distance. A key feature of our analysis is that in the case of zero external field, we show that learnability reduces to \emph{operator norm bounds of the covariance matrix}:

\begin{thm}[\Cref{cor:ising_cov}, informal]
    Suppose $\mu_{A}$ is a zero field Ising model such that each row of $A$ has $\ell_2$ norm at most $C_1$, and the covariance matrix of $\mu_A$ has operator norm at most $C_2$. Then given $\tilde{O}_{C_1,C_2}(n/\varepsilon^4)$ samples from $\mu_A$, logistic regression recovers a matrix $\widehat{A}$ that $\varepsilon$-additively approximates $A$ with high probability.
\end{thm}

The boundedness of covariance matrices, while nontrivial to establish, is \emph{provably weaker} than approximate factorization of entropy, as was required for the result of Anari, et al. \cite{anari_universality}. This is because it is an immediate consequence of a Poincar\'e inequality\footnote{A Poincar\'e inequality for Glauber dynamics essentially asserts that the global variance of a function with respect to the stationary measure is dominated by the local variation of the function along Glauber transitions from stationarity. The boundedness of the operator norm of the covariance matrix can be shown to hold by applying such a bound to just \emph{linear} functions.} for Glauber dynamics, while approximate tensorization of entropy itself implies a much stronger modified log-Sobolev inequality for Glauber dynamics. Therefore, this reduction is qualitatively \emph{stronger} (and conceptually much simpler) than that of Koehler, et al. \cite{DBLP:conf/iclr/KoehlerHR23}. While rapid mixing of the Sherrington-Kirkpatrick model up to $\beta<1$ remains a challenging open question, the boundedness of the covariance matrix was recently established independently by El Alaoui and Gaitonde \cite{alaoui2022bounds} and Brennecke, Scherster, Xu, and Yau \cite{brennecke2022two,brennecke2023operator}.

With significant technical work, we further show that this result can be extended to handle i.i.d. Gaussian external fields in a very large portion of the provably high-temperature regime:

\begin{thm}[\Cref{thm:sk_external}, informal]
    Suppose that the entries of $\bm{h}$ are i.i.d. from a common Gaussian distribution (not necessarily centered) and $A\sim \mathsf{GOE}(n)$. Then for any $\beta\geq 0$ satisfying a ``high-temperature condition'' (see \Cref{assumption:high_temp}), it holds with arbitrarily high probability over $(A,\bm{h})$ that given $\mathsf{poly}(n)$
    samples from $\mu_{\beta,A,\bm{h}}$, logistic regression returns $(\widehat{A},\bm{\widehat{h}})$ such that $\mu_{\widehat{A},\bm{\widehat{h}}}$ is $o(1)$ close to $\mu_{\beta,A,\bm{h}}$ in total variation distance.
\end{thm}

Here, the multiplicative constants depends on the desired probability parameter as well as the size of the high-temperature regime one considers since there are phase transitions as one approaches criticality. 

Finally, in \Cref{sec:extensions}, we provide two more applications of our general analysis. Recent work of Dutt, Lokhov, Vuffray, and Misra \cite{DBLP:conf/icml/DuttLVM21} considers the problem of learning Ising models from off-equilibrium samples. In their parlance, a single trajectory of Glauber dynamics is called the \emph{T-regime} and they show that experimentally, the behavior of existing learning algorithms for such samples approximates those of independent samples. On the other hand, they show that both theoretically and experimentally, these algorithms work much better given samples in the so-called \emph{M-regime}. These samples are obtained by taking a uniform configuration in $\{-1,1\}^n$ and then applying a single step of Glauber dynamics. Their main theoretical result shows that for interaction matrices with maximum degree $d$ and edge strength $\beta$, the sample complexity scales like $\exp(2\beta d)\cdot \log(n)$ per node to learn the underlying parameters. The constant in the exponential is lower than any known algorithm for the stationary, i.i.d. sample setting, suggesting that this regime is algorithmically easier. We show that even further improvements are possible:

\begin{thm}[\Cref{cor:sparse_m_regime}, informal]
    Let $A$ be an interaction matrix with maximum degree $\Delta$ and edge strength $\beta\geq 0$. Then given at least
    \begin{equation*}
        \frac{O((\beta \Delta)^2)\cdot \exp(O(\beta \sqrt{\Delta}))\log(n)}{\varepsilon^4}
    \end{equation*}
    M-regime samples per node, logistic regression recovers a matrix $\widehat{A}$ that $\varepsilon$-additively approximates $A$ with high probability.
\end{thm}
In particular, we obtain via a very simple argument an exponential improvement in the sample complexity of learning from M-regime samples. This result actually certifies that M-regime samples are provably \emph{exponentially easier} for learning than stationary measure samples since the information-theoretic lower bounds of Santhanam and Wainwright \cite{DBLP:journals/tit/SanthanamW12} imply $\exp(\Omega(\beta d))\log n$ samples are necessary in that case. We obtain this result via a more general theorem (\Cref{thm:m-regime}) that also shows that given M-regime samples, there is no ``low-temperature hardness'' for learning in models like Sherrington-Kirkpatrick.

Finally, we show how our analysis for local Markov chains extends to the setting where some adversarial nodes can update dishonestly (i.e. not according to $\mu_A$); such a model of Glauber dynamics was recently introduced by Chin, Moitra, Mossel, and Sandon \cite{adversarial}. We show that under mild constraints on the behavior of adversarial nodes, logistic regression will nonetheless succeed in recovering the parameters for non-adversarial nodes that update honestly on bounded-degree graphs (see \Cref{thm:adversarial_glauber} for a precise statement). Our bounds degrade gracefully with the power of the adversarial nodes.

\subsection{Other Related Work}
The literature on statistical inference in Ising models, or more general graphical models, is quite large; as a result, we only highlight a few recent related works beyond those mentioned above.  Zhang, Kamath, Kulkarni, and Wu \cite{DBLP:conf/icml/Zhang0KW20} have shown that the analysis of Wu, Sanghavi, and Dimakis \cite{DBLP:conf/nips/WuSD19} can be extended to learning more general Markov random fields, as well as be made differentially private---we expect that the first extension should be possible for our results, though we have not done so for simplicity. Dagan, Daskalakis, Dikkala, and Kandiros \cite{DBLP:conf/stoc/DaganDDK21} considers recovery in Frobenius norm in the \emph{data-constrained} regime where one receives a very small number of samples, possibly just one. In the case that samples are stochastically corrupted, Goel, Kane, and Klivans \cite{DBLP:conf/colt/GoelKK19} show that it is possible to still perform parameter recovery using an appropriate denoising procedure; on the other hand, learning from samples in Huber's contamination model was studied by Prasad, Srinivasan, Balakrishnan, and Ravikumar \cite{DBLP:conf/nips/PrasadSBR20}, as well as by Diakonikolas, Kane, Stewart and Sun \cite{DBLP:conf/colt/DiakonikolasKSS21}. Moitra, Mossel, and Sandon provide algorithms to approximately \emph{sample} from censored, high-temperature Ising models \cite{DBLP:conf/colt/MoitraMS21}. Several works have also considered the related, but distinct, testing problem:  given sample access to an Ising model, can one determine whether it belongs to a certain class of distributions? See for instance the work of Daskalakis, Dikkala, and Kamath \cite{DBLP:conf/soda/DaskalakisDK18} and the references therein.

In the case that the underlying Ising model is known to have a tree structure, classical work of Chow and Liu \cite{chow_liu} gives an efficient algorithm for learning the maximum likelihood tree. Several recent works have provided more refined learning guarantees in this setting, see for instance, Bresler and Karzand \cite{bresler_karzand}, Daskalakis and Pan \cite{DBLP:conf/stoc/DaskalakisP21}, Boix-Adsera, Bresler, and Koehler  \cite{DBLP:conf/focs/Boix-AdseraBK21}, and Kandiros, Daskalakis, Dagan, and Choo \cite{DBLP:conf/colt/KandirosDDC23}.

Recent work has also considered the case where one must also learn with latent (unobserved) variables; these models are known as Boltzmann machines. This setting is provably quite computationally challenging in general (see, for instance, Bogdanov, Mossel, and Vadhan \cite{DBLP:conf/approx/BogdanovMV08} for hardness results), motivating algorithms for learning more restricted models. See Anandkumar and Valluvan \cite{anandkumar2013learning}, Bresler, Koehler, and Moitra \cite{DBLP:conf/stoc/BreslerKM19}, Goel \cite{DBLP:conf/aistats/Goel20}, and Goel, Klivans, and Koehler \cite{DBLP:conf/nips/GoelKK20} for recent results and discussion on this problem.\\
 
\noindent\textbf{Organization.}
In \Cref{sec:overview}, we provide a high-level overview of our technical analysis. In \Cref{sec:general}, we state and prove a somewhat abstract form of the main meta-theorem of Wu, Sanghavi, and Dimakis that provides sufficient conditions to obtain guarantees on the performance of logistic regression. In \Cref{sec:dynamics}, we demonstrate how to verify these conditions under simple assumptions on local Markov chains. In \Cref{sec:sk}, we then turn to showing how a similar analysis, combined with recent breakthroughs from the spin glass literature, leads to efficient algorithms for provably learning the Sherrington-Kirkpatrick model at high temperature. Finally, in \Cref{sec:extensions}, we provide further applications to learning in the M-regime and learning from adversarial Glauber dynamics.

\section{Technical Overview}
\label{sec:overview}
In this section, we discuss the main techniques and conceptual insights that will be used to obtain our main results. For the sake of exposition in this overview, we will assume that $\bm{h}=\bm{0}$.\\ 

\noindent\textbf{Logistic Regression: The Wu-Sanghavi-Dimakis Framework.} The starting point of our work is the elegant analysis of Wu, Sanghavi, and Dimakis \cite{DBLP:conf/nips/WuSD19} for the performance of arguably the simplest known approach to learning Ising models given independent stationary samples, which is to perform logistic regression to learn each node's interaction coefficients in $A$ separately. To state the algorithm and abstract guarantees, we focus on the case of learning the vector of interaction terms for node $n$, namely the row $A_n$. We assume that it is known that $A_n\in \mathcal{C}$ for a compact, convex set $\mathcal{C}\subseteq \mathbb{R}^{n-1}$.\footnote{Here, we slightly abuse notation and view $A_n\in \mathbb{R}^{n-1}$ by omitting the zero entry in coordinate $n$.} 

Suppose we are given samples $Z_1=(X_1,Y_1),\ldots,Z_T=(X_T,Y_T)\in \{-1,1\}^{n-1}\times \{-1,1\}$, and let $\mathcal{D}_t$ denote the law of $(X_t,Y_t)$ given $Z_1,\ldots,Z_{t-1}$. Suppose further that it holds that
\begin{equation}
\label{eq:faithful}
\Pr_{\mathcal{D}_t}\left(Y_t=1\vert X_t\right)=\frac{\exp(\langle A_{n}, X_{t}\rangle)}{\exp(\langle A_{n}, X_{t}\rangle)+\exp(-\langle A_{n}, X_{t}\rangle)} \triangleq \sigma(2\cdot(\langle A_{n}, X_{t}\rangle)),
\end{equation}
where $\sigma:\mathbb{R}\to [0,1]$ denotes the standard sigmoid function $\sigma(z)=\frac{1}{1+\exp(-z)}.$
In particular, we assume that $Z_t=(X_t,Y_t)\sim \mathcal{D}_t$ for a (possibly random) distribution $\mathcal{D}_t$ that may depend on the previous samples, but with the promise that the law of $Y_t$ given $X_t$ and the previous samples is taken from the conditional distribution of $\mu_A(\cdot \vert X_t)$ on the last spin. In the case that the $Z_t$ are drawn i.i.d. from $\mu_A$, there are no dependencies between samples and it holds that $\mathcal{D}_t=\mu_A$ surely.

Recall that the \emph{logistic loss} $\ell:\mathbb{R}\to \mathbb{R}$ is defined by $\ell(z)=\log\left(1+\exp(-z)\right).$
With this notation, the \emph{logistic regression} problem solves the convex program
\begin{equation}
\label{eq:lr_overview}
    \widehat{\bm{w}}=\arg\min_{\bm{w}\in \mathcal{C}} \frac{1}{T}\sum_{t=1}^T \ell\left(2Y_t\cdot \langle X_t,\bm{w}\rangle\right).
\end{equation}
For one last piece of notation, for any $\bm{w}\in \mathbb{R}^{n-1}$ and any distribution $\mathcal{D}$ on $\{-1,1\}^{n-1}\times \{-1,1\}$ satisfying \Cref{eq:faithful}, we define the \emph{population logistic loss} by $
    \mathcal{L}_{\mathcal{D}}(\bm{w}) = \mathbb{E}_{(X,Y)\sim \mathcal{D}}\left[\ell\left(2Y\cdot \langle X,\bm{w}\rangle\right)\right].$

We can now informally state a slightly generalized and non-quantitative form of the meta-theorem of Wu, Sanghavi, and Dimakis that provides performance guarantees for the optimization problem \Cref{eq:lr_overview}.
\begin{thm}[\Cref{thm:wsd}, informal]
\label{thm:wsd_overview}
   Let $\bm{w}^*=A_n$. Using the above notation and under the above conditions on the stochastic process $Z_1,\ldots,Z_T$, suppose further that:
   \begin{enumerate}
       \item (Uniform Convergence) It holds with high probability that both 
       \begin{equation*}
       \hspace*{-1.5cm}
           \sup_{\bm{w}\in \mathcal{C}}\frac{1}{T}\sum_{t=1}^T \mathcal{L}_{\mathcal{D}_t}(\bm{w})- \frac{1}{T}\sum_{t=1}^T \ell\left(2Y_t\cdot \langle X_t,\bm{w}\rangle\right)=o_T(1),\quad
           \left\vert\frac{1}{T}\sum_{t=1}^T \mathcal{L}_{\mathcal{D}_t}(\bm{w}^*)- \frac{1}{T}\sum_{t=1}^T \ell\left(2Y_t\cdot \langle X_t,\bm{w}^*\rangle\right)\right\vert=o_T(1).
       \end{equation*}
       \item (Strong Convexity) It almost surely holds that for all $t\in [T]$, $\bm{w}\in \mathcal{C}$, and any $i\in [n-1]$ that\footnote{Note that the distribution $\mathcal{D}_t$ itself may be random and depend on previous samples. Here, $\Omega_A(\cdot)$ hides problem-specific constants depending on $A$. Note that this inequality asserts that if $w_i$ differs from $w^*_i$, then necessarily the \emph{predictions of the bias of $Y_t$}.}
    \begin{equation*}\mathbb{E}_{\mathcal{D}_t}\left[\left(\sigma\left(2(\langle \bm{w},X_t\rangle)\right)-\sigma\left(2(\langle \bm{w}^*,X_t\rangle)\right)\right)^2\right]\gtrsim \Omega_A((w_i-w_i^*)^2).
    \end{equation*}
   \end{enumerate}
   Then for $T$ large enough, it holds with high probability that the solution $\widehat{\bm{w}}$ of \Cref{eq:lr_overview} satisfies $\widehat{\bm{w}}\approx \bm{w}^*$.
\end{thm}

The proof in the case of i.i.d. samples was essentially given by Wu, Sanghavi, and Dimakis and follows standard learning-theoretic arguments with important deterministic identities for $\ell$ and $\sigma$, but extends easily to general stochastic processes as we have defined them so we omit further discussion here. 

Given this result, one needs to establish strong quantitative forms of both conditions to deduce efficient learnability. To later contrast with our work, we briefly describe how this is done by Wu, Sanghavi, and Dimakis in the case of i.i.d. samples from $\mu_A$ with $\ell_1$ width bound $\lambda$, so that $\mathcal{C}$ is the $\ell_1$ ball of radius $\lambda$. The first condition of \Cref{thm:wsd_overview} is easily proven using standard uniform convergence arguments: since the logistic loss can be uniformly bounded in terms of $\lambda$, one may apply McDiarmid's bounded difference inequality to bound the deviation of both quantities from their mean. Note that this application of McDiarmid's inequality crucially relies on independence of samples. The mean of the supremum in the first quantity can in turn be easily bounded using the standard theory of Rademacher complexity.

For the second condition of \Cref{thm:wsd_overview}, the key observation of essentially all prior work on provably learning Ising models is that when samples are drawn i.i.d. from the stationary measure, a uniform $\ell_1$ bound implies that for each coordinate $j\in [n-1]$, spin $j$ retains nontrivial variance given the coordinates outside of $\{j,n\}$. This implies that if $\bm{w}_j$ differs from $\bm{w}^*_j$, it cannot be the case that $\bm{w}$ and $\bm{w}^*$ often makes the same predictions on the bias of spin $n$ since for any conditioning on the coordinates outside $\{j,n\}$, at least one of $X_j=1$ or $X_j=-1$ must give different predictions. Since both occur with non-negligible  probability for any conditioning of the other spins, with some  algebra, one can obtain the requisite lower bound of the second condition of \Cref{thm:wsd_overview}. Given the width bound $\lambda$, the conditional variance of the $j$th spin given the outside coordinates of an independent sample from $\mu_A$ can be easily shown to surely be at least $\exp(-O(\lambda))$, which explains this dependency in state-of-the-art sample complexities.

In the settings we consider, however, establishing one or both of these conditions becomes significantly more difficult. We now turn to explaining the challenges of learning from dynamic processes and for spin glasses, as well as the techniques we use to circumvent them. Simpler versions of these arguments will suffice for our remaining applications.\\

\noindent\textbf{Learning from Dynamics.} We first consider the case where we obtain dependent samples obtained by some local Markov chain with stationary measure $\mu_A$ where $A$ satisfies $\ell_1$ width $\lambda$. For simplicity, let us focus on the case that data is generated by Glauber dynamics: namely, we observe a trajectory of the form 
\begin{equation*}
    X_0\to (X_1,i_1)\to \ldots\to (X_T,i_T),
\end{equation*} 
where $X_0\in \{-1,1\}^n$ is an arbitrary initial configuration, and $X_{t+1}$ is obtained from $X_t$ by sampling $i_{t+1}$ uniformly from $[n]$ and only updating $X_{t+1,i_{t+1}}\in \{-1,1\}$ according to \Cref{eq:faithful} given the other coordinates $X_{t,-i_{t+1}}\in \{-1,1\}^{n-1}$. We again focus on a logistic regression problem associated to node $n$.

The most natural way to formulate the logistic regression problem \Cref{eq:lr_overview} for node $n$ is to extract appropriate samples by defining stopping times $\tau_{\ell}$ recursively: we set $\tau_1$ as the first update time $t$  of node $n$ (i.e. where $i_t=n$) and then let $\tau_{\ell+1}$ be the first time $t>\tau_{\ell}$ such that $i_{t}=n$. We then let our samples be
\begin{equation*}
    Z_1=(X_{\tau_1,-n}, X_{\tau_1,n}),\ldots, Z_{T_n}=(X_{\tau_{T_n},-n}, X_{\tau_{T_n},n}),
\end{equation*}
where $T_n$ is the number of times that $i_t=n$. Notice that the law of $Z_{\ell,n}$ given even the entire trajectory up to time $\tau_{\ell}$ is obtained by the update rule \Cref{eq:faithful} given just $Z_{\ell,-n}=X_{\tau_{\ell},-n}$ by the Markov property.

To leverage the above framework, we must first formulate an appropriate notion of ``population loss'' since the samples are not independently drawn, nor are they identically distributed from a common measure. There are two objectives illustrated above: first, the ``population loss'' must be chosen to satisfy a quantitatively strong version of uniform concentration with respect to the empirical losses. Second, the ``population loss'' must satisfy some sort of strong convexity condition; recall from our earlier discussion that in the i.i.d. case, this followed from nontrivial conditional variance in each coordinate given the rest. This second condition suggests that the ``population loss'' should use minimal conditioning.

A natural proposal is thus to define the (random) $k$th population loss of a vector $\bm{w}$ as the population logistic loss \emph{taken over the conditional distribution of $Z_k$ given just $Z_1,\ldots,Z_{k-1}$} (rather than the full trajectory until the $k$th update). For any fixed vector $\bm{w}\in \mathbb{R}^{n-1}$, the differences between the empirical losses and these population losses form a martingale difference sequence by construction. But unlike in the i.i.d. case, it is no longer clear that the \emph{supremum} of deviations of empirical and population losses over $\bm{w}\in \mathcal{C}$ tends to zero fast enough. As described above, standard uniform convergence bounds rely on McDiarmid's bounded difference inequality and Rademacher complexity arguments which critically require independence. In the dynamical setting, changing any $Z_{k}$ clearly affects the conditional distributions of $Z_{k+1},\ldots$, and therefore the corresponding population loss functions, invalidating the use of such tools. 

To overcome this, we must appeal to \emph{uniform martingale bounds}. To do so, we will leverage results of Rakhlin, Sridharan, and Tewari \cite{DBLP:conf/colt/RakhlinS17,DBLP:journals/jmlr/RakhlinST15} which show that uniform martingale bounds follow from control of the \emph{sequential Rademacher complexity} of the class of functions under consideration. Their main results imply that one can obtain similar concentration properties to the i.i.d. case for any such martingale difference sequence parameterized by a suitable (more precisely, online learnable) class, like the $\ell_1$-ball. By adapting some of their results, we will be able to recover the necessary uniform convergence in this setting up to a mild logarithmic factor, for \emph{any} process generating the samples. In particular, this will hold for any local Markov chain, not just Glauber dynamics.

Even with uniform convergence in hand, the second condition of \Cref{thm:wsd_overview} poses new technical challenges. Unlike the i.i.d. case, it is simply false that, given any initial configuration $Z_0=X_0$ and the coordinates outside of $\{j,n\}$ of $Z_1$ at the first update time of coordinate $n$, the conditional distribution of $j$ will retain nontrivial variance---this fact was crucial in deriving the convexity lower bound in the i.i.d. setting. To understand the difficulty, consider the Curie-Weiss model defined above with inverse temperature $\lambda>0$ initialized at the all-ones configuration. Conditioned on the event that all coordinates outside $\{j,n\}$ remain at $+1$ at the first update time of $n$, one may think that the variance of $Z_{1,j}$ will remain nontrivial since site $j$ and site $n$ were equally likely to update first by symmetry. If site $j$ updated before site $n$, then since the law of each site update can be shown to have variance at least $\exp(-O(\lambda))$ no matter the conditioning of all other nodes, this would appear to translate to the desired conclusion.

However, this reasoning is incorrect due to the following issue: conditioned on this particular value of the outside coordinates, it is actually very likely that site $n$ had an atypically small (i.e. $O_{\lambda}(1)$) first updating time. This arises because if the first update time is $\tau$, it holds with very high probability that at least $\exp(-O(\lambda))\tau$ sites updated to $-1$ at this update time. Conditioning on this configuration thus makes it quite likely that site $n$ was updated before site $j$. On this event, the spin at site $j$ is very likely fixed to be $+1$ as in the initial configuration, so the coordinate will be very biased under this conditioning.

To address this, our main observation is that it is not necessary to prove that large conditional variation in each coordinate holds almost surely conditional on the outside configuration at the update time. In fact, we only need to find some good event with constant probability for each $j\in [n-1]$ where this is the case. For instance, a prerequisite for site $j$ to have nontrivial variance is that $j$ is chosen for updating before site $n$, as otherwise the value of site $j$ is fixed to the value at the previous update time which we conditioned on. By symmetry, this occurs with probability at least $1/2$ conditioned just on the original configuration. However, the Curie-Weiss example above demonstrates that it is quite subtle to determine which good configurations of outside coordinates at the next update time we can condition on to ensure that the spin at site $j$ retains sufficient variance due to these complex dependencies between the conditioning and the law of the dynamics.

Our solution is that rather than conditioning on well-behaved outside configurations at the update time, it will be more technically convenient to condition on \emph{good paths} of the dynamics between updating times (which determine the new configuration given the configuration at the previous update time). More precisely, we will work on the event that the path of updates between updating times of $n$ includes an update of site $j$, and moreover, given that this update occurred, the variance of the value of this update remains controlled. To do this, we consider the law of the update path where we omit the actual value of the last update to site $j$ (if such an update occurred before the next update of site $n$). We will then show that so long as the dynamics are reasonable (in a very mild quantitative sense), then with constant probability over such partial paths, the \emph{likelihood ratio} of the value of spin $j$ conditional on the partial path remains bounded. Roughly speaking, this will hold because the $\lambda$ width condition implies that the spin value of $j$ has controlled influence on the subsequent path of updates with decent probability by explicit calculation. As in the i.i.d. case, an appropriately quantitative version of this argument will be enough to conclude that the population loss remains expressive enough to measure the variation in coordinates as needed for \Cref{thm:wsd_overview}.\\

\noindent\textbf{Learning the SK Model with No External Field.} We apply similar ideas to show how to learn the parameters of the Sherrington-Kirkpatrick model from i.i.d. samples in most of the known high-temperature regime from independent samples. Unlike the dynamical case, because we allow independent samples, well-known uniform convergence results directly apply to establish the first condition of \Cref{thm:wsd_overview}. Instead, the entire difficulty lies in the second condition: using \emph{just} an $\ell_1$ width bound of $\lambda$, this step pays a cost of $\exp(O(\lambda))$ in the sample complexity to account for the worst-case variance of a spin conditional on the others. For the SK model, this translates to $\exp(\Omega(\sqrt{n}))$ sample complexity for any constant inverse temperature $\beta>0$.

Our main insight in this setting is that while it is indeed true that the conditional variance of a spin scales in this way in the  \emph{worst-case} when outside spins are adversarially aligned, \Cref{thm:wsd_overview} itself does not require paying worst-case bounds. Rather, it again suffices to show that for a \emph{typical} (i.e. constant probability) sample from the stationary measure, any fixed coordinate will remain somewhat unbiased conditioned on the others. In the case of the SK model, we show this by observing the rows of the interaction matrix are not arbitrary vectors with $\ell_1$ norm $\Omega(\sqrt{n})$: rather, they are better viewed as vectors with \emph{constant $\ell_2$ norm}. 

With this observation, we obtain tail bounds on the typical variance of a site conditioned on the others by analyzing the quadratic form of the rows of the interaction matrix with the \emph{second moment matrix} $M=M(\mu_{A})$ of the distribution and then applying Chebyshev's inequality. If this concentrates near zero, then it will hold that the bias of sites $j$ and $n$ will simultaneously remain controlled on an event with constant probability, which will suffice for establishing the needed statement in \Cref{thm:wsd_overview}. This argument is conceptually simple and general, and may have applications in other important Ising models.
To actually implement this step, we can then simply appeal to recently established bounds on the operator norm of the covariance matrix (which agrees with the second moment matrix with no external field) at high-temperature that hold with probability $1-o(1)$ over $A\sim \mathsf{GOE}(n)$ for any $\beta<1$ \cite{alaoui2022bounds, brennecke2022two,brennecke2023operator}. While quite elementary, our reduction from learning Ising models with no external field to covariance operator norm bounds appears novel.\\ 

\noindent\textbf{Learning the SK Model with External Field.} The above argument breaks down when we introduce Gaussian external fields. The reason is that to control the typical conditional distribution of a spin, the second moment matrix $M$ is no longer spectrally bounded. Indeed, $M=\text{Cov}(\mu_{A,\bm{h}})+\bm{m}\bm{m}^T$, where $\text{Cov}$ is the covariance matrix and $\bm{m}$ is the mean vector. The second moment matrix is thus necessarily spectrally unbounded because of the rank-one spike due to the mean, which will have operator norm $\Theta(n)$. Therefore, it will no longer be true that \emph{arbitrary} linear functions of spins will concentrate near zero as above.

The key insight in saving the argument is that since the rows of the interaction matrix are Gaussian vectors by definition, one might hope that the quadratic form of the rows of $A$ with the rank-one spike $\bm{m}\bm{m}^T$ will remain bounded with high probability. Since operator norm bounds on the covariance matrix are known even in this setting, this would suffice to carry out the same argument as before. While this would be true by Gaussian concentration if $\bm{m}$ were \emph{independent} of $A$, this fails because the mean vector intimately depends on the realization of $A$ (as well as $\bm{h}$). Nonetheless, we show how approximate orthogonality follows from sufficiently quantitative versions of the well-known Thouless-Anderson-Palmer (TAP) equations for the SK model, as established by Talagrand. The TAP equations assert that for a certain quantity $q\in [0,1]$, it holds that
\begin{equation*}
    \bm{m}\approx \tanh\left(\beta A\bm{m}-\beta^2(1-q)\bm{m}+\bm{h}\right),
\end{equation*}
where $\tanh(\cdot)$ is applied entrywise. These equations are quite promising since they involve the quantity $A\bm{m}$, which we want to argue is entrywise small. If one had good \emph{a priori} bounds on $\|\bm{m}\|_{\infty}$, one could hope to approximately invert this identity and conclude boundedness; however, the only \emph{a priori} bounds we are aware of on $\|\bm{m}\|_{\infty}$ \cite{fan_mei_montanari}, as well as quantitative forms of ``$\approx$'' above, are far too loose to carry out such an argument.

To circumvent this issue, we instead employ machinery from Talagrand's \emph{proof} of the TAP equations \cite{talagrand2010mean} to argue there exists a coupling of $(A,\bm{h})$ with a random vector $\bm{z}$, whose coordinates are each marginally standard Gaussian, such that
\begin{equation*}
    \|\beta \bm{z}\sqrt{q}-\beta A\bm{m}-\beta^2(1-q)\bm{m}\|_{\infty}=o(1)
\end{equation*}
with high probability. This certifies $\|A\bm{m}\|_{\infty}=O(\sqrt{\log n})$ by standard bounds on the maximum of (arbitrarily correlated) Gaussian random variables and the fact $q\in [0,1]$. Our argument only pays exponentially in this quantity (as in the $\ell_1$ width case), so this incurs only sub-polynomial sample complexity.

\section{Preliminaries and Notation}
\label{sec:prelims}
We use the standard notion $\|\cdot\|_p$ for $\ell_p$ norms. For a matrix $A$, we use $\|A\|_p$ to denote the $\ell_p$ norm when vectorized. We let $\|A\|_{\mathsf{op}}$ denote the operator norm and let $I$ denote the identity matrix (where the dimension will be clear from context). Given a vector $\bm{v}\in \mathbb{R}^n$, we write $\text{diag}(\bm{v})$ to denote the diagonal matrix with $\bm{v}$ along the diagonal, and we write $\bm{v}^p\in \mathbb{R}^n$ for the vector obtained by applying $p$th powers entrywise to $\bm{v}$. We will also write $A_j$ for the $j$th row of $A$, and more generally, $A_{S,T}$ for the $\vert S\vert\times \vert T\vert$ submatrix indexed by $S\times T$. For a matrix $A$ with zero diagonal, we will often consider $A_j$ to be a vector in $\mathbb{R}^{n-1}$ (rather than $\mathbb{R}^n$) in the natural way by omitting the zero entry with the obvious indexing; this should not cause any confusion.

For a distribution $\nu$ on a discrete space $\mathcal{X}$ and a sub-sigma-algebra $\mathcal{F}\subseteq 2^{\mathcal{X}}$, we write $\nu(\cdot\vert \mathcal{F})$ to denote the conditional distribution with respect to $\mathcal{F}$. For a distribution $\mu$ on $\{-1,1\}^n$, we write $\text{Cov}(\mu)$ to denote the \emph{covariance matrix}:
\begin{equation*}
    \text{Cov}(\mu)_{ij} = \mathbb{E}_{\mu}[X_iX_j]-\mathbb{E}_{\mu}[X_i]\mathbb{E}_{\mu}[X_j].
\end{equation*}

Given random variables $Z_1,Z_2,\ldots,Z_k$ defined on a common probability space, we let $\sigma(Z_1,Z_2,\ldots,Z_k)$ denote the sigma-algebra generated by them. We write $\mathcal{N}(\mu,\sigma^2)$ to denote the Gaussian distribution with mean $\mu\in \mathbb{R}$ and variance $\sigma^2\geq 0$. We will require the standard tail bounds for  Gaussians:
\begin{fact}[see e.g. \cite{vershyninHighdimensionalProbabilityIntroduction2018}, Equation (2.10)]
\label{fact:gaussian_tails}
    Let $g\sim \mathcal{N}(0,1)$. Then for any $t\geq 0$,
    \begin{equation*}
        \Pr(g\geq t)\leq \exp\left(\frac{-t^2}{2}\right).
    \end{equation*}
\end{fact}
\begin{fact}[see e.g. \cite{vershyninHighdimensionalProbabilityIntroduction2018}, Theorem 3.1.1.]
\label{fact:norm_concentration}
    Let $\bm{g}=(g_1,\ldots,g_n)$ where each $g_i\sim \mathcal{N}(0,1)$ are independent. There exists a numerical constant $c_{\ref{fact:norm_concentration}}>0$ such for any $t\geq 0$,
    \begin{equation*} \Pr\left(\left\vert\frac{1}{\sqrt{n}}\|\bm{g}\|_2-1\right\vert\geq t\right)\leq 2\exp\left(-c_{\ref{fact:norm_concentration}}nt^2\right)
    \end{equation*}
\end{fact}

\subsection{Ising Models}
We consider Ising models on $\{-1,1\}^n$, so that $n$ parametrizes the number of sites. Given a symmetric interaction matrix $A\in \mathbb{R}^{n\times n}$ with zero diagonal and external fields $\bm{h}\in \mathbb{R}^n$, the corresponding Ising model $\mu_{A,\bm{h}}$ on $\{-1,1\}^n$ is given by
\begin{equation*}
    \mu_{A,\bm{h}}(\bm{x})\propto \exp\left(\frac{1}{2}\bm{x}^TA\bm{x}+\bm{h}^T\bm{x}\right),
\end{equation*}
where the normalization is to ensure $\mu_{A,\bm{h}}(\cdot)$ is a probability measure. Formally, we define the partition function $Z=Z_{A,\bm{h}}$ by
\begin{equation*}
    Z=\sum_{\bm{x}\in \{-1,1\}^n} \exp\left(\frac{1}{2}\bm{x}^TA\bm{x}+\bm{h}^T\bm{x}\right)
\end{equation*}
so that
\begin{equation*}
    \mu_{A,\bm{h}}(\bm{x})= \frac{\exp\left(\frac{1}{2}\bm{x}^TA\bm{x}+\bm{h}^T\bm{x}\right)}{Z}.
\end{equation*}
Given the interaction matrix $A$ and external fields $\bm{h}$, we define the $\ell_1$-width to be
\begin{equation*}
    \lambda=\lambda(A)\triangleq \max_{i\in [n]} \sum_{j\neq i} \vert A_{ij}\vert+\vert h_i\vert.
\end{equation*}
We say that the \textbf{support} of $A$, $\text{supp}(A)$, is the set of indices of nonzero entries of $A$ and we let $\mathcal{N}(i)$ denote the set of indices of nonzero entries in the $i$th row of $A$. The \textbf{degree} of $A$ is then the maximum of $\vert \mathcal{N}(i)\vert$ over $i\in [n]$ 
 (i.e. maximum number of nonzero entries in a row).

Given a configuration $X\in \{-1,1\}^n$, we write $X^{i,\pm}$ to denote the configuration that agrees with $X$ except $X_i$ is set to $\pm 1$. We also write $X_{-S}$ to denote the restriction of $X$ to the coordinates not in $S$. Note that given a partial configuration $X_{-S}\in \{-1,1\}^{-S}$, the conditional distribution of $X_S$ is given by 
\begin{equation*}
    \mu_A(X_S\vert X_{-S})\propto \exp\left(\frac{1}{2}X_S^TA_{S,S}X_S+X_S^TA_{S,-S}X_{-S}+\sum_{i\in S}h_iX_i\right),
\end{equation*}
where the implicit normalization sums over values of $X_S\in \{-1,1\}^S$.
In particular, if $S=\{i\}$ for some coordinate $i\in [n]$, then 
\begin{equation}
\label{eq:glauber}
    \mu_A(X_i = 1\vert X_{-i}) = \frac{\exp(\langle A_{i}, X_{-i}\rangle+h_i)}{\exp(\langle A_{i}, X_{-i}\rangle+h_i)+\exp(-(\langle A_{i}, X_{-i}\rangle+h_i)} = \sigma(2\cdot(\langle A_{i}, X_{-i}\rangle+h_i)),
\end{equation}
where we define $\sigma:\mathbb{R}\to [0,1]$ by 
\begin{equation*}
    \sigma(z)=\frac{1}{1+\exp(-z)}.
\end{equation*}
We will require the following deterministic fact about $\sigma(\cdot)$ that was shown by Klivans and Meka:
\begin{fact}[Claim 4.2 of \cite{DBLP:conf/focs/KlivansM17}]
\label{fact:sigmoid_lb}
For any $x,y\in \mathbb{R}$, it holds that
\begin{equation*}
    \vert \sigma(x)-\sigma(y)\vert\geq \frac{\exp(-\vert x\vert)}{4e}\min\{1,\vert x-y\vert\}.
\end{equation*}
\end{fact}
\begin{proof}
    We provide an alternative proof: we simply observe that
    \begin{equation*}
        \sigma'(z) = \frac{\exp(-z)}{(\exp(-z)+1)^2}\geq \frac{\exp(-\vert z\vert)}{4}.
    \end{equation*}
    The last inequality can be proved by cases when $z$ is positive or negative. The desired claim follows then by Taylor's theorem, using the  monotonicity of $\sigma(\cdot)$ to reduce to the case that $y=x\pm 1$. This picks up at most an extra factor of $e^{-1}$.
\end{proof}

In \Cref{sec:sk}, we will investigate the efficient learning of the \emph{Sherrington-Kirkpatrick} model:
\begin{defn}
    Let $\mathsf{GOE}(n)$ denote the Gaussian Orthogonal Ensemble (GOE): this is the distribution over symmetric $n\times n$ matrices where the entries above the diagonal are independent Gaussians with variance $1/n$ and zero diagonal. The \textbf{Sherrington-Kirkpatrick model} $\mu_{\beta,A,\bm{h}}$ on $n$ nodes at inverse temperature $\beta\geq 0$ and with external field $\bm{h}$ is the (random) Ising model with interaction matrix $\beta A$ where $A\sim \mathsf{GOE}(n)$.
\end{defn}

\subsection{Dynamics in Ising Models}

We now define a general class of dynamic processes on Ising models.
\begin{defn}
\label{defn:sbd}
    Let $A\in \mathbb{R}^{n\times n}$ and $\bm{h}\in \mathbb{R}^n$ parametrize an Ising model on $\{-1,1\}^n$. A \textbf{sequential block dynamics} for the corresponding Ising measure is given as follows: the dynamics initialize at some (possible random) $X_0\in \{-1,1\}^n$ and set $S_0=\emptyset$. At each time $t\geq 1$, letting $\mathcal{F}_t=\sigma(X_0,\ldots,X_t,S_0,\ldots,S_t)$, the block dynamics proceed as follows:
    
    \begin{enumerate}
        \item A (possibly) random subset $S_t$ is drawn from some measure on $2^{[n]}$ that, conditional on $\mathcal{F}_{t-1}$, depends only on $\mathcal{S}_{t-1}=(S_0,\ldots,S_{t-1})$.
        \item Given $\mathcal{F}_{t-1}$ and $S_t$, the law of $X_t$ is given by $X_t\sim\mu_{A,\bm{h}}(\cdot\vert (X_{t-1})_{-S_t})$.
    \end{enumerate} 
These dynamics thus give rise to a distribution over paths $\mathcal{P}=((S_0,X_0),(S_1,X_1),\ldots,(S_t,X_t),\ldots)\in (2^{[n]}\times \{-1,1\}^n)^{\mathbb{N}}$.
\end{defn}

While somewhat abstract, this definition easily captures several important natural examples:
\begin{example}[Independent Samples]
\label{ex:independent}
    Suppose that $S_t\equiv [n]$ for all $t\geq 1$. In this case, the sequence $X_1,X_2,\ldots$ form independent samples from $\mu_{A,\bm{h}}$, which is the most commonly studied regime for learning Ising models.
\end{example}
\begin{example}[Glauber Dynamics]
\label{ex:glauber}
    Suppose that $S_t$ is taken to be a uniform element of $[n]$, chosen independently of all other events. In this case, the sequential block dynamics forms (standard) \textbf{Glauber dynamics} and the updates are given via \Cref{eq:glauber}.
\end{example}

\begin{example}[$\ell$-Block Dynamics]
A common generalization of \Cref{ex:independent} and \Cref{eq:glauber} is $\ell$-block dynamics, where $1\leq \ell\leq n$. In this case, $S_t$ is chosen uniformly among all subsets of $[n]$ of size $\ell$. Glauber dynamics corresponds to $\ell=1$ while independent samples corresponds to $\ell=n$.
\end{example}

An even broader generalization of the previous examples is the following:
\begin{defn}
    Let $\mathcal{Q}$ denote a \emph{symmetric} distribution over non-empty subsets of $[n]$.\footnote{The symmetry of the distribution means that it is invariant under any permutation.} Then the \textbf{symmetric block dynamics induced by $\mathcal{Q}$} is the sequential block dynamics where each subset $S_t$ is drawn from $\mathcal{Q}$ independently of all previous events.
\end{defn}
This easily captures the previous examples by letting $\mathcal{Q}$ be uniform on subsets of $[n]$ with a fixed size. However, it also captures more general updates: for example, it captures the dynamics where each $i\in [n]$ is included in $S_t$ with a fixed probability $p$ conditional on $S_t$ being nonempty.

Note further that \Cref{defn:sbd} does not require reversibility; indeed, it also applies to non-reversible Markov chains like the following:
\begin{example}[Round-Robin Dynamics]
 \textbf{Round-robin dynamics} corresponds to choosing a fixed permutation $\sigma\in \mathfrak{S}_n$ and letting $S_t=\{\sigma(t\mod n)\}$, where we consider the modulo operation to return a number between $1$ and $n$. This is a natural updating rule in practice, and a major open problem is to determine the quantitative relation between the mixing times of round-robin dynamics and Glauber dynamics (Question 3 of \cite{levin2017markov}).
\end{example}
\begin{remark}
    Note that any \emph{fixed} deterministic schedule of subsets is valid under \Cref{defn:sbd}; the key assumption is that conditional on the past, each update of spins must have the correct law conditioned on the current configuration outside the subset to be updated. Moreover, the law of the blocks chosen for updating cannot depend on the history of configurations, only possibly on the previous schedule of updated blocks.
\end{remark}

\subsection{Learning Ising Models}

\begin{defn}[Distances between Distributions]
The \textbf{KL divergence} of distributions $P,Q$ on $\{-1,1\}^n$ is defined as
\begin{equation*}
    \mathsf{d}_{KL}(P,Q)=\mathbb{E}_P[\log(P/Q)].
\end{equation*}
    The \textbf{total variation distance} is defined as
\begin{equation*}
    \mathsf{d}_{TV}(P,Q) = \frac{1}{2}\sum_{x\in \{-1,1\}^n}\vert P(x)-Q(x)\vert.
\end{equation*}
\end{defn}

The following simple claim shows how to translate $\ell_{\infty}$ recovery bounds on an Ising model to distances of probability measures.

\begin{lem}
\label{lem:additive_suff}
    For any matrices $A,B\in \mathbb{R}^{n\times n}$ and external fields $\bm{h}_A,\bm{h}_B\in \mathbb{R}^n$
    \begin{equation*}
        \mathsf{d}_{KL}(\mu_{A,\bm{h}_A},\mu_{B,\bm{h}_B})\leq n^2\|A-B\|_{\infty}+2n\|\bm{h}_A-\bm{h}_B\|_{\infty}.
    \end{equation*}

    In particular, 
    \begin{equation*}
    \mathsf{d}_{TV}(\mu_{A,\bm{h}_A},\mu_{B,\bm{h}_B})\leq n\sqrt{\frac{\|A-B\|_{\infty}}{2}}+\sqrt{n\|\bm{h}_A-\bm{h}_B\|_{\infty}}.
    \end{equation*}
\end{lem}
\begin{proof}
    The second inequality is a consequence of the first by Pinsker's inequality and subadditivity of the square root function, so it suffices to prove the first.
Let $\Delta_1 = A-B$ and $\Delta_2 = \bm{h}_A-\bm{h}_B$ and set $\eta_1=\|\Delta_1\|_{\infty}$ and $\eta_2=\|\Delta_2\|_{\infty}$. Then for any $\sigma\in \{-1,1\}^n$, we have
    \begin{equation*}
        \frac{\exp\left(\sigma^TA\sigma/2+\bm{h}_A^T\sigma\right)}{\exp\left(\sigma^TB\sigma/2+ \bm{h}_B^T\sigma\right)}=\exp\left(\sigma^T\Delta_1\sigma/2+\Delta_2^T\sigma\right)\leq \exp(n^2\eta_1/2+n\eta_2),
    \end{equation*}
    and an identical bound for the reciprocal. This pointwise bound implies that 
    \begin{equation*}
        \exp(-n^2\eta_1/2-n\eta_2)Z_{B,\bm{h}_B}\leq Z_{A,\bm{h}_A}\leq \exp(n^2\eta_1/2+n\eta_2) Z_{B,\bm{h}_B}.
    \end{equation*}
The same calculation thus implies that
\begin{equation*}
    \frac{\mu_{A,\bm{h}_A}(\sigma)}{\mu_{B,\bm{h}_B}(\sigma)}=\frac{\exp\left(\sigma^TA\sigma/2+\bm{h}_A^T\sigma\right)}{\exp\left(\sigma^TB\sigma/2+ \bm{h}_B^T\sigma\right)}\cdot \frac{Z_{B,\bm{h}_B}}{Z_{A,\bm{h}_A}}\leq \exp(n^2\eta_1+2n\eta_2).
\end{equation*}
    
Plugging this trivial upper bound into the definition of KL divergence gives the result.
\end{proof}

We will consider the performance of neighborhood, constrained \emph{logistic regression} in learning Ising models. We first define how we set up the associated regression samples for a given vertex $i$ under any sequential block dynamics.

\begin{defn}
\label{defn:view}
    Let $A\in \mathbb{R}^{n\times n}$ be an interaction matrix, $\bm{h}\in \mathbb{R}^n$ be an external field, and $\mu_{A,\bm{h}}$ be the corresponding Ising model. Let $i\in [n]$ be any fixed node. For any sequential block dynamics, let $\tau_1,\tau_2,\ldots\in \mathbb{N}$ be stopping times defined recursively via:
    \begin{gather*}
        \tau_1=\inf\{t\in \mathbb{N}:i\in S_t\}\\
        \tau_{s+1}=\inf\{t> \tau_s:i\in S_t\}.
    \end{gather*}
    Then the \textbf{generated samples for node $i$} are the tuples $Z_1=(X_{\tau_1,-i},X_{\tau_1,i}),\ldots, Z_t = (X_{\tau_t,-i},X_{\tau_t,i})$.
\end{defn}

Note that in \Cref{defn:view}, sampling the coordinates in block $S_{\tau_j}$ can equivalently be done by ``deferred decisions'': first, sample the coordinates in $S_{\tau_j}\setminus \{i\}$ conditional on the coordinates outside $S_{\tau_j}$ according to $\mu_{A,\bm{h}}$, and then sample the spin at $i$ conditional on all other coordinates in the configuration as in \Cref{eq:glauber}. In particular, the law of $X_{\tau_j,i}$ conditional on $Z_1,\dots,Z_{j-1}, X_{\tau_j,-i}$ is exactly $\mu_A(\cdot\vert X_{\tau_j,-i})$.

We consider the performance of the constrained logistic regression problem for any vertex $i$. The \textbf{logistic loss} $\ell:\mathbb{R}\to \mathbb{R}$ is defined by 
\begin{equation*}
    \ell(z)=\log\left(1+\exp(-z)\right).
\end{equation*}
The following simple fact is well-known and can be obtained by differentiation.
\begin{fact}
\label{fact:ell_facts}
    $\ell$ is $1$-Lipschitz.
\end{fact}
Given a compact, convex set $\mathcal{C}\in \mathbb{R}^{n-1}\times \mathbb{R}$, and given a trajectory of sequential  block dynamics as in \Cref{defn:sbd}, we define the $\mathcal{C}$-constrained logistic regression problem for node $i$ as:

\begin{equation}
\label{eq:empiricalopt}
(\widehat{\bm{w}},\widehat{h})\triangleq \arg\min_{(\bm{w},h)\in \mathcal{C}} \frac{1}{T} \sum_{k=1}^T \ell\left(2\cdot X_{\tau_k,i}\cdot(\langle \bm{w},X_{\tau_k,-i}\rangle+h)\right)   
\end{equation}

To obtain an estimator $(\widehat{A},\widehat{\bm{h}})\in \mathbb{R}^{n\times n}\times \mathbb{R}^n$, we set for each $j\in [n]$,
\begin{gather*}
    \widehat{A}'_j = \widehat{\bm{w}}_j\\
    \bm{\widehat{h}}_j = \widehat{h}_j,
\end{gather*}
where $(\widehat{\bm{w}}_j,\widehat{h}_j)\in \mathbb{R}^{n-1}\times \mathbb{R}$ is the output of the analogous regression problem for node $j\in [n]$ using the samples generated for node $j$. We then set
\begin{equation*}
    \widehat{A} = \frac{\widehat{A}'+\widehat{A}'^T}{2}.
\end{equation*}
That is, we just symmetrize our estimates for each interaction coefficient. It is clear that if $\|\widehat{\bm{w}}_j-A_j\|_{\infty}\leq \varepsilon$ for each $j\in [n]$, then $\|A-\widehat{A}\|_{\infty}\leq \varepsilon$, so we will focus on proving the additive guarantee for each logistic regression problem separately.

For any distribution $\mathcal{D}$ on $(X,Y)\in \{-1,1\}^d\times \{-1,1\}$ such that there exists $(\bm{w}^*,h^*)\in \mathbb{R}^{d}\times \mathbb{R}$ with 
\begin{equation}
\label{eq:faithful_2}
    \Pr_{\mathcal{D}}(Y=1\vert X)=\sigma(2\cdot (\langle \bm{w}^*,X\rangle+h^*)),
\end{equation}
define the \textbf{population logistic loss} of $(\bm{w},h)\in \mathbb{R}^{d}\times \mathbb{R}$  via
\begin{equation}
\label{eq:pop_loss}
    \mathcal{L}(\bm{w},h)=\mathcal{L}_{\mathcal{D}}(\bm{w},h)\triangleq \mathbb{E}_{(X,Y)\sim \mathcal{D}}\left[\ell(2\cdot Y\cdot (\langle \bm{w},X\rangle+h))\right].
\end{equation}
We will require the following simple fact from Wu, Sanghavi, and Dimakis:
\begin{lem}[Lemma 8 and 9 of \cite{DBLP:conf/nips/WuSD19}]
\label{lem:loss_quadratic}
    For any $(\bm{w},h)\in \mathbb{R}^d\times \mathbb{R}$ and $\mathcal{D}$ satisfying \Cref{eq:faithful_2}, 
    \begin{equation*}
        \mathcal{L}_{\mathcal{D}}(\bm{w},h)-\mathcal{L}_{\mathcal{D}}(\bm{w}^*,h^*)\geq 2\cdot \mathbb{E}_{(X,Y)\sim \mathcal{D}}\left[\left(\sigma(2\cdot (\langle \bm{w},X\rangle+h))-\sigma(2\cdot (\langle \bm{w}^*,X\rangle+h^*)) \right)^2\right].
    \end{equation*}
\end{lem}
    In words, the population logistic loss is minimized at the true coefficient vector, and the difference between losses with the true terms is at least the quadratic difference in the \emph{predictions} of $Y$ between the two coefficient vectors.

\section{Ising Learning from General Samples}
\label{sec:general}

In this section, we state and prove the main meta-theorem of Wu, Sanghavi, and Dimakis~\cite{DBLP:conf/nips/WuSD19} in a slightly abstracted and generalized form that will allow us to apply the analysis to a wide variety of Ising models and dynamics. The only technical difference is we also provide a simple argument that sufficiently low-error recovery of $\bm{w}^*$ also implies approximate recovery of $h^*$ under their analysis.

\begin{thm}[Implicit in Wu-Sanghavi-Dimakis~\cite{DBLP:conf/nips/WuSD19}]
\label{thm:wsd}
Let $\mathcal{C}\subseteq \mathbb{R}^d\times \mathbb{R}$ be a closed, convex set with $d\geq 1$. Let $Z_1=(X_1,Y_1),\ldots,Z_T=(X_T,Y_T),\ldots\in \{-1,1\}^d\times \{-1,1\}$ be an arbitrary stochastic process such that the law of $Y_t$ conditional on $(Z_1,\ldots,Z_{t-1},X_t)$ is $+1$ with probability $\sigma(2(\langle \bm{w}^*,X_t\rangle+h^*))$ and $-1$ otherwise, where $(\bm{w}^*,h^*)\in \mathcal{C}$ is a fixed vector. Suppose that the following conditions hold:

\begin{enumerate}
    \item There exists a function $T: [0,1]\times [0,1]\to \mathbb{N}$ with the following guarantee: given $\varepsilon,\delta>0$, if $T\geq T(\varepsilon,\delta)$, then with probability at least $1-\delta$, the following two events hold:
    \begin{equation}
    \label{eq:convergence}
         \frac{1}{T}\sum_{t=1}^T \mathbb{E}_{Z_t}\left[\ell(2Y_t(\langle \bm{w},X_t\rangle+h))\vert Z_1,\ldots,Z_{t-1}\right]-\frac{1}{T}\sum_{t=1}^T \ell(2Y_t(\langle \bm{w},X_t\rangle+h)) \leq \varepsilon \quad\quad\forall (\bm{w},h)\in \mathcal{C},
    \end{equation}
    \begin{equation}
        \label{eq:opt_deviation}
        \left\vert\frac{1}{T}\sum_{t=1}^T \ell(2Y_t(\langle \bm{w}^*,X_t\rangle+h^*)-\frac{1}{T}\sum_{t=1}^T \mathbb{E}_{Z_t}\left[\ell(2Y_t(\langle \bm{w}^*,X_t\rangle+h^*))\vert Z_1,\ldots,Z_{t-1}\right] \right\vert\leq \varepsilon.
    \end{equation}

    \item There exists a constant $c_{\ref{thm:wsd}}>0$ such that the following holds for all $i\in [d]$: for all $(\bm{w},h)\in \mathcal{C}$,
    \begin{equation}
    \label{eq:inf_lb}
        \mathbb{E}_{Z_t}\left[\left(\sigma\left(2(\langle \bm{w},X_t\rangle+h)\right)-\sigma\left(2(\langle \bm{w}^*,X_t\rangle+h^*)\right)\right)^2\bigg\vert Z_1,\ldots,Z_{t-1}\right]\geq c_{\ref{thm:wsd}}\min\{1,8(w_i-w_i^*)^2\}
    \end{equation}
    almost surely.\footnote{For comparison to the conditions that appear in \cite{DBLP:conf/nips/WuSD19}, slightly different constants appear in due to minor differences in normalization.}
\end{enumerate}
    Define an empirical risk minimizer $(\hat{\bm{w}},\hat{h})\in \mathcal{C}$ by
    \begin{equation*}
        (\hat{\bm{w}},\hat{h})\in \arg\min_{(\bm{w},h)\in \mathcal{C}} \frac{1}{T}\sum_{t=1}^T \ell(2\cdot Y_t\cdot(\langle \bm{w},X_t\rangle+h)).
    \end{equation*}
    Then for any $\delta>0$ and $\varepsilon\in (0,1/\sqrt{8}]$, if $T\geq T(8c_{\ref{thm:wsd}}\varepsilon^2,\delta)$, then it holds with probability at least $1-\delta$ over the realization of $Z_1,\ldots,Z_T$ that
    \begin{equation*}
        \|\widehat{\bm{w}}-\bm{w}^*\|_{\infty}\leq \varepsilon.
    \end{equation*}

    Suppose that it further holds that for any $(\bm{w},h)\in \mathcal{C}$ such that $(d+1)\cdot \|\bm{w}-\bm{w}^*\|_{\infty}\leq \vert h-h^*\vert/2$, we have
    \begin{equation}
        \label{eq:constant_lb}
        \mathbb{E}_{Z_t}\left[\left(\sigma\left(2(\langle \bm{w},X_t\rangle+h)\right)-\sigma\left(2(\langle \bm{w}^*,X_t\rangle+h^*)\right)\right)^2\bigg\vert Z_1,\ldots,Z_{t-1}\right]\geq c_{\ref{thm:wsd}}\min\{1,(h-h^*)^2\}
    \end{equation}
    Then it must also hold on this event that $\vert \hat{h}-h^*\vert\leq 2(d+1)\varepsilon$.
\end{thm}
\begin{proof}
    This first statement is a mild generalization of the argument of Theorem 1 of Wu, Sanghavi, and Dimakis \cite{DBLP:conf/nips/WuSD19}; we provide the argument for completeness. For $(\bm{w},h)\in \mathcal{C}$, let
    \begin{gather*}
        \widehat{\mathcal{L}}(\bm{w},h)\triangleq \frac{1}{T}\sum_{t=1}^T \ell(2\cdot Y_t\cdot(\langle \bm{w},X_t\rangle+h))\\
        \mathcal{L}_T(\bm{w},h)\triangleq \frac{1}{T}\sum_{t=1}^T \mathbb{E}_{Z_t}[\ell(2\cdot Y_t\cdot(\langle \bm{w},X_t\rangle+h))\vert \mathcal{F}_{t-1}]=\frac{1}{T}\sum_{t=1}^T \mathcal{L}_{Z_t\vert \mathcal{F}_{t-1}}(\bm{w},h),
    \end{gather*}
    where $\mathcal{F}_{t-1}=\sigma(Z_1,\ldots,Z_{t-1})$.
    Given $T\geq T(8c_{\ref{thm:wsd}}\varepsilon^2,\delta)$ samples, the condition above ensures 
    that with probability at least $1-\delta$, the following inequalities hold:
    \begin{equation}
    \label{eq:emp_opt_close}
\mathcal{L}_T(\widehat{\bm{w}},\widehat{h})\leq \widehat{\mathcal{L}}(\widehat{\bm{w}},\widehat{h})+8c_{\ref{thm:wsd}}\varepsilon^2\leq \widehat{\mathcal{L}}(\bm{w}^*,h^*)+8c_{\ref{thm:wsd}}\varepsilon^2\leq \mathcal{L}_T(\bm{w}^*,h^*)+16c_{\ref{thm:wsd}}\varepsilon^2.
    \end{equation}
    The first inequality follows the guarantee of \Cref{eq:convergence} given at least $T(8c_{\ref{thm:wsd}}\varepsilon^2,\delta)$ samples, the second by optimality of $(\widehat{\bm{w}},\widehat{h})$ with respect to $\widehat{\mathcal{L}}$, and the third by applying \Cref{eq:opt_deviation} with the true parameters $(\bm{w}^*,h^*)$. For the first step, it is crucial that the convergence of \Cref{eq:convergence} is uniform over $\mathcal{C}$ to apply it to the empirical optimizer.
    
    Combining the inequality \Cref{eq:emp_opt_close} with the identity given in \Cref{lem:loss_quadratic} and the assumed lower bound \Cref{eq:inf_lb}, it holds for any $i\in [d]$ that
    \begin{align*}
        8c_{\ref{thm:wsd}}\varepsilon^2&\geq \frac{\mathcal{L}_T(\widehat{\bm{w}},\widehat{h})-\mathcal{L}_T(\bm{w}^*,h^*)}{2}\\
        &= \frac{1}{2T}\sum_{t=1}^T \mathcal{L}_{Z_t\vert \mathcal{F}_{t-1}}(\bm{w},h)-\mathcal{L}_{Z_t\vert \mathcal{F}_{t-1}}(\bm{w}^*,h^*)\\
        &\geq\frac{1}{T} \sum_{t=1}^T \mathbb{E}_{X_t,Y_t}\left[\left(\sigma(2\cdot (\langle \bm{w},X_t\rangle+h))-\sigma(2\cdot (\langle \bm{w}^*,X_t\rangle+h^*)) \right)^2\bigg\vert \mathcal{F}_{t-1}\right]\\
        &\geq c_{\ref{thm:wsd}}\min\{1,8(\widehat{w}_i-w_i^*)^2\}.
    \end{align*}
    Since $8\varepsilon^2\leq 1$ by our assumption $\varepsilon\leq 1/\sqrt{8}$, the minimum is attained by the difference in $i$th components, and so we obtain $\vert \widehat{w}_i-w_i^*\vert\leq \varepsilon$. Therefore, $\|\widehat{w}-w^*\|_{\infty}\leq \varepsilon$ with probability at least $1-\delta$, as claimed.

    We now show how one gets approximate recovery of $h^*$ on this same event under the extra condition \Cref{eq:constant_lb}. Suppose that $\vert \widehat{h}-h^*\vert>2(d+1)\varepsilon\geq 2(d+1)\|\widehat{\bm{w}}-\bm{w}^*\|_{\infty}$ on this event so that \Cref{eq:constant_lb} attains for the vector $(\bm{\widehat{w}},\widehat{h})$. Then the same computation and \Cref{eq:constant_lb} would imply that
    \begin{equation*}
        8c_{\ref{thm:wsd}}\varepsilon^2\geq c_{\ref{thm:wsd}} (\widehat{h}-h^*)^2 \implies 8\varepsilon^2> 4(d+1)^2\varepsilon^2.
    \end{equation*}
    This is a contradiction for any $d\geq 1$ (which is the only nontrivial case).
\end{proof}

\begin{remark}
The results of Wu, Sanghavi, and Dimakis \cite{DBLP:conf/nips/WuSD19} and Klivans and Meka \cite{DBLP:conf/focs/KlivansM17} do not appear to explicitly provide guarantees about $\vert h^*-\hat{h}\vert_{\infty}$. Their technique relies on almost sure conditional variation in each spin component to deduce parameter recovery, which trivially fails for the constant term (since this can be viewed as a spin fixed to $+1$). The argument above shows that assuming $\bm{w}^*$ has been recovered to very high accuracy, necessarily $h^*$ must have been recovered to decent accuracy as well. It would be interesting to see if there are more direct recovery guarantees for the external field that remove the polynomial blow-up in the additive error, since this imposes $\mathsf{poly}(n)$ sample complexity rather than $\log(n)$ sample complexity even for constant additive error.
\end{remark}

\section{Learning from Sequential Block Dynamics}
\label{sec:dynamics}
In this section, we show how to verify the conditions of \Cref{thm:wsd} under quite general settings of sequential block dynamics under $\ell_1$ width conditions. Our main result, specialized to the concrete block dynamics mentioned above, is:

\begin{thm}
\label{thm:learning_block}
    For any $\varepsilon,\delta>0$, the following holds: suppose that we observe samples $Z_1,\ldots,Z_T\in \{-1,1\}^n$ generated from any symmetric block dynamics or round robin dynamics for an Ising model $\mu_{A,\bm{h}}$ with width $\lambda$ with any initial configuration. Suppose that $T$ is such that each site is updated at least
    \begin{equation*} \tilde{O}\left(\frac{\lambda^2 \exp(O(\lambda))\log(n/\delta)}{\varepsilon^4}\right)
    \end{equation*}
    times. For each node $i\in [n]$, solving the empirical risk minimization problem as in \Cref{eq:empiricalopt} yields an estimator $(\widehat{A},\widehat{\bm{h}})\in \mathbb{R}^{n\times n}\times \mathbb{R}^n$ such that $\|A-\widehat{A}\|_{\infty}\leq \varepsilon$ with probability at least $1-\delta$.
\end{thm}

We will obtain this theorem as a simple consequence of a more general result (\Cref{prop:sufficiecy_lb}) that provides easily verified conditions for a sequential block dynamics to be efficiently learnable. As described in \Cref{sec:overview}, there are several technical difficulties in extending the analysis of \cite{DBLP:conf/nips/WuSD19} as in \Cref{thm:wsd} under independent samples to block dynamics.

In \Cref{sec:martingale}, we show how the powerful theory of \emph{sequential Rademacher complexity}, as introduced by Rakhlin, Sridharan, and Tewari \cite{DBLP:journals/jmlr/RakhlinST15} in the context of online learning, suffices to obtain near-optimal uniform convergence guarantees. Our main technical tool, \Cref{thm:radbound}, is quite general and holds for arbitrary stochastic processes. In \Cref{sec:population_loss}, we turn to the problem of establishing lower bounds on the difference in population losses for well-behaved block dynamics as in \Cref{eq:inf_lb} by carefully arguing that each coordinate in the Ising model maintains sufficient variability conditional on a \emph{typical} path of the dynamics between snapshots of the dynamics at each update time of the node under consideration. We then conclude \Cref{thm:learning_block} by appealing to \Cref{thm:wsd} with these intermediate results.

\subsection{Uniform Martingale Bounds}
\label{sec:martingale}

In this section, we turn to proving uniform martingale bounds of the form required in \Cref{eq:convergence}. We again consider the following general setting: suppose that $Z_1=(X_1,Y_1),\ldots,Z_T=(X_T,Y_T)$ is an arbitrary stochastic process in $\{-1,1\}^d\times \{-1,1\}$. Our goal is to provide uniform deviation bounds for empirical averages as in \Cref{eq:empiricalopt} to the \emph{conditional, path-dependent expected averages}. Clearly, we have that for any fixed $(\bm{w},h)\in \mathbb{R}^{d}\times \mathbb{R}$, the random deviation of
\begin{equation*}
    \frac{1}{T} \sum_{t=1}^T \ell\left(2\cdot Y_t(\langle \bm{w},X_t\rangle+h)\right) -\frac{1}{T}\sum_{t=1}^T \mathbb{E}_{Z_t}\left[\ell\left(2\cdot Y_t(\langle \bm{w},X_t\rangle+h)\rangle\right)\bigg\vert Z_1,\ldots,Z_{t-1}\right]
\end{equation*}
satisfies strong concentration bounds. This is a consequence of the Azuma-Hoeffding inequality for martingale differences, assuming upper bounds on the arguments of the loss functions. As we will require two-sided bounds on this quantity for the true parameters, we record this formally below:

\begin{lem}
\label{lem:opt_azuma}
    Fix any $(\bm{w},h)\in \mathbb{R}^d\times \mathbb{R}$ and suppose $\|(\bm{w},h)\|_1\leq \lambda$. Then for any stochastic process $Z_1=(X_1,Y_1),\ldots,Z_T=(X_T,Y_T)\in \{-1,1\}^d\times \{-1,1\}$, it holds that
    \begin{equation*}
        \Pr\left(\left\vert\frac{1}{T} \sum_{t=1}^T \ell\left(2\cdot Y_t(\langle \bm{w},X_t\rangle+h)\right) -\frac{1}{T}\sum_{t=1}^T \mathbb{E}_{Z_t}\left[\ell\left(2\cdot Y_t(\langle \bm{w},X_t\rangle+h)\rangle\right)\bigg\vert Z_1,\ldots,Z_{t-1}\right]\right\vert \geq \varepsilon\right)\leq 2\exp\left(\frac{-T\varepsilon^2}{32\lambda^2}\right).
    \end{equation*}
\end{lem}
\begin{proof}
    This is an immediate consequence of the Azuma-Hoeffding inequality (see, e.g. Corollary 2.20 of \cite{wainwright_2019}), since we have
    \begin{align*}
        &\left\vert \ell\left(2\cdot Y_t(\langle \bm{w},X_t\rangle+h)\right) -\mathbb{E}_{Z_t}\left[\ell\left(2\cdot Y_t(\langle \bm{w},X_t\rangle+h)\rangle\right)\bigg\vert Z_1,\ldots,Z_{t-1}\right]\right\vert\\
        &\leq \max_{X_t,Y_t,X_t',Y_t'}\left\vert \ell\left(2\cdot Y_t(\langle \bm{w},X_t\rangle+h)\right) - \ell\left(2\cdot Y'_t(\langle \bm{w},X'_t\rangle+h)\right)\right\vert\\
        &\leq 2\max_{X_t,Y_t,X_t',Y_t'}\left(\left\vert \langle \bm{w},X_t-X_t'\rangle\right\vert +2\vert h\vert\right)\\
        &\leq 4\lambda.
    \end{align*}
    Here, we replace the expectation with the largest feasible choice of $X_t',Y_t'$, and then we apply \Cref{fact:ell_facts} and H\"older's inequality. Since these provide uniform bounds on the almost sure behavior of each term in the martingale difference sequence, Azuma-Hoeffding provides the conclusion.
\end{proof}

However, we also require \emph{uniform} bounds on the one-sided version of this bound over all feasible regressors. This can be done by one-step chaining and union bounding a net, but will incur sample complexity $\mathsf{poly}(n)$ simply because the log-covering number is proportional to $n$. To recover near logarithmic sample complexity per node, we will appeal to stronger uniform martingale tail bounds. In particular, the main result of this section is the following:

\begin{thm}
\label{thm:hpbounds}
Let $\mathcal{X}=\{-1,1\}^d\times \{-1,1\}$ and $\mathcal{F}$ be the set of functions $(\mathbf{x},y)\in \mathcal{X}\mapsto - \ell(2y(\langle w,\mathbf{x}\rangle+h))$ over $(\bm{w},h)\in \mathbb{R}^d\times \mathbb{R}$ such that $\|(\bm{w},h)\|_1\leq \lambda$. 
Then there exists an absolute constant $C>0$ such that for any stochastic process $Z_1=(X_1,Y_1),\ldots,Z_T=(X_T,Y_T)\in \mathcal{X}$, it holds that
\begin{equation*}
    \Pr\left(\sup_{f\in \mathcal{F}} \sum_{t=1}^T f(Z_t)-\mathbb{E}[f(Z_t)\vert Z_1,\ldots, Z_{t-1}] \geq C\lambda\log^{3/2}(T)\sqrt{T\log d}+u\right)\leq \exp\left(-\frac{u^2}{64\lambda^2T}\right)
\end{equation*}

In particular, there exists an absolute constant $C_{\ref{thm:hpbounds}}>0$ such that it holds with probability at least $1-\delta$ that
\begin{equation*}
    \sup_{w\in \mathbb{R}^d:\|\bm{w},h\|_1\leq \lambda} \frac{1}{T}\sum_{t=1}^T\mathbb{E}_{X_t,Y_t}[\ell(Y_t\langle w,X_t\rangle)\vert Z_1,\ldots, Z_{t-1}] -\frac{1}{T}\sum_{t=1}^T \ell(Y_t\langle w,X_t\rangle)\leq C_{\ref{thm:hpbounds}}\lambda\sqrt{\frac{\log^{3}(T)\log (d/\delta)}{T}}.
\end{equation*}
\end{thm}

To prove this tail bound, we recall the following definition from Rakhlin, Sridharan, and Tewari \cite{DBLP:journals/jmlr/RakhlinST15}: 

\begin{defn}[Sequential Rademacher Complexity]
  Let $\varepsilon_1,\ldots,\varepsilon_n\in \{-1,1\}$ denote a sequence of i.i.d. Rademacher random variables\footnote{That is, each $\varepsilon_i$ is uniform on $\{-1,1\}$.} and let $\mathcal{F}$ be a class of functions from some space $\mathcal{X}$ into $\mathbb{R}$. An \textbf{$\mathcal{X}$-valued process predictable with respect to the dyadic filtration} is a sequence of functions $\bm{z}=(\bm{z}_1,\ldots,\bm{z}_n)$ such that $\bm{z}_t:\{-1,1\}^{t-1}\to \mathcal{X}$ is a function of $\varepsilon_1,\ldots,\varepsilon_{t-1}$. Then the \textbf{sequential Rademacher complexity} of $\mathcal{F}$ is defined to be
\begin{equation*}
    \mathcal{R}_n(\mathcal{F})=\sup_{\mathbf{z}} \mathbb{E}\left[\sup_{f\in \mathcal{F}}\sum_{t=1}^n \varepsilon_t f(\mathbf{z}_t(\varepsilon_{1:t-1}))\right].
\end{equation*}  
\end{defn}

In another slight abuse of notation, we define $\mathcal{F}-\mathcal{G}$ as the set of functions on $\mathcal{X}^2$ to $\mathbb{R}$ by $f(x_1)-g(x_2)$ for $f\in \mathcal{F}$ and $g\in \mathcal{G}$ and $(x_1,x_2)\in \mathcal{X}^2$.
The following theorems are key results of Rakhlin and Sridharan:
\begin{thm}[Corollary 11 of \cite{DBLP:conf/colt/RakhlinS17}]
\label{thm:reduction}
    For any function class $\mathcal{F}$ mapping $\mathcal{X}\to \mathbb{R}$, if for every predictable $\mathcal{X}^2$-valued process $\mathbf{z}$ with respect to the dyadic filtration, the following tail bound holds:
    \begin{equation*}
        \Pr_{\bm{\varepsilon}}\left(\sup_{h\in \mathcal{F}- \mathcal{F}} \sum_{t=1}^n\varepsilon_t h(\mathbf{z}_t)\geq \mathcal{R}_n(\mathcal{F}-\mathcal{F})+u\right)\leq C\exp(-\mu(u)),
    \end{equation*}
    for some function $\mu:\mathbb{R}\to \mathbb{R}$, then for any stochastic process $Z_1,\ldots,Z_n\in \mathcal{X}$, it holds that
    \begin{equation*}
        \Pr_{Z_1,\ldots,Z_n}\left(\sup_{f\in \mathcal{F}} \sum_{t=1}^n f(Z_t)-\mathbb{E}[f(Z_t)\vert Z_1,\ldots Z_{t-1}]\geq \mathcal{R}_n(\mathcal{F}-\mathcal{F})+u\right)\leq C\exp(-\mu(u-\mu^{-1}(1))).
    \end{equation*}
\end{thm}

\begin{thm}[Corollary 9 of \cite{DBLP:conf/colt/RakhlinS17}]
\label{thm:radbound}
    For any function class $\mathcal{F}$ from $\mathcal{X}$ into $\mathbb{R}$ and any $\mathcal{X}$-valued predictable process $\mathbf{x}$ with respect to the dyadic filtration, it holds that 
    \begin{equation*}
        \Pr_{\bm{\varepsilon}}\left(\sup_{f\in \mathcal{F}} \sum_{t=1}^n \varepsilon_tf(\mathbf{x}_t)\geq \mathcal{R}_n(\mathcal{F})+u\right)\leq \exp\left(\frac{-u^2}{4\sup_{\mathbf{z}}\sum_{t=1}^n \sup_{f\in \mathcal{F}} f(\mathbf{z}_t)^2}\right).
    \end{equation*}
\end{thm}

With these results stated, it is straightforward to deduce \Cref{thm:hpbounds}:
\begin{proof}[Proof of \Cref{thm:hpbounds}]
The last claim is a simple consequence of the tail bound after setting $u$ appropriately and simplifying, so we focus on the first statement. 
    We first apply \Cref{thm:radbound} for the class $\mathcal{G}=\mathcal{F}-\mathcal{F}$ to obtain that for any $\mathcal{X}^2$-valued predictable process $\mathbf{z}$ with respect to the dyadic filtration, it holds that
    \begin{equation*}
        \Pr\left(\sup_{f\in \mathcal{F}-\mathcal{F}} \sum_{t=1}^T \varepsilon_tf(\mathbf{z}_t)\geq \mathcal{R}_T(\mathcal{F}-\mathcal{F})+u\right)\leq \exp\left(\frac{-u^2}{4\sup_{\mathbf{z}}\sum_{t=1}^T \sup_{f\in \mathcal{F}-\mathcal{F}} f(\mathbf{z}_t)^2}\right)
    \end{equation*}
Note that for any $g\in \mathcal{F}-\mathcal{F}$, $g(\mathbf{z})^2\leq 16\lambda^2$ by the same calculation in the proof of \Cref{lem:opt_azuma} using the Lipschitz property. In particular, the tail bound holds with $\mu(u)=u^2/64\lambda^2T$. Applying \Cref{thm:reduction}, it therefore holds that for any stochastic process $Z_1=(X_1,Y_1),\ldots,Z_{T}=(X_T,Y_T)\in \mathcal{X}$,
\begin{equation*}
        \Pr\left(\sup_{f\in \mathcal{F}} \sum_{t=1}^T f(Z_t)-\mathbb{E}[f(Z_t)\vert Z_1,\ldots, Z_{t-1}] \geq \mathcal{R}_T(\mathcal{F}-\mathcal{F})+u\right)\leq \exp\left(-\frac{(u-8\lambda\sqrt{T})^2}{64\lambda^2 T}\right).
    \end{equation*}

It remains to estimate $\mathcal{R}_{T}(\mathcal{F}-\mathcal{F})$. It is easy to see that $\mathcal{R}_T(\mathcal{F}-\mathcal{F})\leq 2\mathcal{R}_T(\mathcal{F})$ since
\begin{align*}
    \mathcal{R}_T(\mathcal{F}-\mathcal{F})&=\sup_{\bm{z}=(\bm{z}^1,\bm{z}^2)}\mathbb{E}\left[\sup_{f_1,f_2\in \mathcal{F}}\sum_{t=1}^T \varepsilon_t (f_1(\mathbf{z}^1_t(\varepsilon_{1:t-1}))-f_2(\mathbf{z}^2_t(\varepsilon_{1:t-1})))\right]\\
    &\leq \sup_{\bm{z}=(\bm{z}^1,\bm{z}^2)}\mathbb{E}\left[\sup_{f_1\in \mathcal{F}}\sum_{t=1}^T \varepsilon_t (f_1(\mathbf{z}^1_t(\varepsilon_{1:t-1})))+\sup_{f_2\in \mathcal{F}}\sum_{t=1}^n\varepsilon_t(-f_2(\mathbf{z}^2_t(\varepsilon_{1:t-1})))\right]\\
    &=\sup_{\bm{z}=(\bm{z}^1,\bm{z}^2)}\mathbb{E}\left[\sup_{f_1\in \mathcal{F}}\sum_{t=1}^T \varepsilon_t (f_1(\mathbf{z}^1_t(\varepsilon_{1:t-1})))\right]+\mathbb{E}\left[\sup_{f_2\in \mathcal{F}}\sum_{t=1}^T\varepsilon_t(-f_2(\mathbf{z}^2_t(\varepsilon_{1:t-1})))\right]\\
    &=2\sup_{\bm{z}^1}\mathbb{E}\left[\sup_{f_1\in \mathcal{F}}\sum_{t=1}^n \varepsilon_t (f_1(\mathbf{z}^1_t(\varepsilon_{1:t-1})))\right]\\
    &=2\mathcal{R}_T(\mathcal{F}),
\end{align*}
where we first use the definition, the the triangle inequality, linearity of expectation, and the fact that the distribution of $\bm{\varepsilon}$ is invariant under negation.

Recall that by definition $\mathcal{F}=-\ell\circ \mathcal{B}$ where $\mathcal{B}$ is the set of maps $(\mathbf{x},y)\mapsto 2y(\langle \bm{w},\mathbf{x}\rangle+h)$ for $\|(\bm{w},h)\|_1\leq \lambda$. Because $-\ell$ is 1-Lipschitz by \Cref{fact:ell_facts}, the Lipschitz contraction principle for sequential Rademacher complexity of Rakhlin, Sridharan, and Tewari (Corollary 5 of \cite{DBLP:journals/jmlr/RakhlinST15}) implies that there is a constant $C>0$ such that\footnote{We note that the factor $3/2$ can be improved to $1$ using more refined results of Block, Dagan, and Rakhlin \cite{DBLP:conf/colt/BlockDR21} by combining Corollary 10 and Proposition 15 of their work with known bounds on the fat-shattering dimension of linear functions. However, their general Lipschitz contraction principle does not appear to apply to this class, so we do not know whether the logarithmic factor can be removed altogether.}
\begin{equation*}
    \mathcal{R}_T(\mathcal{F})\leq C\log^{3/2}(T)\mathcal{R}_T(\mathcal{B}).
\end{equation*}
Finally, we can easily compute this last Rademacher complexity: for any $\mathcal{X}$-valued predictable process $\mathbf{z}$, it holds by $\ell_p$ duality that
\begin{align*}   \mathbb{E}_{\varepsilon}\left[\sup_{\|(\bm{w},h)\|_1\leq \lambda}2\sum_{t=1}^T\varepsilon_ty_t(\varepsilon_{1:t-1}) (\langle \bm{w},\mathbf{x}_t(\varepsilon_{1:t-1})\rangle+h)\right]&=2\mathbb{E}_{\varepsilon}\left[\sup_{\|(\bm{w},h)\|_1\leq \lambda}\left\langle (\bm{w},h),\sum_{t=1}^T \varepsilon_ty_t(\varepsilon_{1:t-1})(\bm{x}_t(\varepsilon_{1:t-1}),1)\right\rangle \right]\\
&=2\lambda \mathbb{E}_{\varepsilon}\left[\left\|\sum_{t=1}^T\varepsilon_t (y_t(\varepsilon_{1:t-1})\mathbf{x}_t(\varepsilon_{1:t-1}),y_t(\varepsilon_{1:t-1}))\right\|_{\infty}\right]
\end{align*}
Now, note that since $(y_t\mathbf{x}_t,y_t)\in \{-1,1\}^{d+1}$ is predictable with respect to the dyadic filtration, each coordinate above forms a simple random walk on $\mathbb{Z}$ with $T$ steps. We now just use the following well-known bound: if $Y_1,\ldots,Y_d$ are $\sigma^2$-sub-Gaussian, then $\mathbb{E}[\max_{i\in [n]} Y_i]\leq \sqrt{2\sigma^2\log d}$. Applying this here for each coordinate and the negation with $\sigma^2=T$, we conclude that there is a constant $C'>0$ such that
\begin{equation*}
    \mathcal{R}_T(\mathcal{F}-\mathcal{F})\leq C'\lambda\log^{3/2}(T)\sqrt{T\log d}.
\end{equation*}

Putting all these estimates together, we conclude that for any stochastic process $Z_1=(X_1,Y_1),\ldots,Z_T=(X_T,Y_T)\in \{-1,1\}^d\times \{-1,1\}$,
\begin{equation*}
    \Pr\left(\sup_{f\in \mathcal{F}} \sum_{t=1}^T f(Z_t)-\mathbb{E}[f(Z_t)\vert Z_1,\ldots, Z_{t-1}] \geq C''\lambda\log^{3/2}(T)\sqrt{T\log d}+u\right)\leq \exp\left(-\frac{u^2}{64\lambda^2T}\right).
\end{equation*}
In this last step, we shifted $u$ to $u+8\lambda\sqrt{T}$ by just increasing the constant $C''>0$ in the tail bound.
\end{proof}

\subsection{Population Loss Lower Bounds}
\label{sec:population_loss}
In this section, we provide a general argument for sequential block dynamics towards certifying bounds of the form in \Cref{eq:inf_lb}. We will require the following simple bound on the likelihood ratio of a subset of spins conditioned on different configurations for the rest of the coordinates.

\begin{lem}
\label{lem:ratios}
    Let $S,T\subseteq [n]$ be disjoint subsets. Let $\sigma\in \{-1,1\}^S$ and $\tau,\tau'\in \{-1,1\}^T$ be partial configurations, and let $U\subseteq T$ denote the set of coordinates $i$ such that $\tau_i\neq \tau'_i$. Then
    \begin{equation}
    \label{eq:ratio_bound}
    \frac{\mu_{A,\bm{h}}(X_S=\sigma\vert X_T=\tau)}{\mu_{A,\bm{h}}(X_S=\sigma\vert X_T=\tau')}\leq \exp\left(4\sum_{i\in [n]\setminus T,j\in U}\vert A_{ij}\vert\right).
    \end{equation}
\end{lem}
\begin{proof}
    Let $V=[n]\setminus (S\cup T)$. First, note that by averaging,
    \begin{align*}
        \frac{\mu_{A,\bm{h}}(X_S=\sigma\vert X_T=\tau)}{\mu_{A,\bm{h}}(X_S=\sigma\vert X_T=\tau')}&=\frac{\sum_{z\in \{-1,1\}^V}\mu_{A,\bm{h}}(X_S=\sigma,X_V=z\vert X_T=\tau)}{\sum_{z\in \{-1,1\}^V}\mu_{A,\bm{h}}(X_S=\sigma, X_V=z\vert X_T=\tau')}\\
        &\leq \max_{z\in \{-1,1\}^V} \frac{\mu_{A,\bm{h}}(X_S=\sigma,X_V=z\vert X_T=\tau)}{\mu_{A,\bm{h}}(X_S=\sigma,X_V=z\vert X_T=\tau')}.
    \end{align*}
    Since the desired bound \Cref{eq:ratio_bound} does not depend on $\sigma$ nor $S$, by replacing $S$ with $S\cup V$ and taking the maximum over both $(\sigma,z)\in \{-1,1\}^{S\cup V}$, we may as well assume $S,T$ form a partition (i.e. $S=[n]\setminus T$).

    Next, note that
    \begin{align*}
        \mu_{A,\bm{h}}(X_S=\sigma\vert X_T=\tau)&=\frac{\exp\left(\frac{1}{2}(\sigma^TA_{S,S}\sigma+2\sigma^TA_{S,T}\tau+\tau^TA_{T,T}\tau)+\sum_{i\in S}h_i\sigma_i+\sum_{j\in T} h_j\tau_j\right)}{\sum_{\sigma'\in \{-1,1\}^S}\exp\left(\frac{1}{2}(\sigma'^TA_{S,S}\sigma'+2\sigma'^TA_{S,T}\tau+\tau^TA_{T,T}\tau)+\sum_{i\in S}h_i\sigma'_i+\sum_{j\in T} h_j\tau_j\right)}\\
        &=\frac{\exp\left(\frac{1}{2}(\sigma^TA_{S,S}\sigma+2\sigma^TA_{S,T}\tau)+\sum_{i\in S}h_i\sigma_i\right)}{\sum_{\sigma'\in \{-1,1\}^S}\exp\left(\frac{1}{2}(\sigma'^TA_{S,S}\sigma'+2\sigma'^TA_{S,T}\tau)+\sum_{i\in S}h_i\sigma_i'\right)},
    \end{align*}
    and analogously for $X_T=\tau'$. It follows that
    \begin{align*}
        \frac{\mu_{A,\bm{h}}(X_S=\sigma\vert X_T=\tau)}{\mu_{A,\bm{h}}(X_S=\sigma\vert X_T=\tau')}&=\frac{\exp\left(\frac{1}{2}(\sigma^TA_{S,S}\sigma+2\sigma^TA_{S,T}\tau)\right)}{\exp\left(\frac{1}{2}(\sigma^TA_{S,S}\sigma+2\sigma^TA_{S,T}\tau')\right)}\\
        &\times \frac{\sum_{\sigma'\in \{-1,1\}^S}\exp\left(\frac{1}{2}(\sigma'^TA_{S,S}\sigma'+2\sigma'^TA_{S,T}\tau')+\sum_{i\in S}h_i\sigma_i'\right)}{\sum_{\sigma'\in \{-1,1\}^S}\exp\left(\frac{1}{2}(\sigma'^TA_{S,S}\sigma'+2\sigma'^TA_{S,T}\tau)+\sum_{i\in S}h_i\sigma_i'\right)}\\
        &\leq \frac{\exp\left(\frac{1}{2}(\sigma^TA_{S,S}\sigma+2\sigma^TA_{S,T}\tau)\right)}{\exp\left(\frac{1}{2}(\sigma^TA_{S,S}\sigma+2\sigma^TA_{S,T}\tau')\right)}\\
        &\times \max_{\sigma'\in \{-1,1\}^S}\frac{\exp\left(\frac{1}{2}(\sigma'^TA_{S,S}\sigma'+2\sigma'^TA_{S,T}\tau')\right)}{\exp\left(\frac{1}{2}(\sigma'^TA_{S,S}\sigma'+2\sigma'^TA_{S,T}\tau)\right)}\\
        &\leq \max_{\sigma'\in \{-1,1\}^S}\frac{\exp\left(2\sigma'^TA_{S,T}\tau\right)}{\exp\left(2\sigma'^TA_{S,T}\tau'\right)}\\
        &\leq \exp\left(4\sum_{i\in S,j\in U} \vert A_{ij}\vert\right).
    \end{align*}
    In the last step, we canceled all identical terms, which keeps only the products that include a factor in $U$ since those terms differ. This picks up the extra factor of $2$.
\end{proof}

Next, we show that we can provide the requisite lower bounds as in \Cref{eq:inf_lb} for a wide class of sequential block dynamics. The key point is that we will be able to show that each site \emph{typically} does not have excessive influence over the path of updates during the dynamics under very mild assumptions. For us, ``typically'' can be some small constant probability; this alone will be enough to conclude a lower bound of the form \Cref{eq:inf_lb} for any reasonable updating dynamics.

In the below proposition, recall again that we consider the regression problem just for node $n$ for notational convenience.

\begin{proposition}
\label{prop:sufficiecy_lb}
    Fix any $X_0\in \{-1,1\}^n$ and history of sequential block dynamics up to time $0$ (which we denote $\mathcal{F}_0$ after reindexing) such that $X_0$ is obtained by updating site $n$. Let $Z=(X_{\tau_1,-n},X_{\tau_1,n})$ be the configuration at the next update time of node $n$ (as in \Cref{defn:view} after reindexing). For simplicity, write $X=X_{\tau_1,-n}\in \{-1,1\}^{n-1}$.

    There exists a constant $c=c_{\ref{prop:sufficiecy_lb}}>0$ such that the following holds. Suppose $\delta_{\ref{prop:sufficiecy_lb}}>0$, $C_{\ref{prop:sufficiecy_lb}}\geq 1$ satisfy the following conditions: for any $j\neq n$, with probability at least $\delta_{\ref{prop:sufficiecy_lb}}>0$ (over the updates in the sequential block dynamics before $\tau_1$, conditional on all previous events that determined $X_0$), the following events simultaneously occur:\footnote{Note that we do not require that these events hold simultaneously for all $j\neq n$ with this probability. We only require this to be true for each index separately.}
    \begin{enumerate}
        \item There exists $0\leq t\leq \tau_1$ such that $j\in S_t$. In other words, site $j$ was chosen for updating as part of some block before or with site $n$.

        \item For any $k\neq j$, let $N_k$ denote the (random) number of sets $S_t$ such that $k\in S_t$ for $0\leq t\leq \tau_1$ during the sequential block dynamics. Then it holds that
        \begin{equation*}
            \sum_{k\neq j} \vert A_{j,k}\vert N_k\leq C_{\ref{prop:sufficiecy_lb}}\lambda.
        \end{equation*}
    \end{enumerate}

    Then it holds that for all $j\neq n$, and any $(\bm{w},h)\in \mathbb{R}^{n-1}\times \mathbb{R}$,
    \begin{equation*}
        \mathbb{E}_{X}\left[\left(\sigma\left(2(\langle \bm{w},X\rangle+h)\right)-\sigma\left(2(\langle \bm{w}^*,X\rangle+h^*)\right)\right)^2\bigg\vert \mathcal{F}_0\right]\geq c_{\ref{prop:sufficiecy_lb}}\delta_{\ref{prop:sufficiecy_lb}}\exp(-O(C_{\ref{prop:sufficiecy_lb}}\lambda))\min\{1,8(w_j-w_j^*)^2\}.
    \end{equation*}
\end{proposition}
\begin{proof}
    Fix $j\neq n$ and let $\mathcal{E}$ denote the good event of the proposition statement. For technical reasons, we first modify the sequence of updates depending on the size of $S_{\tau_1}$, the first block where $n$ is updated:
    \begin{enumerate}
        \item \textbf{Case 1}: $S_{\tau_1}=\{n\}$. In this case, we define the path of updates $\mathcal{P}$ to be
    \begin{equation*}
        \mathcal{P}=(X_0,(S_1,X_1),\ldots,(S_{\tau_1-1},X_{\tau_1-1})),
    \end{equation*}

    \item \textbf{Case 2}: $\vert S_{\tau_1}\vert>1$. In this case, we define the path by
    \begin{equation*}
        \mathcal{P}=(X_0,(S_1,X_1),\ldots,(S_{\tau_1-1},X_{\tau_1-1}),(S'_{\tau_1}, X'_{\tau_1}))
    \end{equation*}
    where $S'_{\tau_1}=S_{\tau_1}\setminus \{n\}$ denotes the nodes in $S_{\tau_1}$ other than $n$, and $X'_{\tau_1}$ applies the update of block $S_{\tau_1}$ only to the coordinates outside of $\{n\}$ (so in particular, site $n$ does not update in this step). We then define $X_{\tau_1,-n}=X'_{\tau_1,-n}$. 
    \end{enumerate}

    The key point is that in both cases, the law of $X_{\tau_1,n}$ conditional on $\mathcal{P}$ and $\mathcal{F}_0$ is the same as the law of $X_{\tau_1,n}$ conditional on $X_{\tau_1,-n}$; however, in \textbf{Case 2}, the law of $X'_{\tau_1,S'_{\tau_1}}$ at time $\tau_1$ is taken conditional on $X_{\tau_1-1,-S_{\tau_1}}$ (as opposed to $X_{\tau_1-1,-S'_{\tau_1}}$, as holds for all other block updates). This minor difference accounts for the fact that in \textbf{Case 1}, the last update $S_{\tau_1-1}$ is taken conditional also on the initial value at node $n$, while in \textbf{Case 2}, the last update on $S'_{\tau_1}$ does not condition on node $n$ as it is also updated at the same time.
    
    Now, let $\mathcal{P}'$ denote the sequence of updates in $\mathcal{P}$ where the \emph{last} update of site $j$ (if it exists) is replaced by the symbol ``*''.  Notice that the law of $\mathcal{P}$ naturally induces a law on $\mathcal{P}'$ and that the event $\mathcal{E}$ is measurable with respect to $\mathcal{P}'$ since the first condition is determined by the existence of ``*'' and the second does not depend on the value of node $j$. Moreover, note that $X=X_{\tau_1,-n}$ is determined by $\mathcal{P}'$ and the value of $X_j$.

    For convenience of notation, for a configuration $X\in \{-1,1\}^{n-1}$ and $(\bm{w},h)\in \mathbb{R}^{n-1}\times \mathbb{R}$, let us define
    \begin{equation*}
        X(\bm{w},h)\triangleq \sigma(2(\langle \bm{w},X\rangle + h)).
    \end{equation*}
    
    We therefore have the lower bound
    \begin{align*}
        \mathbb{E}_{Z}&\left[\left(\sigma\left(2(\langle \bm{w},X\rangle+h)\right)-\sigma\left(2(\langle \bm{w}^*,X\rangle+h^*)\right)\right)^2\bigg\vert \mathcal{F}_0\right]= \mathbb{E}_{\mathcal{P}'}\left[\mathbb{E}_{X_{j}}\left[\left(X(\bm{w},h)-X(\bm{w}^*,h^*)\right)^2\bigg\vert \mathcal{P}'\right]\bigg\vert \mathcal{F}_0\right]\\
        &= \mathbb{E}_{\mathcal{P}'}[\underbrace{\Pr(X_j=1\vert \mathcal{P}')(X^{j,+}(\bm{w},h)-X^{j,+}(\bm{w}^*,h^*))^2+\Pr(X_j=-1\vert \mathcal{P}')(X^{j,-}(\bm{w},h)-X^{j,-}(\bm{w}^*,h^*))^2}_{=A(\mathcal{P}')}\vert \mathcal{F}_0]\\
        &\geq \mathbb{E}_{\mathcal{P}'}[A(\mathcal{P}')\mathbf{1}[\mathcal{E}]\vert \mathcal{F}_0].
    \end{align*}
    Here, $\Pr(\cdot\vert \mathcal{P}')$ is conditional on $\mathcal{F}_0$ as well, though we omit this for convenience. The penultimate line follows since $X$ is determined by $X_j$ and $\mathcal{P}'$ (given the initial configuration $X_0$) by construction. 

    We claim that on the event $\mathcal{E}$, we have for $\varepsilon\in \{-1,1\}$,
    \begin{equation}
    \label{eq:prob_lb}
        \Pr(X_j=\varepsilon \vert \mathcal{P}')\geq \exp(-O(C_{\ref{prop:sufficiecy_lb}}\lambda)).
    \end{equation}
    The key point is that by the definition of $\mathcal{E}$, it will hold that the likelihood ratio of $X_j=\varepsilon$ given $\mathcal{P}'$ will be bounded above and below; in other words, for paths $\mathcal{P}'\in \mathcal{E}$, the posterior distribution on the true spin at site $j$ will remain somewhat unbiased. 
    
    It suffices to show that on $\mathcal{E}$, we have
    \begin{equation*}
        \exp(-O(C_{\ref{prop:sufficiecy_lb}}\lambda))\leq \frac{\Pr(X_j=1 \vert \mathcal{P}')}{\Pr(X_j=-1 \vert \mathcal{P}')}\leq \exp(O(C_{\ref{prop:sufficiecy_lb}}\lambda)).
    \end{equation*}
    As our argument will be symmetric in $\pm 1$, it will suffice to establish the upper bound. First note that by Bayes rule,
        \begin{equation}
        \label{eq:path_bayes}
\frac{\Pr(X_j=+1\vert \mathcal{P}')}{\Pr(X_j=-1\vert \mathcal{P}')}=\frac{\Pr(\mathcal{P}^{+1})}{\Pr(\mathcal{P}^{-1})}
\end{equation}
where $\mathcal{P}^{\varepsilon}$ is the path of updates given by $\mathcal{P}'$ where the last update of $j$ (which was replaced by the symbol ``*'' on $\mathcal{E}$) is substituted with the value $\varepsilon$.

For a path $\mathcal{P}'\in \mathcal{E}$, let $\tau$ denote the last update time of site $j$ before $\tau_1$ (which exists by definition of $\mathcal{E}$). For \textbf{Case 1}, we calculate:
\begin{align*}
\Pr(\mathcal{P}^{\varepsilon})&=\left(\prod_{t=1}^{\tau_1} \Pr(S_t\vert \mathcal{F}_0, S_1,\ldots,S_{t-1})\right)\left(\prod_{t=1}^{\tau-1}\mu_{A,\bm{h}}(X_{t,S_t}\vert X_{t-1,-S_{t}})\right)\\
    &\times \mu_{A,\bm{h}}(X_{\tau,S_{\tau}}^{j,\varepsilon}\vert X_{\tau-1,-S_{\tau}})\times \left(\prod_{t=\tau+1}^{\tau_1-1} \mu_{A,\bm{h}}(X_{t,S_t}\vert X_{t-1,-S_{t}}^{j,\varepsilon})\right)
\end{align*}
Here, we factorize the probability of the path the choice of blocks and of the actual spin updates using the  independent of the actual updates conditioned on the previous blocks by \Cref{defn:sbd}.

For \textbf{Case 2}, if $j\not\in S_{\tau_1}$ we can calculate in the analogous way:
\begin{align*}
\Pr(\mathcal{P}^{\varepsilon})&=\left(\prod_{t=1}^{\tau_1} \Pr(S_t\vert \mathcal{F}_0, S_1,\ldots,S_{t-1})\right)\left(\prod_{t=1}^{\tau-1}\mu_{A,\bm{h}}(X_{t,S_t}\vert X_{t-1,-S_{t}})\right)\\
    &\times \mu_{A,\bm{h}}(X_{\tau,S_{\tau}}^{j,\varepsilon}\vert X_{\tau-1,-S_{\tau}})\times \left(\prod_{t=\tau+1}^{\tau_1-1} \mu_{A,\bm{h}}(X_{t,S_t}\vert X_{t-1,-S_{t}}^{j,\varepsilon})\right)\times \mu_{A,\bm{h}}(X_{\tau_1,S'_{\tau_1}} \vert X^{j,\varepsilon}_{\tau_1-1,-S_{\tau_1}}).
\end{align*}
In the event that $j\in S_{\tau_1}$, the same factorization holds omitting the last two terms, and replacing $\mu_{A,\bm{h}}(X_{\tau,S_{\tau}}^{j,\varepsilon}\vert X_{\tau-1,-S_{\tau}})$ with $\mu_{A,\bm{h}}(X_{\tau,S'_{\tau}}^{j,\varepsilon}\vert X_{\tau-1,-S_{\tau}})$ (i.e. updating only $S'_{\tau}=S'_{\tau_1}=S_{\tau_1}\setminus \{n\}$).

It thus follows that in \textbf{Case 1}, on the event $\mathcal{E}$, we have by the above display and \Cref{eq:path_bayes},
\begin{equation*}
    \frac{\Pr(X_j=1 \vert \mathcal{P}')}{\Pr(X_j=-1 \vert \mathcal{P}')}=\frac{\mu_{A,\bm{h}}(X_{\tau,S_{\tau}}^{j,+}\vert X_{\tau-1,-S_{\tau}})\prod_{t=\tau+1}^{\tau_1-1} \mu_{A,\bm{h}}(X_{t,S_t}\vert X_{t-1,-S_{\tau}}^{j,+})}{\mu_{A,\bm{h}}(X_{\tau,S_{\tau}}^{j,-}\vert X_{t-1,-S_{t}})\prod_{t=\tau+1}^{\tau_1-1} \mu_{A,\bm{h}}(X_{t,S_t}\vert X_{t-1,-S_{t}}^{j,-})},
\end{equation*}
while in \textbf{Case 2}, we have
\begin{equation*}
    \frac{\Pr(X_j=1 \vert \mathcal{P}')}{\Pr(X_j=-1 \vert \mathcal{P}')}=\frac{\mu_{A,\bm{h}}(X_{\tau,S_{\tau}}^{j,+}\vert X_{\tau-1,-S_{\tau}})\prod_{t=\tau+1}^{\tau_1} \mu_{A,\bm{h}}(X_{t,S_t}\vert X_{t-1,-S_{\tau}}^{j,+})\mu_{A,\bm{h}}(X_{\tau_1,S'_{\tau_1}} \vert X^{j,+}_{\tau_1-1,-S_{\tau_1}})}{\mu_{A,\bm{h}}(X_{\tau,S_{\tau}}^{j,-}\vert X_{t-1,-S_{t}})\prod_{t=\tau+1}^{\tau_1-1} \mu_{A,\bm{h}}(X_{t,S_t}\vert X_{t-1,-S_{t}}^{j,-})\mu_{A,\bm{h}}(X_{\tau_1,S'_{\tau_1}} \vert X^{j,-}_{\tau_1-1,-S_{\tau_1}})},
\end{equation*}
where again, if $j\in S_{\tau_1}$, we omit the last two products and make the same replacement of $S_{\tau}$ with $S'_{\tau}$ in the first ratio.

In both cases, applying \Cref{lem:ratios} shows that the first product ratio in both cases is bounded by $\exp(4\lambda)$ (to see this, apply Bayes rule to reduce this to calculating the ratio of the probability of spins at $j$ conditional on all other sites and apply the lemma). Similarly, \Cref{lem:ratios} implies that for $\tau+1\leq t\leq\tau_1-1$,
\begin{equation*}
    \frac{\mu_{A,\bm{h}}(X_{t,S_t}\vert X_{t-1,-S_{t}}^{j,+})}{\mu_{A,\bm{h}}(X_{t,S_t}\vert X_{t-1,-S_{t}}^{j,-})}\leq \exp\left(4\sum_{k\in S_t} \vert A_{jk}\vert\right),
\end{equation*}
and
\begin{equation*}
    \frac{\mu_{A,\bm{h}}(X_{\tau_1,S'_{\tau_1}} \vert X^{j,+}_{\tau_1-1,-S_{\tau_1}})}{\mu_{A,\bm{h}}(X_{\tau_1,S'_{\tau_1}} \vert X^{j,-}_{\tau_1-1,-S_{\tau_1}})}\leq \exp\left(4\sum_{k\in S_{\tau_1}} \vert A_{jk}\vert\right)
\end{equation*}
Thus, in either case, on the event $\mathcal{E}$, multiplying out all the products and using the assumption on the $N_k$ implies that 
\begin{equation*}
    \frac{\Pr(X_j=1 \vert \mathcal{P}')}{\Pr(X_j=-1 \vert \mathcal{P}')}\leq \exp\left(4\left(\lambda+\sum_{k\neq j} N_k \vert A_{jk}\vert\right)\right)\leq \exp(4(C_{\ref{prop:sufficiecy_lb}}+1)\lambda)\leq \exp(8C_{\ref{prop:sufficiecy_lb}}\lambda),
\end{equation*}
where we use the assumption $C_{\ref{prop:sufficiecy_lb}}\geq 1$.
 By our reduction, this verifies the lower bound \Cref{eq:prob_lb}. It follows that
 \begin{align*}
     \mathbb{E}_{\mathcal{P}'}[A(\mathcal{P}')\mathbf{1}[\mathcal{E}]\vert \mathcal{F}_0]&\geq \exp(-O(C_{\ref{prop:sufficiecy_lb}}\lambda))\mathbb{E}_{\mathcal{P}'}[\mathbf{1}(\mathcal{E})[(X^{j,+}(\bm{w},h)-X^{j,+}(\bm{w}^*,h^*))^2\\
     &+(X^{j,-}(\bm{w},h)-X^{j,-}(\bm{w}^*,h^*))^2]\vert \mathcal{F}_0].
 \end{align*}
 At this point, the remainder of the argument is entirely deterministic and follows the calculation of \cite{DBLP:conf/nips/WuSD19}. By \Cref{fact:sigmoid_lb}, we have
 \begin{equation*}
     (X^{j,\pm}(\bm{w},h)-X^{j,\pm}(\bm{w}^*,h^*))^2\geq \frac{\exp(-2(\langle \bm{w}^*,X^{j,\pm}\rangle+h^*))}{16e^2}\cdot \min\{1,4(\langle \bm{w}^*-\bm{w},X^{j,\pm}\rangle + (h^*-h))^2\}.
 \end{equation*}
 By the inequality (c.f. Equation (35) of \cite{DBLP:conf/nips/WuSD19}, which can be proved by cases)
 \begin{equation*}
     \min\{1,a^2\}+\min\{1,b^2\}\geq \min\{1,(a-b)^2/2\},
 \end{equation*}
 we further obtain
 \begin{align*}
     (X^{j,+}(\bm{w},h)-X^{j,+}(\bm{w}^*,h^*))^2&+(X^{j,-}(\bm{w},h)-X^{j,-}(\bm{w}^*,h^*))^2\\
     &\geq \frac{\exp(-2\lambda)}{16e^2}\min\{1,8(w^*_j-w_j)^2\},
 \end{align*}
 where we use the fact $X^{j,+}-X^{j,-}=2\bm{e}_j$, the $j$th standard basis vector. Therefore, we obtain a lower bound 
 \begin{align*}
     \mathbb{E}_{\mathcal{P}'}[A(\mathcal{P}')\mathbf{1}[\mathcal{E}]\vert \mathcal{F}_0]&\geq \exp(-O(C_{\ref{prop:sufficiecy_lb}}\lambda))\min\{1,8(w_j-w_j^*)^2\}\mathbb{E}_{\mathcal{P}'}[\mathbf{1}[\mathcal{E}]\vert \mathcal{F}_0]\\
     &\geq \delta_{\ref{prop:sufficiecy_lb}}\exp(-O(C_{\ref{prop:sufficiecy_lb}}\lambda))\min\{1,8(w_j-w_j^*)^2\},
 \end{align*}
 where the last inequality is the assumed lower bound on the probability of the good event. This completes the proof.
\end{proof}

It is quite simple to show that \Cref{prop:sufficiecy_lb} holds for a wide range of natural dyamics:
\begin{proposition}
\label{prop:goodblocks}
    For any symmetric block dynamics and for round robin dynamics, the conditions of \Cref{prop:sufficiecy_lb} hold with parameters $\delta_{\ref{prop:sufficiecy_lb}}=1/4$ and $C_{\ref{prop:sufficiecy_lb}}=8$.
\end{proposition}
\begin{proof}
    For symmetric block dynamics, it is easy to see that the first condition holds with probability at least $1/2$ by the assumed permutation invariance, since the permutation that transposes $j$ and $n$ weakly reverses which site was chosen for updating first while preserving the law of the updated sites. For round robin dynamics, the first condition trivially holds with probability $1$ since after an update of site $n$, each site updates exactly once before site $n$ is updated again.

    Meanwhile for any node $k<n$ and for any symmetric block dynamics, we have $\mathbb{E}[N_k]\leq 2$ since the number of times a node $k$ is updated before $n$ is updated is stochastically dominated by a geometric random variable with probability $1/2$ (by the same symmetry argument), so in particular,
    \begin{equation*}
        \mathbb{E}\left[\sum_{k\neq i,j} \vert A_{j,k}\vert N_k\right]\leq 2\sum_{k\neq i,j} \vert A_{j,k}\vert\leq 2\lambda
    \end{equation*}
    By Markov's inequality, it thus holds with probability at least $3/4$ that $\sum_{k\neq i,j} \vert A_{j,k}\vert N_k\leq 8\lambda$. A union bound then implies that both events hold with probability at least $1/4$. Again, for round robin dynamics, the condition actually holds for $C=1$ trivially.
\end{proof}

We can finally complete the proof of \Cref{thm:learning_block} as a simple consequence:
\begin{proof}[Proof of \Cref{thm:learning_block}]
It suffices to verify the conditions of \Cref{thm:wsd} when solving the regression problem \Cref{eq:empiricalopt} for each node separately and then setting parameters appropriately. By the tail bound of \Cref{thm:hpbounds} and the two-sided bounds of \Cref{lem:opt_azuma}, the sample complexity function for the required inequalities in the first condition \Cref{eq:convergence,eq:opt_deviation} satisfies 
\begin{equation*}
T(\varepsilon,\delta)\leq O\left(\frac{\lambda^2\log(n/\delta)\log^3\left(\lambda^2\log(n/\delta)/\varepsilon^2\right)}{\varepsilon^2}\right)=\tilde{O}\left(\frac{\lambda^2\log(n/\delta)}{\varepsilon^2}\right)
\end{equation*} 
for the number of samples generated for each node. We also obtain the second condition \Cref{eq:inf_lb} for each of these dynamics with $c_{\ref{thm:wsd}}=\exp(-O(\lambda))$ and $\delta_{\ref{thm:wsd}}=1/4$ by applying \Cref{prop:sufficiecy_lb} conditioned on the snapshots obtained at each update time of a given node, where the assumptions for the stated dynamics are verified by \Cref{prop:goodblocks}. Applying \Cref{thm:wsd} with additive accuracy parameter $\varepsilon$ and failure probability $\delta/n$ for each node and then taking a union bound yields the desired conclusion.
\end{proof}

\section{Learning the Sherrington-Kirkpatrick Model at High-Temperature}
\label{sec:sk}

In this section, we show how a similar analysis leads to new efficient learning guarantees in even in for the well-studied setting of independent samples. In particular, we will derive as a result of a simple but general \Cref{thm:ising_cov} efficient algorithms for learning the Sherrington-Kirkpatrick (SK) model in essentially the entire (currently provable) ``high-temperature'' regime where the model is known to be replica-symmetric. To our knowledge, the only previous guarantee for learning the SK model with $\mathsf{poly}(n)$ samples was obtained by Anari, Jain, Koehler, Pham, and Vuong \cite{anari_universality} for inverse temperature $\beta<1/4$. As described before, their analysis relies on difficult-to-prove properties of Gibbs measure using techniques that inherently (at least in general) cannot break this barrier. On the other hand, without external field, the high-temperature regime of the SK model extends to $\beta<1$, at least at the level of free energy and replica-symmetry \cite{talagrand2011mean}.

\iffalse
Their result is a consequence of very recent results showing the measures satisfy approximate tensorization of entropy, which Koehler, Heckett, and Risteski \cite{DBLP:conf/iclr/KoehlerHR23} have recently shown is sufficient for efficient learning. However, approximate tensorization of entropy is rather difficult to establish in the entire high-temperature regime since it implies, among other things, modified log-Sobolev inequalities for Glauber dynamics, a major open problem in this area. The barrier at $\beta<1/4$ arises due to a very general argument of Bauerschmidt and Bodineau \cite{bauerschmidt} (see also Eldan, Koehler, and Zeitouni \cite{eldan2022spectral}) which provably cannot extend past this parameter, and it is unclear whether current techniques will be able to achieve the entire (conjectured) regime $\beta<1$ for SK.
\fi

Nonetheless, our analysis bypasses these difficulties and entirely derives from the following simple observation: the performance of logistic regression should depend on the \emph{typical} properties of samples, not worst-case bounds. Moreover, the typical properties of the samples should rely on global properties of the measure, rather than local to global principles like approximate tensorization of entropy. We show that this is indeed the case: operator norm bounds on the covariance matrix are already sufficient to learning these models in total variation with no external field. Such bounds have very recently been obtained \cite{alaoui2022bounds,brennecke2022two,brennecke2023operator} throughout the high-temperature regime. For models with suitable external field, we provide a somewhat more complex argument relying on the so-called TAP equations (or the analysis thereof). In particular, we show how well-studied and simpler structural properties in the entire high-temperature regime are already sufficient to establishing efficient learning via logistic regression.

We now carry out this plan in the result below, which makes precise the idea that the typical behavior of spins at stationarity suffices for learning from i.i.d. samples:

\begin{thm}
\label{thm:ising_cov}
    Let $A\in \mathbb{R}^{n\times n}$ be an interaction matrix and $\bm{h}^*\in \mathbb{R}^n$ the external fields for an Ising model. Suppose there is a constant $C_{\ref{thm:ising_cov}}\geq 1$ such that:
    \begin{enumerate}
        \item Each fixed row $A_j$ satisfies $\|A_j\|_{\infty}\leq C_{\ref{thm:ising_cov}}$.
        \item For each fixed row $A_j$ for $j\in [n]$, it holds with probability at least $3/4$ over $X\sim \mu_{A,\bm{h}}$ that $\vert X^T A_j+h^*_j\vert\leq C_{\ref{thm:ising_cov}}$.\footnote{Note that since $A_{j,j}=0$, this event depends only on spins outside $j$.}
    \end{enumerate}
    
    For fixed $i\in [n]$, let $w^* = A_i$, the $i$th row of $A$, and $h^*=h^*_i$. Then it holds that for any $(\bm{w},h)\in \mathbb{R}^{n-1}\times \mathbb{R}$ and any $j\neq i$ that
    \begin{equation*}
        \mathbb{E}_{X}\left[\left(\sigma\left(2(\langle \bm{w},X\rangle+h)\right)-\sigma\left(2(\langle \bm{w}^*,X\rangle+h^*)\right)\right)^2\right]\geq \exp(-O(C_{\ref{thm:ising_cov}}))\cdot \min\{1,8(w_j-w^*_j)^2\}.
    \end{equation*}
    Here, we slightly abuse notation and also consider $X\in \{-1,1\}^{[n]\setminus \{i\}}$ since the above expectation does not depend on $X_i$. Furthermore, for any $(\bm{w},h)\in \mathbb{R}^{n-1}\times \mathbb{R}$ such that $\|\bm{w}-\bm{w}^*\|_\infty\leq \vert h-h^*\vert/2n$, it holds that
    \begin{equation*}
        \mathbb{E}_{X}\left[\left(\sigma\left(2(\langle \bm{w},X\rangle+h)\right)-\sigma\left(2(\langle \bm{w}^*,X\rangle+h^*)\right)\right)^2\right]\geq \exp(-O(C_{\ref{thm:ising_cov}}))\cdot \min\{1,(h-h^*)^2\}
    \end{equation*}
\end{thm}
\begin{proof}
    The conclusion is symmetric in the index of the coordinate, so we consider just $i=n$. Fix a site $j<n$, set $\bm{w}^*=A_n$ and $h^*=h_n$, and let $A_j$ be the row of interactions with site $j$. We turn to lower bounding the quantity:
    \begin{equation*}
\mathbb{E}_{X}\left[\left(\sigma\left(2(\langle \bm{w},X\rangle+h)\right)-\sigma\left(2(\langle \bm{w}^*,X\rangle+h^*)\right)\right)^2\right],
    \end{equation*}
    where $X$ is a sample from $\mu_{A}$ restricted to the coordinates outside $n$ in terms of the deviation of $w_j$ and $w^*_j$. By the second assumption above and a union bound, it holds with probability at least $1/2$ over $X\sim \mu_{A}$ that 
    \begin{gather*}
        \max\{\vert X^T A_j+h^*\vert,\vert X^T w^*+h^*\vert\}\leq C_{\ref{thm:ising_cov}}.
    \end{gather*} 
    By \Cref{eq:glauber}, this implies that with probability at least $1/2$ over $X_{-\{j,n\}}\sim \mu_{A}\vert_{-\{j,n\}}$ (i.e. the marginal over the coordinates outside $\{j,n\}$), we have for both $\varepsilon\in \{-1,1\}$ 
    \begin{align*}
        \Pr(X_j=\varepsilon\vert X_{-j})&\geq \min\{\Pr(X_j=\varepsilon\vert X_{-j},X_n=1),\Pr(X_j=\varepsilon\vert X_{-j},X_n=-1)\}\\
        &\geq \exp(-O(\vert \langle A_j,X\rangle+h_j\vert +\vert A_{jn}\vert))\\
        &\geq \exp(-O(C_{\ref{thm:ising_cov}})),
    \end{align*}
    and 
    \begin{equation*}
        \vert X^T\bm{w}^*+h^*\vert\leq \vert X_{-j}^T\bm{w}^*_{-j}+h^*\vert+\vert X_j\bm{w}^*_j\vert\leq  2C_{\ref{thm:ising_cov}}.
    \end{equation*}

    Let $\mathcal{E}$ be the event over $X\sim \mu_A$ restricted to $X_{-j}$ that both these events hold. We thus have
    \begin{align*}
        \mathbb{E}_{X}\left[\left(\sigma\left(2(\langle \bm{w},X\rangle+h)\right)-\sigma\left(2(\langle \bm{w}^*,X\rangle+h^*)\right)\right)^2\right]&\geq 
        \mathbb{E}_{X}\left[\left(\sigma\left(2(\langle \bm{w},X\rangle+h)\right)-\sigma\left(2(\langle \bm{w}^*,X\rangle+h^*)\right)\right)^2\mathbf{1}(\mathcal{E})\right]\\
        &\geq \exp(-O(C_{\ref{thm:ising_cov}}))\min\{1,8(w_j-w_j^*)^2\}\mathbb{E}[\mathbf{1}(\mathcal{E})]\\
        &\geq \exp(-O(C_{\ref{thm:ising_cov}}))\min\{1,8(w_j-w_j^*)^2\}/2,
    \end{align*}
    by following the exact same deterministic argument as in \Cref{prop:sufficiecy_lb} and absorbing constants, using the fact that the good event has probability at least $1/2$. Here, we are using the fact that the argument of the sigmoid evaluated at $(\bm{w}^*,h^*)$ in the expression cannot be too large on $\mathcal{E}$ by definition, so we may apply \Cref{fact:sigmoid_lb} and pay at most $\exp(-O(C_{\ref{thm:ising_cov}}))$ in deriving the lower bound, and the conditional probability of $X_j=\varepsilon$ given the remaining coordinates outside $\{j,n\}$ is lower bounded by the same quantity on $\mathcal{E}$.

    The same computation shows the desired claim about $\vert h-h^*\vert$ when $\|\bm{w}-\bm{w}^*\|_{\infty}\leq  \vert h-h^*\vert/2n$: an even more direct argument using \Cref{fact:sigmoid_lb} shows that
    \begin{align*}
        \mathbb{E}_{X}&\left[\left(\sigma\left(2(\langle \bm{w},X\rangle+h)\right)-\sigma\left(2(\langle \bm{w}^*,X\rangle+h^*)\right)\right)^2\right]\geq 
        \mathbb{E}_{X}\left[\left(\sigma\left(2(\langle \bm{w},X\rangle+h)\right)-\sigma\left(2(\langle \bm{w}^*,X\rangle+h^*)\right)\right)^2\mathbf{1}(\mathcal{E})\right]\\
        &\geq \exp(-O(C_{\ref{thm:ising_cov}}))\mathbb{E}[\min\{1,4(\langle X,\bm{w}-\bm{w}^*\rangle+(h-h^*))^2\}\mathbf{1}(\mathcal{E})]\\
        &\geq \exp(-O(C_{\ref{thm:ising_cov}}))\mathbb{E}[\min\{1,4((h-h^*)/2)^2\}\mathbf{1}(\mathcal{E})]\\
        &=\exp(-O(C_{\ref{thm:ising_cov}}))\mathbb{E}[\min\{1,(h-h^*)^2\}\mathbf{1}(\mathcal{E})]\\
        &\geq \exp(-O(C_{\ref{thm:ising_cov}}))\min\{1,(h-h^*)^2\}/2.
    \end{align*}
    Here, we use the fact that if $\vert x\vert\leq \vert y\vert/2$, then $(x+y)^2\geq y^2/4$. This is precisely guaranteed by our assumption on the closeness of $\bm{w}$ and $\bm{w}^*$ in $\ell_{\infty}$ and Holder's inequality since $\|X\|_1=n-1$.
\end{proof}

We now show a simple condition that suffices to certify the condition in \Cref{thm:ising_cov}. For zero-mean Ising measures, we can easily obtain sufficient conditions for \Cref{thm:ising_cov} from operator norm bounds on the covariance matrix. To see this, we apply the following simple lemma:

\begin{lem}
\label{lem:covprob}
    Let $\mu$ be any zero-mean distribution over $\{-1,1\}^n$ and suppose that $\|\mathrm{Cov}(\mu)\|_{\mathsf{op}}\leq C$. Then for fixed vector $\bm{z}\in \mathbb{R}^n$, it holds with probability at least $3/4$ over $X\sim \mu$ that $\vert X^T\bm{z}\vert\leq 4\sqrt{C}\|\bm{z}\|_2$.
\end{lem}
\begin{proof}
    By Jensen's inequality, the assumed covariance bound, and then Cauchy-Schwarz, we have
    \begin{equation*}
        \mathbb{E}_{X}[\vert \bm{z}^TX\vert]\leq \mathbb{E}[\bm{z}^TXX^T\bm{z}]^{1/2}=(\bm{z}^T\text{Cov}(\mu)\bm{z}^T)^{1/2}\leq \sqrt{C}\|\bm{z}\|_2.
    \end{equation*}
    By Markov's inequality, it follows that the probability over $X\sim \mu$ that $\vert \bm{z}^TX\vert\geq 4\sqrt{C}\|\bm{z}\|_2$ is at most $1/4$.
\end{proof}

\begin{corollary}
\label{cor:ising_cov}
    Let $A\in \mathbb{R}^{n,n}$ be an interaction matrix for an zero-field, Ising model with the following properties:
    \begin{enumerate}
        \item Each row of $A$ has $\ell_2$ norm at most $C_1$.
        \item There exists $C_2\geq 1$ such that $\|\mathrm{Cov}(\mu_A)\|_{\mathsf{op}}\leq C_2$.
    \end{enumerate}
    Then to learn each entry $A_{i,j}$ up to accuracy $\varepsilon$ with $1-\delta$ probability simultaneously, it suffices to run logistic regression as in \Cref{eq:empiricalopt} with constraint set $\mathcal{C}=\{\bm{w}:\mathbb{R}^{n-1}:\|w\|_1\leq C_1\sqrt{n}\}$ for each node with $T=\frac{C_1^2\exp(O(C_1\sqrt{C_2}))n\log (n/\delta)}{\varepsilon^4}$ samples.
\end{corollary}
\begin{proof}
    As usual, we focus on recovering the row of interactions for each node $i\in [n]$ with probability at least $1-\delta/n$ and then taking a union bound. Let $\mathcal{C}=\{w:\mathbb{R}^{n-1}:\|w\|_1\leq C_1\sqrt{n}\}$. Note that by assumption, $\bm{w}^*=A_i$ is feasible by Cauchy-Schwarz. We again consider solving the logistic regression problem for each node $i$ of the form
    \begin{equation*}
        {\hat{\bm{w}}}\in \arg\min_{\bm{w}\in \mathcal{C}} \frac{1}{T}\sum_{t=1}^T \ell(2Y_t\langle \bm{w},X_t\rangle),
    \end{equation*}
    where we write $Y_t=Z_{t,i}$ and $X_t=Z_{t,-i}$ for independent samples $Z_1,\ldots,Z_T\sim \mu_A$. Standard uniform convergence bounds for the i.i.d. setting (see, e.g. Theorem 26.12 of \cite{books/daglib/0033642}) %\EM{Cite?} 
    imply that with probability at least $1-\delta$,
    \begin{equation*}
        \sup_{\bm{w}\in \mathcal{C}} \frac{1}{T}\sum_{t=1}^T \mathbb{E}_{Z_t}[\ell(2Y_t\langle \bm{w},X_t\rangle)]-\frac{1}{T}\sum_{t=1}^T \ell(2Y_t\langle \bm{w},X_t\rangle) \leq O\left(C_1\sqrt{\frac{n\ln(n/\delta)}{T}}\right).
    \end{equation*}
    Combined with the simpler uniform convergence bound \Cref{lem:opt_azuma} for the true interaction coefficients $\bm{w}^*$, we find that $T(\varepsilon,\delta)\leq O(C_1^2n\ln(n/\delta)/\varepsilon^2)$.

    Next, note that the above assumptions on the rows of $A$ and operator norm of the covariance, along with \Cref{lem:covprob}, certify the conditions of \Cref{thm:ising_cov} with parameter $4C_1\sqrt{C_2}\geq 4C_1$.\footnote{Note that $C_2\geq 1$ since the covariance matrix of a zero field Ising model has all ones on the diagonal, so the variational characterization of eigenvalues implies the inequality.} Putting these two statements together allows us to conclude via \Cref{thm:wsd}.
\end{proof}

\subsection{Zero External Field}

While \Cref{cor:ising_cov} is quite elementary, we show that it already implies efficient algorithms for learning the SK model using $\widetilde{O}_{\beta}(n^9)$ samples in the entire high-temperature regime $\beta<1$ with no external field. We write $\mu_{\beta,A}$ to denote the Sherrington-Kirkpatrick measure at inverse temperature $\beta\geq 0$ and $A\sim \mathsf{GOE}(n)$:

\begin{equation*}
    \mu_{\beta,A}(\bm{x})\propto \exp\left(\frac{\beta}{2}\bm{x}^TA\bm{x}\right).
\end{equation*}
In this section, we will use the following facts valid for any fixed $\beta<1$.

\begin{proposition}
\label{prop:skfacts}
Fix $\beta<1$. Then the following simultaneously holds for the Sherrington-Kirkpatrick measure $\mu_{\beta,A}$ where $A\sim \mathsf{GOE}(n)$ with probability $1-o(1)$:
\begin{enumerate}
    \item It holds that
    \begin{equation}
    \label{eq:covbound}
        \|\mathrm{Cov}(\mu_{\beta,A})\|_{\mathsf{op}}\leq \frac{2}{(1-\beta)^2}.
    \end{equation}
    
\item It holds that 
    \begin{equation}
    \label{eq:normbound}
        \max_{i\in [n]} \sqrt{\sum_{j\neq i} a^2_{ij}}\leq 2\beta .
    \end{equation}
    In particular, each row of $\beta\cdot A$ has $\ell_2$ norm at most $2$.
\end{enumerate}
\end{proposition}
\begin{proof}
    The first claim is shown by Brennecke, Scherster, Xu, and Yau \cite{brennecke2022two} to hold with probability $1-o(1)$. The second is obtained from \Cref{fact:norm_concentration} with $t=1$.
\end{proof}
\begin{remark}
    The probability error term appears to be nonexplicit, as the result of \cite{brennecke2022two} appears to rely on asymptotic convergence in distribution in their argument. However, the $o(1)$ term in  \Cref{prop:skfacts} can be made explicit using the result of El Alaoui and Gaitonde \cite{alaoui2022bounds} at minimal cost to the bound in \Cref{eq:covbound}; their result shows constant bounds (in terms of $\beta$) on the \emph{expected} operator norm, so one can obtain an error rate of $1/\log^c n$ for small enough $c>0$ by Markov's inequality without significantly affecting the guarantees of the remainder of this section.
\end{remark}

\begin{thm}
\label{thm:sk_no_external}
    There exists a universal constant $C_{\ref{thm:sk_no_external}}>0$ such that the following holds. For any fixed $\beta<1$, and for any $\alpha>0$, it holds with probability $1-o(1)$ over $A\sim \mathsf{GOE}(n)$ that given $T\geq \frac{\exp\left(\frac{C}{1-\beta}\right)n^9\log(n/\delta)}{\alpha^8}$ samples from $\mu_{\beta,A}$, with probability at least $1-\delta$ over the samples, solving the sparse logistic regression problem \Cref{eq:empiricalopt} with $C_1=2\sqrt{n}$ returns an Ising model $\widehat{A}$ such that $\mathsf{d}_{TV}(\mu_{\widehat{A}},\mu_{\beta,A})\leq \alpha$.
\end{thm}
\begin{proof}
    By \Cref{lem:additive_suff}, it suffices to learn the entries of $\beta A$ to additive accuracy $\varepsilon=\frac{2\alpha^2}{n^{2}}$ with probability at least $1-\delta$ over the samples. Applying \Cref{cor:ising_cov} shows that with $1-o(1)$ probability over $A\sim \mathsf{GOE}(n)$, it suffices to run logistic regression with sample complexity
    \begin{equation*}
        T=\frac{C_1^2\exp(O(C_1\sqrt{C_2}))n\log(n/\delta)}{\varepsilon^4}=\frac{\exp\left(\frac{C}{1-\beta}\right)n^9\log(n/\delta)}{\alpha^8},
    \end{equation*}
    for some absolute constant $C>0$. We use \Cref{prop:skfacts} in the last step to provide values of $C_1$ and $C_2$ for the Sherrington-Kirkpatrick interaction matrix that hold with $1-o(1)$ probability.
\end{proof}

A natural question is whether even weaker conditions than covariance norm bounds suffice for the above result. We note that the assumption of replica-symmetry (for instance, subgaussian tails of overlaps from independent samples drawn from the measure) does not appear sufficient to work in a black-box manner. Indeed, \Cref{lem:covprob} is false when the assumption is weakened in this way. To see this, consider a planted measure where the first $\sqrt{n}$ coordinates are fixed to all identically $\pm 1$ with equal probability and the remaining coordinates are independent and uniform. Each coordinate is unbiased and the overlap also satisfies similar subgaussian tails (with a slight positive bias) since the overlaps will typically remain of order $1/\sqrt{n}$. But in the direction of the unit vector $v$ with entries $+1/n^{1/4}$ in the first $\sqrt{n}$ coordinates, the absolute overlap with $v$ is $n^{1/4}$ with probability $1$. This example can be approximately realized by an Ising model with ferromagnetic clique interactions on the first $\sqrt{n}$ nodes and no other edges. 

\subsection{Gaussian External Field}

In this section, we show via a somewhat more involved argument how existing structural results are already sufficient in our framework to establish the efficient learnability of the Sherrington-Kirkpatrick model with Gaussian external field in a large portion of the high-temperature regime.\footnote{The case of uniform, or more general non-centered Gaussian, fields is the most commonly considered model in the literature.} To state the result, we require the following notation: given a (possibly degenerate) normal random variable $h\sim \mathcal{N}(\mu,\sigma^2)$, define a fixed point
\begin{equation}
\label{eq:RSq}
    q = q(\beta,\mu,\sigma^2)=\mathbb{E}_{h,z}\left[\tanh^2(\beta z\sqrt{q}+h)\right],
\end{equation}
where $z$ is standard normal.
It was shown by Lata\l{}a and Guerra (see Proposition 1.3.8 of \cite{talagrand2010mean}) that the solution exists, clearly lies in $[0,1)$ and is unique whenever $\mathbb{E}[h^2]>0$ or if $\beta<1$. We will consider the Sherrington-Kirkpatrick measure $\mu$ on $\{-1,1\}^n$ at inverse temperature $\beta\geq 0$ and with i.i.d. external fields drawn from $\mathcal{N}(\mu,\sigma^2)$:
\begin{equation}
\label{eq:sk_with_external}
    \mu_{\beta,A,\bm{h}}(\bm{x})\propto \exp\left(\frac{\beta}{2}\bm{x}^TA\bm{x}+\sum_{i=1}^n h_ix_i\right),
\end{equation}
where $A\sim \mathsf{GOE}(n)$ and each $h_i\sim \mathcal{N}(\mu,\sigma^2)$ are independent of all other random variables. Furthermore, define the following function for $q\leq q'\leq 1$ and $0<m\leq 1$:

\begin{equation*}
    \Phi_{\beta,q}(m,q')=\log 2+\frac{\beta^2}{4}(1-q')^2-\frac{\beta^2}{4}m(q'^2-q^2)+\frac{1}{m}\mathbb{E}_{h,z}\left[\log\mathbb{E}_{z'}\left[\cosh^m\left(\beta z\sqrt{q}+\beta z'\sqrt{q'-q}+h\right)\right]\right],
\end{equation*}
where $z,z'$ are standard normal and $h\sim \mathcal{N}(\mu,\sigma^2)$.

Throughout this section, we will impose the following "high-temperature" assumptions:

\begin{assumption}[High-Temperature Region \cite{talagrand2010mean}]
\label{assumption:high_temp}
    Let $\mathcal{A}=\mathcal{A}(\overline{\beta},\underline{h},\overline{h},\underline{\sigma}^2,\overline{\sigma}^2)\subseteq \mathbb{R}^3$, denote the region $[0,\overline{\beta}]\times [\underline{h},\overline{h}]\times [\underline{\sigma}^2,\overline{\sigma}^2]$. For a given set of parameters $(\beta,\mu,\sigma^2)\in \mathcal{A}$, let $q=q(\beta,\mu,\sigma^2)$ be as in \Cref{eq:RSq}. Suppose that for each $(\beta,\mu,\sigma^2)\in \mathcal{A}$, the following holds:
    \begin{enumerate}
        \item For all $1\geq q'>q=q(\beta,\mu,\sigma^2)$, it holds that
        \begin{equation}
        \label{eq:high_temp_derivative}
        \frac{\partial}{\partial m} \Phi_{\beta,q}(m,q')\bigg\vert_{m=1}<0.
        \end{equation}
        \item It holds that
        \begin{equation}
        \label{eq:at_line}
        \beta^2\mathbb{E}_{h,z}\left[\frac{1}{\cosh^4(\beta z\sqrt{q}+h)}\right]< 1.
        \end{equation}
    \end{enumerate}
    Note that these inequalities are both continuous in the parameters and so are satisfied on $\mathcal{A}$ with some uniform additive gap $\delta=\delta(\mathcal{A})$.
\end{assumption}

\begin{remark}
    The second condition \Cref{eq:at_line} is the well-known ``de Almeida-Thouless (AT) line'' where it is expected that the Sherrington-Kirkpatrick model is replica-symmetric: the law of the overlap between two independent samples from the corresponding measure with external field drawn i.i.d. concentrates about the deterministic value $q$. The first condition \Cref{eq:high_temp_derivative} is a somewhat technical condition of Talagrand that suffices for the validity of the replica-symmetric solution to the free energy of the model; it morally asserts that the so-called ``1-replica-symmetry-breaking'' upper bound of the free energy does not strictly improve on the replica-symmetric bound. We remark that the only reason for considering \Cref{eq:high_temp_derivative} is that this is provably sufficient for Talagrand's ``bounds for coupled copies'' (Chapter 13 of \cite{talagrand2011mean}); conjecturally, this condition should not be necessary.
\end{remark}

We may now state our main result of this section:
\begin{thm}
\label{thm:sk_external}
    Let $\mathcal{A}=\mathcal{A}(\overline{\beta},\underline{h},\overline{h},\underline{\sigma}^2,\overline{\sigma}^2)$ satisfy \Cref{assumption:high_temp}, where either $\underline{h}>0$ or $\underline{\sigma}^2>0$. Let $A\sim \mathsf{GOE}(n)$ and suppose that $(\beta,\mu,\sigma^2)\in \mathcal{A}$. Suppose further that one of the following two conditions holds on the true Sherrington-Kirkpatrick model \Cref{eq:sk_with_external} where each $h_i\sim \mathcal{N}(\mu,\sigma^2)$:
    \begin{enumerate}
        \item (Small Inverse Temperature) It holds that $\overline{\beta}<1/2-\eta$ for some $\eta>0$ (that is, the inverse temperature is small enough).
        \item (Deterministic Field) It holds that $\overline{\sigma}=0$ (that is, the external field is deterministic and parallel to $\mathbf{1}$).
    \end{enumerate}

    Then for any $\delta>0$, there exists a constant $C_{\ref{thm:sk_external}}=C_{\ref{thm:sk_external}}(\mathcal{A},\delta)$ ($=C_{\ref{thm:sk_external}}(\mathcal{A},\delta,\eta)$ in the first case) such that with probability at least $1-\delta$ over the randomness of $A\sim \mathsf{GOE}(n)$ and $\bm{h}\sim \mathcal{N}(\mu\cdot \bm{1},\sigma^2\cdot I)$, the following statement holds for $\mu_{\beta,A,\bm{h}}$. For any $\delta'>0$, with probability at least $1-\delta'$, given $\exp(C_{\ref{thm:sk_external}}\sqrt{\log n})\cdot \mathsf{poly}(n\log (n/\delta')/\alpha)$ samples from $\mu_{\beta,A,\bm{h}}$, the output of the learning algorithm $(\widehat{A},\widehat{h})$ will satisfy $\mathsf{d}_{TV}(\mu_{\beta,A,\bm{h}},\mu_{\widehat{A},\widehat{h}})\leq \alpha$.
\end{thm}

To clarify the quantifiers, the above theorem says that except with arbitrarily small (constant) probability, the random SK measure will satisfy deterministic conditions that will ensure that the learning algorithm succeeds with high probability given sufficient (polynomial) samples, where this latter probability is taken just over the samples and the sub-polynomial factor of the sample complexity depends on the original probability parameter.

To establish \Cref{thm:sk_external}, we will follow similar reasoning to \Cref{thm:sk_no_external} by leveraging recent bounds on the operator norm of the covariance matrix. Notice that the argument actually required control of the \emph{second moment} matrix of the measure; in the zero field case, this is identical to the covariance matrix. With external fields, the covariance matrix has a nontrivial negative rank-one spike from the mean of the measure. For notation, we write
\begin{equation*}
    \bm{m}=\bm{m}_{\beta,A,\bm{h}} =\mathbb{E}_{X\sim \mu_{\beta,A,\bm{h}}}[X]\in \mathbb{R}^n.
\end{equation*}

From well-known results in the statistical physics literature, we expect that $\|\bm{m}\|_2^2\approx qn$ with very high probability, so the rank-one spike $\bm{m}\bm{m}^T$ cannot be neglected when attempting similar calculations. In fact, from the variational characterization of eigenvalues, the second moment matrix will truly be spectrally quite large. We show that this does not affect the learning guarantees by arguing that $\bm{m}$ will be roughly orthogonal with the rows of $A$ despite their nontrivial dependencies\footnote{If $\bm{m}$ were independent of the rows of $\beta A$, this would be a simple consequence of Gaussian concentration.}; this will in turn be a consequence of the (proof of the) \emph{Thouless-Anderson-Palmer (TAP) equations} as rigorously established by Talagrand.

We will require the following two results to make this argument precise. Both can be obtained from the existing spin glass literature; for the interested reader, we explain how to derive them in \Cref{sec:appendix}. The first are bounds on the covariance matrices for the Sherrington-Kirkpatrick model with external field given by Brennecke, Xu, and Yau \cite{brennecke2023operator}.
\begin{thm}
\label{thm:cov_bound_external}
    Let $\mathcal{A}$ satisfy the conditions of \Cref{assumption:high_temp} and let $(\beta,\mu,\sigma^2)\in \mathcal{A}$. Suppose that either:
    \begin{enumerate}
        \item (Small Inverse Temperature) It holds that $\overline{\beta}<1/2-\eta$ for some $\eta>0$.
        \item (Deterministic Field) It holds that $\overline{\sigma}=0$ (that is, the external field is deterministic and parallel to $\mathbf{1}$).
    \end{enumerate}
    
    Then for any $\delta>0$, there exists $K_{\ref{thm:cov_bound_external}}(\mathcal{A},\delta,\eta)$ independent of $n$ such that with probability at least $1-\delta$ over $A\sim \mathsf{GOE}(n)$,
    \begin{equation*} \|\mathrm{Cov}(\mu_{\beta,A,\bm{h}})\|_{\mathsf{op}}\leq K_{\ref{thm:cov_bound_external}}(\mathcal{A},\delta,\eta).
    \end{equation*}
\end{thm}

We also have the following key theorem which is instrumental in Talagrand's proof of the TAP equations \cite{talagrand2010mean}:

\begin{thm}
\label{thm:tap_consequence}
    Let $\mathcal{A}=\mathcal{A}(\overline{\beta},\underline{h},\overline{h},\underline{\sigma}^2,\overline{\sigma}^2)$ satisfy the conditions of \Cref{assumption:high_temp} with either $\underline{h}>0$ or $\underline{\sigma}^2>0$. There is a constant $C_{\ref{thm:tap_consequence}}(\mathcal{A})>0$ such that for any $(\beta,\mu,\sigma^2)\in \mathcal{A}$, there exists standard Gaussian random variables $z_1,\ldots,z_n$ such that
    \begin{equation*}
        \mathbb{E}_{A,\bm{h},\bm{z}}\left[\max_{i\in [n]} \left\vert \beta\sum_{j\neq i} A_{ij}\bm{m}_j -\beta^2(1-q)\bm{m}_i -\beta z_i\sqrt{q}\right\vert\right]\leq \frac{C_{\ref{thm:tap_consequence}}(\mathcal{A})}{n^{1/4}}.
    \end{equation*}
\end{thm}
\begin{remark}
    The Gaussian random variables in the above are not independent of the disorder, nor are they independent of each other. The key point is that one can define random variables $z_1,\ldots,z_n$ on the same sample space as $A,\bm{h}$ such that the above conclusion holds and each $z_i$ is marginally standard normal. The rate in $n$ can also be easily improved to $n^{1/2-\gamma}$ for arbitrarily small $\gamma>0$ using techniques of Talagrand, though this is not necessary for our application. We further note that the requirement on $\underline{h}>0$ or $\underline{\sigma}>0$ is almost certainly technical and could likely be removed. It arises from a certain ``discontinuity'' when $q\to 0$ (see Chapter 1.6 of \cite{talagrand2010mean} for more discussion), even though the claim remains true in the limiting case trivially since all terms are identically zero.
\end{remark}

Again, we defer the proof of this result to \Cref{sec:appendix}; it follows with some moderate effort from existing results in the spin glass literature due to Talagrand \cite{talagrand2010mean}. However, we obtain the following crucial corollary:

\begin{corollary}
\label{cor:tap_bound}
    Let $\mathcal{A}=\mathcal{A}(\overline{\beta},\underline{h},\overline{h},\underline{\sigma}^2,\overline{\sigma}^2)$ satisfy the conditions of \Cref{assumption:high_temp} with either $\underline{h}>0$ or $\underline{\sigma}^2>0$. Then for any $(\beta,\mu,\sigma^2)\in \mathcal{A}$, and any $\delta>0$, we have\footnote{The dependence on $\delta$ can be improved by showing a version of \Cref{thm:tap_consequence} for higher powers. However, since we can only take $\delta>0$ constant anyway, we omit this improvement.}
    \begin{equation*}
        \Pr_{A,\bm{h}}\left(\max_{i\in [n]} \left\vert \beta\sum_{j\neq i} A_{i,j} \bm{m}_j\right\vert\leq \beta\sqrt{2\log (4n/\delta)}+\beta^2+\frac{2C_{\ref{thm:tap_consequence}}(\mathcal{A})}{\delta n^{1/4}}\right)\geq 1-\delta.
    \end{equation*}
\end{corollary}
\begin{proof}
    This follows readily from \Cref{thm:tap_consequence} and standard bounds on the maximum of Gaussian random variables (no matter how they are correlated). 
    First, we note that by Markov's inequality and \Cref{thm:tap_consequence}, we have
    \begin{equation}
    \label{eq:markov_tap}
        \Pr_{A,\bm{h},\bm{z}}\left(\max_{i\in [n]} \left\vert \beta\sum_{j\neq i} A_{ij}\bm{m}_j -\beta^2(1-q)\bm{m}_i-\beta z_i\sqrt{q}\right\vert\geq \frac{2C(\mathcal{A})}{\delta n^{1/4}}\right)\leq \delta/2.
    \end{equation}

    Next, for any $n$ standard normal random variables $z_1,\ldots,z_n$ on some probability space, we have
    \begin{equation}
    \label{eq:gaussian_concentration}
        \Pr\left(\max_{i\in [n]}\left\vert z_i\right\vert\geq \sqrt{2\log(4n/\delta)}\right)\leq n\Pr(\vert z_1\vert\geq \sqrt{2\log(4n/\delta)})\leq 2n\cdot \exp(-\log (4n/\delta))=\frac{\delta}{2},
    \end{equation}
   by \Cref{fact:gaussian_tails} and a union bound. Since the complement of these two events imply the event of the corollary statement by the triangle inequality, union bounding \Cref{eq:markov_tap,eq:gaussian_concentration} gives the desired claim, noting that $q,\bm{m}_i\in [0,1]$.
\end{proof}

With these results in hand, we may formally show that the rows of $\beta\cdot A$ are typically orthogonal to samples from $\mu_{\beta,A,\bm{h}}$, which is sufficient to carry out the proof of learnability.

\begin{corollary}
\label{cor:approx_orth}
    Let $\mathcal{A}=\mathcal{A}(\overline{\beta},\underline{h},\overline{h},\underline{\sigma}^2,\overline{\sigma}^2)$ satisfy the conditions of \Cref{assumption:high_temp} with either $\underline{h}>0$ or $\underline{\sigma}^2>0$, and also assume one of the two additional conditions of \Cref{thm:cov_bound_external}. Then for any $(\beta,\mu,\sigma^2)\in \mathcal{A}$, and any $\delta>0$, the following holds with probability at least $1-\delta$ over $A\sim \mathsf{GOE}(n)$ and $\bm{h}\sim \mathcal{N}(\mu\cdot \bm{1},\sigma^2\cdot I)$. 
    \begin{enumerate}
        \item It holds that
        \begin{equation*}
            \max_{i\in [n]} \|\beta A_i\|_1+\vert h_i\vert\leq \beta\left(\sqrt{n}+\sqrt{\frac{\log(8n/\delta)}{c_{\ref{fact:norm_concentration}}}}\right)+ \mu+\sqrt{2\sigma^2\ln(8n/\delta)}.
        \end{equation*}

        \item For $X\sim \mu_{\beta,A,\bm{h}}$ and any $j\in [n]$, it holds with probability at least $3/4$ over $X$ that:
    \begin{align*}
        \left\vert \beta\langle X,A_j\rangle + h_j\right\vert &\leq 4\beta\left(1+\sqrt{\frac{\log(8n/\delta)}{c_{\ref{fact:norm_concentration}}n}}\right)\cdot \sqrt{K_{\ref{thm:cov_bound_external}}(\mathcal{A},\delta/4,\eta)}\\
        &+4\beta\sqrt{2\log \left(\frac{16n}{\delta}\right)}+\frac{32C_{\ref{thm:tap_consequence}}(\mathcal{A})}{\delta n^{1/4}}+\beta^2+ \mu+\sqrt{2\sigma^2\ln\left(\frac{8n}{\delta}\right)}.
    \end{align*}
    \end{enumerate}
\end{corollary}
While this bound looks complicated, we stress that for fixed $\mathcal{A}$ and $\delta>0$, the bound grows only like $\sqrt{\log n}$ up to constants.
\begin{proof}[Proof of \Cref{cor:approx_orth}]
    First, observe by \Cref{fact:gaussian_tails} and a union bound, we have that
    \begin{equation*}
        \Pr\left(\max_{i\in [n]} \vert h_i\vert \geq \mu+\sqrt{2\sigma^2\ln(8n/\delta)}\right)\leq \delta/4.
    \end{equation*}
    We also have by \Cref{fact:norm_concentration} and a union bound that
    \begin{equation*}
        \Pr\left(\max_{i\in [n]} \|\beta A_j\|_2\geq \beta\left(1+\sqrt{\frac{\log(8n/\delta)}{c_{\ref{fact:norm_concentration}}n}}\right)\right)\leq \delta/4.
    \end{equation*}
    We then apply \Cref{thm:cov_bound_external} and \Cref{cor:tap_bound} with error probability $\delta/4$ to conclude that with probability at least $1-\delta$ over $A\sim \mathsf{GOE}(n)$ and $\bm{h}\sim \mathcal{N}(\mu\cdot \bm{1},\sigma^2\cdot I)$, the following four events hold:
    \begin{gather*}
        \max_{i\in [n]} \vert h_i\vert \leq \mu+\sqrt{2\sigma^2\ln(8n/\delta)}\\
        \max_{i\in [n]} \|\beta A_j\|_2\leq \beta\left(1+\sqrt{\frac{\log(8n/\delta)}{c_{\ref{fact:norm_concentration}}n}}\right)\\
        \|\mathrm{Cov}(\mu_{\beta,A,\bm{h}})\|_{\mathsf{op}}\leq K_{\ref{thm:cov_bound_external}}(\mathcal{A},\delta/4,\eta)\\
        \max_{i\in [n]} \left\vert \beta\sum_{j\neq i} A_{i,j} \bm{m}_j\right\vert\leq \beta\sqrt{2\log (16n/\delta)}+\beta^2+\frac{8C_{\ref{thm:tap_consequence}}(\mathcal{A})}{\delta n^{1/4}}.
    \end{gather*}
    Holder's inequality with the first two inequalities implies the validity of the first inequality in the theorem statement.

    We now show that the second statement holds on this event. We easily compute for any fixed $j\in [n]$,
    \begin{align*}
        \mathbb{E}_X\left[\left\vert\beta\langle X,A_j\rangle\right\vert\right]&\leq \mathbb{E}_X\left[\left\vert\beta\langle X-\bm{m},A_j\rangle\right\vert\right]+\left\vert \beta\sum_{j\neq i} A_{i,j} \bm{m}_j\right\vert\\
        &\leq \mathbb{E}_X\left[\left\vert\beta\langle X-\bm{m},A_j\rangle\right\vert^2\right]^{1/2}+\beta\sqrt{2\log (16n/\delta)}+\beta^2+\frac{8C_{\ref{thm:tap_consequence}}(\mathcal{A})}{\delta n^{1/4}}\\
        &\leq \|\beta A_j\|_2 \sqrt{\|\mathrm{Cov}(\mu_{\beta,A,\bm{h}})\|_{\mathsf{op}}}+\beta\sqrt{2\log (16n/\delta)}+\beta^2+\frac{8C_{\ref{thm:tap_consequence}}(\mathcal{A})}{\delta n^{1/4}}\\
        &\leq \beta\left(1+\sqrt{\frac{\log(8n/\delta)}{c_{\ref{fact:norm_concentration}}n}}\right)\cdot \sqrt{K_{\ref{thm:cov_bound_external}}(\mathcal{A},\delta/4,\eta)}+\beta\sqrt{2\log (16n/\delta)}+\beta^2+\frac{8C_{\ref{thm:tap_consequence}}(\mathcal{A})}{\delta n^{1/4}}.
    \end{align*}
    The stated bound holds with probability at least $3/4$ over $X\sim \mu_{\beta,A,\bm{h}}$ by applying a triangle inequality to deal with the $h_j$ term on this good event and then Markov's inequality over $X\sim \mu_{\beta,A,\bm{h}}$ to deal with the inner product.
\end{proof}

After all of these intermediate results, we may now easily conclude \Cref{thm:sk_external} in the same way as \Cref{thm:sk_no_external}. The only difference is slightly more bookkeeping for recovery of the external field.
\begin{proof}[Proof of \Cref{thm:sk_external}]
    \Cref{cor:approx_orth} shows that under the assumptions of the theorem, with probability at least $1-\delta$ over $A\sim \mathsf{GOE},\bm{h}\sim \mathcal{N}(\mu\cdot \bm{1},\sigma^2\cdot I)$:

    \begin{enumerate}
        \item It holds that
        \begin{equation*}
            \max_{i\in [n]} \|\beta A_i\|_1+\vert h_i\vert\leq \beta\left(\sqrt{n}+\sqrt{\frac{\log(8n/\delta)}{c_{\ref{fact:norm_concentration}}}}\right)+ \mu+\sqrt{2\sigma^2\ln(8n/\delta)}\triangleq C_1(\mathcal{A},n,\delta).
        \end{equation*}

        \item For $X\sim \mu_{\beta,A,\bm{h}}$ and any $j\in [n]$, it holds with probability at least $3/4$ over $X$ that:
    \begin{align*}
        \left\vert \beta\langle X,A_j\rangle + h_j\right\vert &\leq 4\beta\left(1+\sqrt{\frac{\log(8n/\delta)}{c_{\ref{fact:norm_concentration}}n}}\right)\cdot \sqrt{K_{\ref{thm:cov_bound_external}}(\mathcal{A},\delta/4,\eta)}\\
        &+4\beta\sqrt{2\log \left(\frac{16n}{\delta}\right)}+\frac{32C_{\ref{thm:tap_consequence}}(\mathcal{A})}{\delta n^{1/4}}+\beta^2+ \mu+\sqrt{2\sigma^2\ln\left(\frac{8n}{\delta}\right)}\triangleq C_2(\mathcal{A},n,\delta,\eta).
    \end{align*}
    \end{enumerate}

    We may therefore apply \Cref{thm:ising_cov} with $C_{\ref{thm:ising_cov}}=C_2=O_{\mathcal{A},\delta,\eta}(\sqrt{\log n})$, which shows that on this event, for any $j\in [n]$ and letting $\bm{w}^*=A_j$ and $h^*=h_i$, we have the lower bounds for any $(\bm{w},h)$
        \begin{equation*}
        \mathbb{E}_{X}\left[\left(\sigma\left(2(\langle \bm{w},X\rangle+h)\right)-\sigma\left(2(\langle \bm{w}^*,X\rangle+h^*)\right)\right)^2\right]\geq \exp(-O(C_{\ref{thm:ising_cov}}))\cdot \min\{1,8(w_j-w^*_j)^2\}.
    \end{equation*}
    and for any $(\bm{w},h)\in \mathbb{R}^{n-1}\times \mathbb{R}$ such that $\|\bm{w}-\bm{w}^*\|_\infty\leq \vert h-h^*\vert/2n$
    \begin{equation*}
        \mathbb{E}_{X}\left[\left(\sigma\left(2(\langle \bm{w},X\rangle+h)\right)-\sigma\left(2(\langle \bm{w}^*,X\rangle+h^*)\right)\right)^2\right]\geq \exp(-O(C_{\ref{thm:ising_cov}}))\cdot \min\{1,(h-h^*)^2\}.
    \end{equation*}

    Let $\mathcal{C}=\{(\bm{w},h)\in \mathbb{R}^{n-1}\times \mathbb{R}: \|(\bm{w},h)\|_1\leq C_1\}$ so that each $(\beta A_j,h_j)\in \mathcal{C}$. Let $Z_1,\ldots,Z_T\in \{-1,1\}^n$ denote the independent samples from $\mu_{\beta,A,\bm{h}}$. For a fixed node $i\in [n]$, as usual, write $X_t=Z_{t,-i}$ and $Y_t=Z_{t,i}$. Standard uniform convergence bounds (again, see Theorem 26.12 of \cite{books/daglib/0033642}) again imply that for node $i$'s logistic regression problem, with probability at least $1-\delta'$, we have the one-sided deviation bounds
        \begin{equation*}
        \sup_{(\bm{w},h)\in \mathcal{C}} \frac{1}{T}\sum_{t=1}^T \mathbb{E}[\ell(2Y_t\langle \bm{w},X_t\rangle)]-\frac{1}{T}\sum_{t=1}^T \ell(2Y_t\langle \bm{w},X_t\rangle) \leq O\left(C_1\sqrt{\frac{\ln(n/\delta)}{T}}\right).
    \end{equation*}

    Combined with the two-sided deviation bound \Cref{lem:opt_azuma} for $(\bm{w}^*,h^*)$, we find  $T(\varepsilon,\delta')\leq O(C_1^2\ln(n/\delta')/\varepsilon^2)$. We now may apply \Cref{thm:wsd} to obtain an estimate $\widehat{A}$ satisfying $\|\widehat{A}-\beta A\|_{\infty}\leq \varepsilon$ with probability at least $1-\delta'$ given
    \begin{equation*}
        T\geq \frac{O(C_1^2)\exp(O(C_2))\log(n/\delta')}{\varepsilon^4}.
    \end{equation*}
    Setting $\varepsilon=\alpha^2/8n^2$, the last conclusion of \Cref{thm:wsd} further implies that on the event $\|\widehat{A}-\beta A\|_{\infty}\leq \varepsilon$, we must also have
    \begin{equation*}
        \|\widehat{\bm{h}}-\bm{h}\|_{\infty}\leq 2n\varepsilon=\alpha^2/4n.
    \end{equation*}
    By \Cref{lem:additive_suff}, we thus find that
    \begin{equation*}
        \mathsf{d}_{TV}(\mu_{\widehat{A},\widehat{\bm{h}}},\mu_{\beta,A,\bm{h}})\leq \frac{\alpha}{2\sqrt{2}}+\frac{\alpha}{2}<\alpha,
    \end{equation*}
    on this event, as claimed. This completes the proof, noting that $C_1=O_{\mathcal{A},\delta}(\sqrt{n})$ and $C_2=O_{\mathcal{A},\delta,\eta}(\sqrt{\log n})$.
\end{proof}

\section{Extensions: Off-Equilibrium and Adversarial Dynamics}
\label{sec:extensions}

In this last section, we demonstrate how our above framework, though quite elementary, can rather efforlessly lead to significant improvements in the learning guarantees of other dynamic settings. In \Cref{sec:m_regime}, we study the setting of learning in the $M$-regime, as considered by Dutt, Lokhov, Vuffray, and Misra \cite{DBLP:conf/icml/DuttLVM21}. In this model, the independent samples from the Ising model one receives are one-step Glauber dynamics started from a uniformly random configuration in $\{-1,1\}^n$. This setting is meant to capture the learning problem of Glauber dynamics far from equilibrium. As was shown by Dutt, et al., samples in the M-regime appear to be significantly more informative about the underlying Ising parameters, both in theory and in practice. 

Their main theoretical result is that there exists an algorithm, D-RISE, such that for Ising models with maximum degree $d$ and maximum edge strength $\beta$ (with some minimum edge strength $\alpha>0$), the sample complexity of learning the support scales in the width at most like $\exp(2\beta d)$ per node \cite{DBLP:conf/icml/DuttLVM21}. This provable guarantee is stronger than any known algorithm in the exponent than any known algorithm in the i.i.d. case from the stationary measure, and is closer to the information-theoretic scaling lower bound in this setting of $\exp(\beta d)$ established by Santhanam and Wainwright \cite{DBLP:journals/tit/SanthanamW12}. Experimentally, Dutt, et al. also show that the scaling of these algorithms are even better than these theoretical bounds.

Using our simple framework, we show that the gains of learning in the M-regime are significantly more pronounced. Conceptually, the reason is transparent within our analysis: with a uniformly random configuration before each Glauber step, the magnitude of the effective field of the updated site will typically be of order the $\ell_2$ norm of the row of interactions, rather than the $\ell_1$ norm. As in the previous sections, the existence of this typical event is sufficient to conclude that the population loss in logistic regression measures the variation from the true interaction vector component-wise by paying a bound that is exponential in $\beta\sqrt{d}$ rather than $\beta d$. In particular, this result shows that structure recovery in the M-regime is \emph{strictly easier} than learning from i.i.d. samples from the stationary measure.

In \Cref{sec:adversarial_glauber}, we will consider a simple model of \emph{adversarial Glauber dynamics}, which was recently introduced by Chin, Moitra, Mossel, and Sandon \cite{adversarial}.\footnote{The exact adversary model considered by Chin, et al., is strictly stronger than the one we consider here. In their work, adversarial nodes are allowed to apply an arbitrary update rule, coordinate amongst themselves, and even update at any time, not just when selected for updating.} In this model, some fraction of nodes are corrupted and may update according to an arbitrary decision rule at each time they are selected for updating. Using our framework for sequential block dynamics, it will be straightforward to show that on bounded degree graphs, if adversarial nodes are \emph{smooth and local}, then running logistic regression will nonetheless succeed in recovering the interaction coefficients for non-corrupted nodes. In particular, if the identities of corrupted nodes are known, then one can recover the interactions between non-corrupted nodes. We will make these notions more precise, but they roughly correspond to somewhat mild constraints on how adversarial nodes may update. A more refined notion of learning guarantees under fully adversarial behavior remains an interesting open question we leave for future work.

Throughout the remainder of this section, we focus on the case with no external field for simplicity.

\subsection{Learning in the M-Regime}
\label{sec:m_regime}

We now proceed to formally the sample generation process in the so-called M-regime. In words, the samples in the M-regime are obtained by doing a single step of Glauber dynamics from a uniformly random initial configuration. The formal definition is as follows:

\begin{defn}[M-Regime \cite{DBLP:conf/icml/DuttLVM21}]
    Let $A\in \mathbb{R}^{n\times n}$ be the interaction matrix of an Ising model. Samples $(X_1,X_1',i_1),\ldots,(X_T,X_T',i_T)\in \{-1,1\}^n\times \{-1,1\}^n\times [n]$ from the M-regime from $\mu_A$ are obtained as follows. For each $t\geq 1$,
    \begin{enumerate}
        \item Sample $X_t\sim \{-1,1\}^n$ uniformly and independently.
        \item Sample $i_t\sim [n]$ uniformly at random, and update the $i_t$th coordinate of $X_t$ via \Cref{eq:glauber}. Let $X_t'$ denote the new configuration.
    \end{enumerate}
    For a node $j\in [n]$, let $I_j\subseteq [T]$ denote the set of times $t$ where $i_t=j$. Then the samples generated for $j$ are the samples $(X_{\ell,-j},X'_{\ell,j})$ for $\ell\in I_j$.\footnote{Note that for each $\ell\in I_j$, the law of $X'_{\ell,j}$ conditional on all other samples is indeed given by \Cref{eq:glauber} conditional on $X_{\ell,-j}$.}
\end{defn}

We are now ready to state and prove the main result of this subsection:

\begin{thm}
\label{thm:m-regime}
    Let $\mathcal{C}$ be a closed convex set and let $A$ be an Ising model such that $A_j\in \mathcal{C}$ for each $j\in [n]$. Let $R_1\triangleq \max_{\bm{z}\in \mathcal{C}}\|z\|_1$ and $R_2\triangleq \max_{j\in [n]} \|A_j\|_2$. Then given at least 
    \begin{equation*}
        \frac{O\left(R_1^2\exp(O(R_2))\cdot \log(n/\delta)\right)}{\varepsilon^4}
    \end{equation*}
    updates per node from M-regime samples, it holds with probability at least $1-\delta$ that running logistic regression as in \Cref{eq:empiricalopt} with constraint set $\mathcal{C}$ for each $j\in [n]$ returns a matrix $\widehat{A}$ such that $\|A-\widehat{A}\|_{\infty}\leq \varepsilon$.
\end{thm}
\begin{proof}
    Note that by assumption, each $A_j$ is feasible for the corresponding logistic regression problem for node $j$. The result is now essentially immediate from \Cref{cor:ising_cov}, with simple modifications. 
    For each $j\in [n]$, and any $\ell\in I_j$, the covariance matrix of $X_{\ell,-j}$ is the identity matrix by the assumption $X_{\ell}$ is uniform on $\{-1,1\}^n$ and the fact that the updated index is independent of the configuration. Therefore, the same argument as \Cref{cor:ising_cov} applies but with the smaller constraint set $\mathcal{C}$. This argument thus gives the sample complexity bound for each node of
    \begin{equation*}
        \frac{O\left(R_1^2\exp(O(R_2))\cdot\log (n/\delta)\right)}{\varepsilon^4}
    \end{equation*}
    to obtain an $\varepsilon$ additive approximation with probability at least $1-\delta$ for each node.
\end{proof}

While quite simple, \Cref{thm:m-regime} already implies that with bounded degree interaction matrices $A$, the sample complexity of recovering the support of $A$ from samples in the M-regime scales like $\exp(O(\beta\sqrt{d}))$, an exponential improvement (per node) over what is information-theoretically possible given samples from the stationary distribution:

\begin{corollary}
\label{cor:sparse_m_regime}
    Let $A\in \mathbb{R}^{n\times n}$ be an interaction matrix with maximum degree at most $d$, maximum absolute coefficient at most $\beta>0$, and minimum nonzero coefficient at least $\alpha>0$. Then, given at least 
    \begin{equation*}
        \frac{O((\beta d)^2\exp(O(\beta\sqrt{d})))\log (n/\delta)}{\varepsilon^4}
    \end{equation*} 
    M-regime samples per node, running logistic regression for each node with $\mathcal{C}=\{\bm{w}\in \mathbb{R}^{n-1}:\|\bm{w}\|_1\leq \beta d\}$ as in \Cref{eq:empiricalopt} yields $\widehat{A}$ such that $\|\widehat{A}-A\|_{\infty}\leq \varepsilon$ with probability at least $1-\delta$. 

    In particular, if $\varepsilon<\alpha/2$, then
    \begin{equation*}
        \text{supp}(A)=\{(i,j)\in [n]^2: \vert\widehat{A}_{i,j}\vert \geq \alpha/2\}.
    \end{equation*}
    Thus, given at least $O((\beta d)^2)\exp(O(\beta\sqrt{d}))\log (n/\delta)/\alpha^4$ updates for each node, with probability at least $1-\delta$, one may recover the support of $A$.
\end{corollary}
\begin{proof}
    It is well-known that the second claim is a consequence of the first; it is clear that if the minimum nonzero edge strength is at least $\alpha>0$, then if $\|A-\widehat{A}\|_{\infty}< \alpha/2$, the absolute value of $\widehat{A}$ for each edge must be at least $\alpha/2$ while the value for each nonedge must be strictly smaller. Thresholding thus recovers the true support of $A$.

    Observe that each row is indeed feasible for the logistic regression problem by the assumption on the degree and maximum absolute coefficient condition. Moreover, under our assumptions, each row of $A$ has $\ell_2$ norm at most $\beta\sqrt{d}$ since
    \begin{equation*}
        \|A_j\|_2\leq \sqrt{\|A_j\|_1\|A_j\|_{\infty}}\leq \sqrt{\beta^2 d}=\beta \sqrt{d}.
    \end{equation*}
    to obtain an $\varepsilon$ additive approximation with probability at least $1-\delta$ for each node. Applying \Cref{thm:m-regime} gives the conclusion.
\end{proof}

Returning to the setting of the Sherrington-Kirkpatrick model, we also easily obtain that in the M-regime, there is effectively no ``low-temperature'' regime for learning:

\begin{corollary}
    Let $\overline{\beta}\geq 0$ be fixed and let $\mathcal{C}=\{\bm{w}\in \mathbb{R}^{n-1}:\|\bm{w}\|_1\leq 2\overline{\beta}\sqrt{n}$. Then with probability $1-o(1)$ over $A\sim \mathsf{GOE}(n)$, the following holds: for any $\beta\leq \overline{\beta}$, it holds with probability at least $1-\delta$ that, given at least
    \begin{equation*}
        \frac{O\left(\overline{\beta}^2\cdot n\cdot\exp(O(\overline{\beta})))\cdot \log(n/\delta) \right)}{\varepsilon^4}
    \end{equation*}
    M-regime samples per node of the SK model $\mu_{\beta,A}$, running logistic regression \Cref{eq:empiricalopt} for each node with constraint set $\mathcal{C}$ returns $\widehat{A}$ such that $\|\widehat{A}-\beta A\|_{\infty}\leq \varepsilon$.
\end{corollary}
\begin{proof}
    By \Cref{fact:norm_concentration} and a union bound, we see that with probability at least $1-o(1)$ that $\|\beta A_j\|_2\leq 2\beta$ and then by Cauchy-Schwarz, $\|\beta A_j\|_1\leq 2\beta\sqrt{n}$. As such, each row of $\beta A$ is feasible for the stated logistic regression problem, and applying \Cref{thm:m-regime} gives the stated conclusion.
\end{proof}

To reiterate, the key observation is that since the M-regime returns well-behaved samples regardless of the structural features of the Ising model, the learning problem becomes \emph{much} easier statistically. While the actual Ising model may exhibit challenging correlations at the level of Gibbs measure, the samples used for learning do not reflect them, and thus the learning guarantees need not either.

\subsection{Adversarial Glauber Dynamics}
\label{sec:adversarial_glauber}

We now turn to the learnability of adversarial Glauber dynamics. 
\begin{defn}
    Let $A\in \mathbb{R}^{n\times n}$ be the interaction matrix for an Ising model $\mu_A$. Let $C\subseteq [n]$ denote a subset of \textbf{corrupted nodes} and let $H=[n]\setminus C$ denote the \textbf{honest nodes}. Then given an initial configuration $X_0\in \{-1,1\}^n$, \textbf{$\gamma$-smooth and local adversarial Glauber dynamics} proceeds as follows. For each $t\geq 1$:
    \begin{enumerate}
        \item Sample $i_t\sim [n]$ uniformly at random.
        \item If $i_t\in H$, then $X_t$ is obtained from $X_{t-1}$ by updating the $i_t$th coordinate of $X_{t-1}$ via the conditional measure in \Cref{eq:glauber} given $X_t$.
        \item If $i_t\in C$, then $X_t$ is obtained by $X_{t-1}$ by updating the $i_t$th coordinate according to any distribution $\mathcal{D}_{i_t,t}$ on $\{-1,1\}$ that (i) depends only on $X_{0,\mathcal{N}(i_t)},\ldots, X_{t-1,\mathcal{N}(i_t)}$, and (ii) satisfies $\mathcal{D}_{i_t,t}(\varepsilon)\geq \gamma$ for $\varepsilon\in \{-1,1\}$.
    \end{enumerate}
    In words, $\gamma$-smooth and local adversarial Glauber dynamics are such that honest nodes update according to regular Glauber dynamics, while corrupt nodes can update according to any randomized policy that depends only on the history of spins \emph{in their neighborhood} (locality), and constrained to put at least $\gamma$ mass on each possibility (smoothness).
\end{defn}

Again, our main result of this subsection shows that $\gamma$-smooth and local dynamics cannot meaningfully affect the learning guarantees of logistic regression for the honest nodes:

\begin{thm}
\label{thm:adversarial_glauber}
    Let $A\in \mathbb{R}^{n\times n}$ be the interaction matrix for an Ising model $\mu_A$ with maximum degree at most $d$ and $\ell_1$ width $\lambda$. Consider any $\gamma$-smooth and local adversarial Glauber dynamics and run logistic regression as in \Cref{eq:empiricalopt} for each node. Then given at least
    \begin{equation*}
\tilde{O}\left(\frac{\lambda^2\exp(O(\lambda))\log(n/\delta)}{\varepsilon^4 \gamma^{O(d)}}\right)
    \end{equation*}
    samples for each node, it holds with probability at least $1-\delta$ that the logistic regression output $\widehat{w}_j$ for each \emph{honest} node $j\in H$ obtains a $\varepsilon>0$ additive approximation of $A_j$.
\end{thm}
\begin{proof}[Sketch of \Cref{thm:adversarial_glauber}]
    The proof is very similar to that of \Cref{prop:goodblocks} and considers likelihood ratios. For convenience, suppose node $n$ is an honest node; our first goal is to show that a similar conclusion to \Cref{prop:goodblocks} holds. Formally, one can show the following:

    \begin{claim}
    \label{claim:adv_claim}
        Assume the conditions of \Cref{thm:adversarial_glauber} and suppose node $n$ is honest. Fix any $X_0\in \{-1,1\}^n$ and history of $\gamma$-smooth and local adversarial Glauber dynamics up to time $0$ (which we denote $\mathcal{F}_0$ after reindexing). Let $Z=(X_{\tau_1,-n},X_{\tau_1,n})$ be the configuration at the next update time of node $n$ (as in \Cref{defn:view} after reindexing). For simplicity, write $X=X_{\tau_1,-i}\in \{-1,1\}^{n-1}$.

    Then there exists a constant $c>0$ such that for all $j\neq n$, and any $(\bm{w},h)\in \mathbb{R}^{n-1}\times \mathbb{R}$,
    \begin{equation*}
        \mathbb{E}_{X}\left[\left(\sigma\left(2(\langle \bm{w},X\rangle+h)\right)-\sigma\left(2(\langle \bm{w}^*,X\rangle+h^*)\right)\right)^2\bigg\vert \mathcal{F}_0\right]\geq c\gamma^{O(d)}\exp(-O(\lambda))\min\{1,8(w_j-w_j^*)^2\}.
    \end{equation*}
    \end{claim}

    We merely sketch the proof as it is very similar to \Cref{prop:goodblocks}: consider any $j\neq [n]$, honest or not. As in that proof, let $\mathcal{P}$ denote the sequence of updates before the next update of $n$ and let $\mathcal{P}'$ denote the sequence of updates replacing the last occurrence of $j$ with $*$ (if one exists).

    As in \Cref{prop:goodblocks}, it then suffices to show that with constant probability over $\mathcal{P}'$, it holds that $j$ was indeed updated in $\mathcal{P}$ and that the likelihood ratio of $X_j=\varepsilon$ conditional on $\mathcal{P}'$ is bounded above and below. The former holds with probability at least $1/2$, so it suffices to control the likelihood ratio with probability at least $3/4$ and take a union bound. By writing out and factorizing the likelihood of $\mathcal{P}^{\pm}$ as in \Cref{prop:goodblocks}, we note that the likelihood ratio of each node update after the last update of $j$ given the true configuration with $X_j=\pm 1$ is controlled as follows:
    \begin{enumerate}
        \item If node $k\in \mathcal{N}(j)^c$ is updated, the likelihood ratio is $1$ by locality as the update rule does not depend on the value of $X_j$.
        \item If node $k\in C\cap \mathcal{N}(j)$ is updated, the likelihood ratio is at most $1/\gamma$ by smoothness.
        \item If node $k\in H\cap \mathcal{N}(j)$, the likelihood ratio is bounded by $\exp(O(\vert A_{i,j}\vert))$ by \Cref{lem:ratios}.
    \end{enumerate}
    Let $N_k$ denote the number of times node $k$ is updated after the last update of $j$. Then the likelihood ratio of $X_j$ given $\mathcal{P}'$ is bounded by 
    \begin{equation*}
        \exp\left(O\left(\sum_{k\in \mathcal{N}(j)\cap H} N_k\vert A_{ij}\vert\right)\right)\cdot \gamma^{-O(\sum_{k\in C\cap \mathcal{N}(j)} N_k)}.
    \end{equation*}
    As in \Cref{prop:goodblocks}, by Markov's inequality, we have with probability at least $3/4$ over $\mathcal{P}'$ that 
    \begin{gather*}
        \sum_{k\in \mathcal{N}(j)\cap H} N_k\vert A_{ij}\vert\leq 8\lambda\\
        \sum_{k\in \mathcal{N}(j)\cap C} N_k\leq 8d.
    \end{gather*}
    Combining the previous two displays gives bound on the likelihood ratio, and following the rest of the argument in \Cref{prop:goodblocks} gives the stated conclusion of the claim. At this point, we may apply the uniform martingale bounds of \Cref{thm:radbound} with the bound of \Cref{claim:adv_claim} to the meta-theorem \Cref{thm:wsd} to obtain the result.
\end{proof}

\section{Conclusions and Open Questions}
In this work, we have shown that logistic regression can efficiently learn Ising models under more general conditions than previously known. Our analysis identifies  refined structural conditions that more accurately reflect the statistical and computational hardness of learning specific Ising models in a variety of regimes. We conclude with a few interesting open directions:

\begin{question}
    What is the sample complexity of learning from dynamics where one only observes updates when the configuration changes?
\end{question}
Note that in \Cref{thm:learning_block}, as well as the prior work of Bresler, Gamarnik, and Shah \cite{DBLP:journals/tit/BreslerGS18} and Dutt, Lokhov, Vuffray, and Misra \cite{DBLP:conf/icml/DuttLVM21}, we have assumed that the indices of vertices that are chosen for updating at each time are observed regardless of whether their value actually changes. In practice, it is much more likely that one only observes the \emph{changes} in configurations and possibly the clock of updating. One expects that an appropriately debiased version of logistic regression may succeed but would likely need new ideas.

\begin{question}
    Can one remove the extraneous $\mathsf{poly}(\log\log n)$ term from \Cref{thm:learning_block}?
\end{question}
Recall that this term arises from applying the contraction principle for sequential Rademacher complexity, and it appears open whether this factor is necessary in general. But for the more structured processes we consider here, it may be possible to give a more direct argument avoiding these bounds or the use of sequential Rademacher complexity altogether.

\begin{question}
Can one extend the results of \Cref{sec:sk} to general mixed $p$-spin models in the entire high-temperature regime?
\end{question}
We expect that our general results should extend to Markov random fields with higher-order interactions without much effort. However, the specific case of mixed $p$-spin models may exhibit an interesting regime where reducing learning to appropriate moment bounds, rather than dynamical notions, may be provably necessary. Anari, Jain, Koehler, Pham, and Vuong \cite{anari_universality} have shown that one can learn very high-temperature $p$-spin models since entropy factorization holds, but this certainly cannot extend to the entire high-temperature regime: unlike the SK model (which corresponds to $p=2$), general mixed $p$-spin models exhibit \emph{shattering} in an interval of the high-temperature (replica-symmetric) regime, which provably prevents fast mixing and thus entropy factorization (see, e.g. \cite{shattering_baj,alaoui2023shattering,gamarnik2023shattering}). It is not clear to us whether shattering precludes the weaker condition of boundedness of appropriate moment matrices (with respect to an appropriate norm) with no external field. 

\begin{question}
    Can one obtain more refined guarantees for learning from dishonest dynamics, possibly under extra assumptions? 
\end{question}
Note that our result \Cref{thm:adversarial_glauber} only ensures that dishonest updates, under some constraints, cannot affect the recovery of honest nodes given enough samples. Beyond improving the sample complexity given there, it would be interesting to understand whether these constraints can be relaxed. Under some additional modelling conditions, it may also be possible to identify corrupt nodes or otherwise minimize their influence. For instance, one may hope to argue that if one cannot identify a corrupt node while learning from dynamics, then the set of corrupt nodes must ``hide'' in a sense that enables improved recovery guarantees. While it is possible to show that one cannot obtain such a guarantee in the exact setting we have considered,  perhaps imposing other restrictions could allow for improved guarantees. See, for instance, Alon, Mossel, and Pemantle \cite{DBLP:journals/toc/AlonMP20} and subsequent work for the related problem of detecting corruption in graphs under somewhat different conditions.

\bibliography{bibliography}
\bibliographystyle{alpha}

\appendix

\section{Deferred Proofs}
\label{sec:appendix}
We now provide the deferred proofs of \Cref{thm:cov_bound_external} and \Cref{thm:tap_consequence} for the interested reader. As mentioned, the former is essentially the main result of \cite{brennecke2023operator}, while the second can be obtained from the proof of the TAP equations of Talagrand \cite{talagrand2010mean}. To begin, we require the following simple fact that holds under the AT line showing that the $q$ parameter in \Cref{eq:RSq} remains well-behaved at high-temperature.

\begin{lem}
\label{lem:q_derivative}
    Let $\mathcal{A}$ satisfy \Cref{assumption:high_temp} and suppose $(\beta,\mu,\sigma^2)\in \mathcal{A}$. Then it holds that
    \begin{equation*}
        \frac{\partial}{\partial \beta} q(\beta,\mu,\sigma^2)\leq C_{\ref{lem:q_derivative}}(\mathcal{A})
    \end{equation*}
    for some finite constant $C_{\ref{lem:q_derivative}}(\mathcal{A})>0$ that holds uniformly on $\mathcal{A}$.
\end{lem}
\begin{proof}
    We follow the argument of Lemma 1.7.5 of Talagrand \cite{talagrand2010mean}, which is stated for $\beta\leq 1/2$ but holds under \Cref{assumption:high_temp}. Let $f(x)=\tanh^2(x)$. Fixing $\mu,\sigma^2$, define $F(\beta,q)=\mathbb{E}_{h,z}[f(\beta z\sqrt{q}+h)]$ so that $q$ satisfies $q(\beta)=F(\beta,q(\beta))$. Letting $f(x)=\tanh^2(x)$, implicit differentiation implies
    \begin{align*}
        q'(\beta)=\frac{\frac{\partial F}{\partial \beta}(\beta,q(\beta))}{1-\frac{\partial F}{\partial q}(\beta,q(\beta))}&=\frac{\mathbb{E}[z\sqrt{q}f'(\beta z\sqrt{q}+h]}{1-\frac{\beta \mathbb{E}[zf'(\beta z\sqrt{q}+h)]}{2\sqrt{q}}}\\
        &=\frac{\beta q\mathbb{E}[f''(\beta z\sqrt{q}+h)]}{1-\beta^2\mathbb{E}[f''(\beta z\sqrt{q}+h)]/2},
    \end{align*}
    where $f''(x)=\frac{2-4\sinh^2(x)}{\cosh^4(x)}$ and the last line follows from applying Gaussian integration by parts on $z$. Since $\vert f''(x)\vert \leq 2$ and we assume $\beta$ is bounded on $\mathcal{A}$ (and $q\in [0,1]$), it suffices to show that the denominator is bounded below by a constant depending on $\mathcal{A}$. But this is clear since
    \begin{equation*}
     \frac{\beta^2}{2}\mathbb{E}\left[\frac{2-4\sinh^2(\beta z\sqrt{q}+h)}{\cosh^4(\beta z\sqrt{q}+h)}\right]\leq \beta^2 \mathbb{E}\left[\frac{1}{\cosh^4(\beta z\sqrt{q}+h)}\right]\leq1-\delta(\mathcal{A})
    \end{equation*}
    for some constant $\delta(\mathcal{A})>0$ where the last step is given by \Cref{eq:at_line} and continuity recalling $\mathcal{A}$ is compact.
\end{proof}

The main consequence we require is the following:
\begin{lem}[see e.g. Talagrand, Lemma 1.7.5 of \cite{talagrand2010mean}]
\label{lem:q_close}
    Let $\mathcal{A}$ satisfy \Cref{assumption:high_temp} and suppose $(\beta,\mu,\sigma^2)\in \mathcal{A}$. Let $\beta_{-}=\beta\sqrt{1-1/n}$. Then
    \begin{equation*}
        \vert q(\beta_{-},\mu,\sigma^2)-q(\beta,\mu,\sigma^2)\vert \leq \frac{\beta C_{\ref{lem:q_derivative}}(\mathcal{A})}{n}.
    \end{equation*}
\end{lem}
\begin{proof}
    We have
    \begin{align*}
        \vert q(\beta_{-},\mu,\sigma^2)-q(\beta,\mu,\sigma^2)\vert&=\left\vert \int_{\beta_{-}}^{\beta} q'(x)\mathrm{d}x \right\vert\\
        &\leq \vert \beta - \beta_{-}\vert C_{\ref{lem:q_derivative}}\\
        &=\beta(1-\sqrt{1-1/n})C_{\ref{lem:q_derivative}}\\
        &\leq \frac{\beta C_{\ref{lem:q_derivative}}}{n},
    \end{align*}
    where we use the fact that $\sqrt{1-x}\geq 1-x$ for $x\in [0,1]$.
\end{proof}

To state the next intermediate results, we require the following notation. 

\begin{defn}
    Given the Sherrington-Kirkpatrick measure $\mu_{\beta,A,\bm{h}}$ as in \Cref{eq:sk_with_external}, let $\sigma^1,\sigma^2,\dots\in \{-1,1\}^n$ denote independent samples from $\mu_{\beta,A,\bm{h}}$. Let 
    \begin{equation*}R_{1,2}\triangleq \frac{1}{n}\sum_{i=1}^n \sigma^1_i\sigma^2_i
    \end{equation*}
    denote the \textbf{overlap} of $\sigma^1,\sigma^2$. 
    
    Finally, let $\nu=\nu_{n,\beta,\mu,\sigma^2}$ denote the measure on $(\{-1,1\}^n)^{\otimes \mathbb{N}}$ defined by
    \begin{equation*}
        \nu \triangleq \mathbb{E}_{A\sim \mathsf{GOE}(n),\bm{h}\sim \mathcal{N}(\mu\cdot \bm{1},\sigma^2 I)}[\mu^{\otimes N}_{A,\beta,\bm{h}}(\cdot)],
    \end{equation*} where $\mu^{\otimes \mathbb{N}}$ denotes the product measure on countably many independent samples from $\mu$.
\end{defn}

Next, we require the following result of Talagrand as stated by Brennecke, Xu, and Yau \cite{brennecke2023operator}:

\begin{thm}[Theorem 13.7.1 of \cite{talagrand2010mean}, see also Theorem 2.1 of \cite{brennecke2023operator}]
\label{thm:moment_bounds}
Suppose $\mathcal{A}$ satisfies \Cref{assumption:high_temp}. Then there exists a constant $K=K(\mathcal{A})$ such that for any $(\beta,\mu,\sigma^2)\in \mathcal{A}$, it holds that
\begin{equation*}
    \nu\left(\exp\left(\frac{n(R_{1,2}-q(\beta,\mu,\sigma^2))^2}{K}\right)\right)\leq 2.
\end{equation*}
Equivalently, the random variable $\sqrt{n}(R_{1,2}-q(\beta,\mu,\sigma^2))$ is subgaussian on $\mathcal{A}$ with uniformly bounded subgaussian norm.
\end{thm}
\begin{proof}
    From Talagrand (Theorem 13.7.1 of \cite{talagrand2010mean}) and Theorem 2.1 of \cite{brennecke2023operator}, it is shown that for any $(\beta,\mu,\sigma^2)$ satisfying \Cref{eq:high_temp_derivative} and \Cref{eq:at_line}, the random variable $\sqrt{n} (R_{1,2}-q)$ is subgaussian under $\nu$ with locally bounded subgaussian norm. That is, for each such $(\beta,\mu,\sigma^2)\in \mathcal{A}$, there exists constants $K=K(\beta,\mu,\sigma^2)$ and $\delta=\delta(\beta,\mu,\sigma^2)>0$ such that for any $(\beta',\mu',\sigma'^2)$ satisfying \Cref{eq:high_temp_derivative} and \Cref{eq:at_line} such that $\|(\beta',\mu',\sigma'^2)-(\beta,\mu,\sigma^2)\|_{\infty}\leq \delta$, for $\nu_{n,\beta',\mu',\sigma'^2}$ it holds that
\begin{equation*}
    \nu\left(\exp\left(\frac{n(R_{1,2}-q(\beta',\mu',\sigma'^2))^2}{K}\right)\right)\leq 2.
    \end{equation*}

    A standard compactness argument then implies the uniform bound for any $\mathcal{A}$ satisfying \Cref{assumption:high_temp}: for each $(\beta,\mu,\sigma^2)\in \mathcal{A}$, form the associated nontrivial open set with corresponding subgaussian constant. This collection of open sets covers $\mathcal{A}$, and since $\mathcal{A}$ is compact, there is a finite subcover. Simply taking the largest parameter $K=K(\mathcal{A})<\infty$ over this subcover completes the proof.
\end{proof}

With these preliminaries, we can return to the proof of \Cref{thm:cov_bound_external}.

\begin{thm}[\Cref{thm:cov_bound_external}, restated]
    Let $\mathcal{A}$ satisfy the conditions of \Cref{assumption:high_temp} and let $(\beta,\mu,\sigma^2)\in \mathcal{A}$. Suppose that either:
    \begin{enumerate}
        \item (Small Inverse Temperature) It holds that $\overline{\beta}<1/2-\eta$ for some $\eta>0$.
        \item (Deterministic Field) It holds that $\overline{\sigma}=0$ (that is, the external field is deterministic and parallel to $\mathbf{1}$).
    \end{enumerate}
    
    Then for any $\delta>0$, there exists $K_{\ref{thm:cov_bound_external}}(\mathcal{A},\delta,\varepsilon)$ independent of $n$ such that with probability at least $1-\delta$ over $A\sim \mathsf{GOE}(n)$,
    \begin{equation*} \|\mathrm{Cov}(\mu_{\beta,A,\bm{h}})\|_{\mathsf{op}}\leq K_{\ref{thm:cov_bound_external}}(\mathcal{A},\delta,\varepsilon).
    \end{equation*}
\end{thm}
\begin{proof}
    The main result of Brennecke, Xu, and Yau (Theorem 1.4 of \cite{brennecke2023operator}) shows for any $\delta>0$ and $(\beta,\mu,\sigma^2)\in \mathcal{A}$, and for large enough $n\geq n(\mathcal{A},\delta)$, there exists a constant $K_1(\mathcal{A},\delta)>0$ such that with probability at least $1-\delta$, it holds that
    \begin{equation*}
    \|\text{Cov}(\mu_{\beta,A,\bm{h}})\|_{\mathsf{op}}\leq 2\cdot (\|L^{-1}\|_{\mathsf{op}}+K_1(\mathcal{A},\delta)),
    \end{equation*}
    where $L$ is the matrix defined by 
    \begin{equation*}
        L\triangleq \text{diag}((\bm{1}-\bm{m}^2_{\beta, A,\bm{h}})^{-1})+\beta^2(1-q)I-\beta A,
    \end{equation*}
    where we interpret $\bm{y}^p$ to be the entrywise $p$th powers of a vector. In comparison the the stated result in \cite{brennecke2023operator}, the fact that this bound holds uniformly over fixed $(\beta,\mu,\sigma^2)\in \mathcal{A}$ is because the estimate in \Cref{thm:moment_bounds} holds uniformly over $\mathcal{A}$ under \Cref{assumption:high_temp}, and the constants above depend just on these parameters.

    In the case of deterministic field, it is further shown in Proposition 1.4 of \cite{brennecke2023operator} that for any $\delta>0$, it holds for all $n\geq n(\delta)$ with probability at least $1-\delta$ that $\|L^{-1}\|_{\mathsf{op}}\leq K_2(\mathcal{A},\delta)$. By replacing $\delta$ with $\delta/2$ and taking a union bound, we deduce the desired conclusion for all large enough $n\geq n(\delta)$. For smaller $n$, we may take the trivial upper bound of $n$.

    In the case where $\beta\leq 1/2-\eta$ for some $\eta>0$ with arbitrary Gaussian field, a more direct argument suffices since for any $\eta'>0$, $\|A\|_{\mathsf{op}}\geq 2+\eta'$ with exponentially small probability in $n$ by standard bounds on the operator norm of GOE matrices (see Lemma 4.12 of El Alaoui, Montanari, and Sellke \cite{DBLP:conf/focs/AlaouiMS22} for details). It follows on the complementary event\footnote{Here, we use the standard notation for Loewner order: $A\preceq B\iff B-A$ is positive semi-definite.} that
    \begin{equation*}
        L\succeq I-\beta A\succeq I - (1/2-\eta)(2+\eta')I\succeq \frac{\eta}{2}I
    \end{equation*}
    if we choose $\eta=\eta'$, for instance. For large enough $n\geq n(\delta,\eta)$, this event occurs with probability at least $1-\delta$, and replacing $\delta$ with $\delta/2$ and taking a union bound yields the desired conclusion in this case. Again, for all $n\leq n(\delta,\eta)$, one may simply take the trivial upper bound of $n$.
\end{proof}

We now turn to the proof of \Cref{thm:tap_consequence}. To prove it, we require the following notation: for an index $i\in [n]$, we consider the Sherrington-Kirkpatrick measures $\mu^{(i)}_{\beta,A,\bm{h}}$ where we remove the $i$th spin. That is, for $\bm{x}\in \{-1,1\}^{[n]\setminus \{i\}}$
\begin{equation*}
    \mu^{(j)}_{\beta,A,\bm{h}}(\bm{x})\propto \exp\left( \beta \bm{x}^TA_{[n]\setminus \{i\},[n]\setminus \{i\}}\bm{x}+\sum_{j\neq i} h_j x_j\right).
\end{equation*}
We note that for each $j\in [n]$, the measure $\mu_{\beta,A,\bm{h}}^{(j)}$ has the law of the Sherrington-Kirkpatrick measure on $n-1$ sites with inverse temperature $\beta_{-}=\beta\sqrt{1-1/n}$ and $A'\sim \mathsf{GOE}(n-1)$ due to the normalization. Of course, these subsystems are highly correlated across $j\in [n]$ since they share the bulk of the original $\mathsf{GOE}(n)$ matrix. With this notation, we write $\bm{m}^{(j)}\in \mathbb{R}^{[n]\setminus \{j\}}$ to denote the mean vector of the measure $\mu^{(j)}_{\beta,A,\bm{h}}$, i.e. for $i\neq j$,
\begin{equation*}
\bm{m}^{(j)}_i=\mathbb{E}_{\bm{x}\sim\mu^{(j)}_{\beta,A,\bm{h}}} [x_i].
\end{equation*}

We are now ready to give the proof of \Cref{thm:tap_consequence}, which we restate for convenience:
\begin{thm}[\Cref{thm:tap_consequence}, restated]
        Let $\mathcal{A}=\mathcal{A}(\overline{\beta},\underline{h},\overline{h},\underline{\sigma}^2,\overline{\sigma}^2)$ satisfy the conditions of \Cref{assumption:high_temp} with $\max\{\underline{h},\underline{\sigma^2}\}>0$. There is a constant $C_{\ref{thm:tap_consequence}}(\mathcal{A})>0$ such that for any $(\beta,\mu,\sigma^2)\in \mathcal{A}$, there exists standard Gaussian random variables $z_1,\ldots,z_n$ such that
    \begin{equation*}
        \mathbb{E}_{A,\bm{h},\bm{z}}\left[\max_{i\in [n]} \left\vert \beta\sum_{j\neq i} A_{ij}\bm{m}_j -\beta^2(1-q)\bm{m}_i-\beta z_i\sqrt{q}\right\vert\right]\leq  \frac{C_{\ref{thm:tap_consequence}}(\mathcal{A})}{n^{1/4}}.
    \end{equation*}
\end{thm}
\begin{proof}
    We rely on several facts established by Talagrand. First, Talagrand (proof of Lemma 1.7.6, \cite{talagrand2010mean}) shows that there exists a standard Gaussian $z=z_n$ defined on the same probability space as $A,\bm{h}$ such that
    \begin{equation*}
        \mathbb{E}_{A,\bm{h},z}\left[\left( z\sqrt{q} - \sum_{i<n} A_{i,n}\bm{m}^{(n)}_{i}\right)^4\right]\leq \frac{3}{q^2}\mathbb{E}_{A,\bm{h},z}\left[\mathbb{E}_{\sigma^1,\sigma^2\sim \mu^{(n)}_{\beta,A,\bm{h}}}\left[\left(R^{-}_{1,2}-q\right)^4\right]\right].
    \end{equation*}
    Here, we write $R_{1,2}^{-}$ to denote $\frac{1}{n-1}\sum_{i=1}^{n-1} \sigma^1_i\sigma^2_i$ since these samples are from the cavity measure on the first $n-1$ spins. We now replace $q$ with $q_-=q(\beta_-,\mu,\sigma^2)$:
    \begin{equation*}
\mathbb{E}_{A,\bm{h},z}\left[\mathbb{E}_{\sigma^1,\sigma^2\sim \mu^{(n)}_{\beta,A,\bm{h}}}\left[\left(R^{-}_{1,2}-q\right)^4\right]\right]\leq 8\mathbb{E}_{A,\bm{h},z}\left[\mathbb{E}_{\sigma^1,\sigma^2\sim \mu^{(n)}_{\beta,A,\bm{h}}}\left[\left(R^{-}_{1,2}-q_-\right)^4\right]\right]+8(q-q_-)^4,
    \end{equation*}
    where we use the simple inequality $(x+y)^4\leq 8x^4+8y^4$. But since $\mu_{\beta,A,\bm{h}}^{(n)}$ has the same law as $\mu_{\beta_-,A',\bm{h}}$ where $A'\sim \mathsf{GOE}(n-1)$ as described above,  we may apply \Cref{thm:moment_bounds} to see that there a constant $K_1(\mathcal{A})$ such that
    \begin{equation*}
\mathbb{E}_{A,\bm{h},z}\left[\mathbb{E}_{\sigma^1,\sigma^2\sim \mu^{(n)}_{\beta,A,\bm{h}}}\left[\left(R^{-}_{1,2}-q_-\right)^4\right]\right]\leq \frac{K_1(\mathcal{A})}{n^2}.
    \end{equation*}
    Finally, applying \Cref{lem:q_close} and combining these inequalities implies that there is some constant $K_2(\mathcal{A})$ such that
    \begin{equation}
    \label{eq:close_1}
        \mathbb{E}_{A,\bm{h},z}\left[\left( z\sqrt{q} - \sum_{i<n} A_{i,n}\bm{m}^{(n)}_{i}\right)^4\right]\leq \frac{K_2(\mathcal{A})}{n^2};
    \end{equation}
    here, we also use the fact that our assumption on $\mathcal{A}$ ensures that $q\geq c(\mathcal{A})>0$ for some constant $c(\mathcal{A})$.

    Next, Talagrand (proof of Theorem 1.7.7 on page 78 of \cite{talagrand2010mean}) shows that under these same conditions, it holds that
    \begin{equation}
    \label{eq:close_2}
        \mathbb{E}_{A,\bm{h},z}\left[\left( \beta\sum_{i=1}^{n-1} A_{i,n}\bm{m}_i-\beta^2(1-q)\bm{m}_n-\beta\sum_{i=1}^{n-1} A_{i,n}\bm{m}^{(n)}_{i} \right)^4\right]\leq \frac{K_3(\mathcal{A})}{n^2}
    \end{equation}
    for some other constant $K_3(\mathcal{A})>0$. Combining \Cref{eq:close_1,eq:close_2} with the same inequality $(x+y)^4\leq 8(x^4+y^4)$ implies that 
    \begin{equation*}
        \mathbb{E}_{A,\bm{h},z}\left[\left( \beta\sum_{i=1}^{n-1} A_{i,n}\bm{m}_i-\beta^2(1-q)\bm{m}_n-\beta z\sqrt{q} \right)^4\right]\leq \frac{K_4(\mathcal{A})}{n^2}.
    \end{equation*}

    The same argument holds for each index $j\in [n]$, so we may find standard normal variables $z_1,\ldots,z_n$ on the same probability space where the previous display holds for each $j\in [n]$. It finally follows by symmetry that 
    \begin{align*}
    \mathbb{E}_{A,\bm{h},z}&\left[\max_{i\in [n]}\left( \beta\sum_{j\neq i} A_{i,j}\bm{m}_j-\beta^2(1-q)\bm{m}_i-\beta z_i\sqrt{q} \right)^4\right]\\
    &\leq \mathbb{E}_{A,\bm{h},z}\left[\sum_{i\in [n]}\left( \beta\sum_{j\neq i} A_{i,j}\bm{m}_j-\beta^2(1-q)\bm{m}_i-\beta z_i\sqrt{q} \right)^4\right]\\
    &\leq \frac{K_4(\mathcal{A})}{n},
    \end{align*}
    and Jensen's inequality implies that 
    \begin{equation*}
        \mathbb{E}_{A,\bm{h},z}\left[\max_{i\in [n]}\left\vert\beta\sum_{j\neq i} A_{i,j}\bm{m}_j-\beta^2(1-q)\bm{m}_i-\beta z_i\sqrt{q} \right\vert\right]\leq \frac{K_4(\mathcal{A})^{1/4}}{n^{1/4}}.
    \end{equation*}
\end{proof}

\end{document}